\documentclass[twoside]{article}

\usepackage{microtype}
\usepackage{graphicx}
\usepackage{subcaption}
\usepackage{booktabs}
\usepackage{multirow}
\usepackage{hyperref}
\usepackage{enumitem}
\usepackage{adjustbox}
\usepackage{longtable}
\usepackage{csvsimple}
\usepackage{fancyvrb}
\usepackage{fvextra}
\usepackage{array}
\usepackage{rotating}
\usepackage{pdflscape}
\usepackage{xcolor}
\usepackage{colortbl}
\usepackage{placeins}
\usepackage{dblfloatfix}
\usepackage{underscore}

\usepackage[preprint]{aistats2026}

\usepackage[round]{natbib}

\usepackage{amsmath}
\usepackage{amssymb}
\usepackage{mathtools}
\usepackage{amsthm}

\usepackage[capitalize,noabbrev]{cleveref}

\usepackage{algorithm}
\usepackage{algorithmic}

\newcommand{\cP}{\mathcal{P}}
\newcommand{\cQ}{\mathcal{Q}}

\theoremstyle{plain}
\newtheorem{theorem}{Theorem}[section]
\newtheorem{proposition}[theorem]{Proposition}
\newtheorem{lemma}[theorem]{Lemma}
\newtheorem{corollary}[theorem]{Corollary}
\theoremstyle{definition}
\newtheorem{definition}[theorem]{Definition}
\newtheorem{assumption}[theorem]{Assumption}
\theoremstyle{remark}
\newtheorem{remark}[theorem]{Remark}

\newcommand{\RR}{\mathbb{R}}
\newcommand{\EE}{\mathbb{E}}

\newcommand{\cX}{\mathcal{X}}
\newcommand{\cY}{\mathcal{Y}}
\newcommand{\cD}{\mathcal{D}}
\newcommand{\cS}{\mathcal{S}}
\newcommand{\cU}{\mathcal{U}}
\newcommand{\cM}{\mathcal{M}}

\newcommand{\ip}[2]{\left\langle #1,\, #2 \right\rangle}

\begin{document}

\runningtitle{RUBRIC: Realism--Utility Balanced Ranking}
\runningauthor{Yu et al.}

\twocolumn[

\aistatstitle{RUBRIC: Realism-Utility Balanced Ranking for Imbalanced Classification}

\aistatsauthor{%
Yanxuan Yu$^{1}$ \quad Dong Liu$^{2}$ \quad Shu Wang$^{2}$ \quad Wenxiao Zhao$^{2}$ \quad Eric Jiang$^{2}$ \\[0.3em]
\bfseries Chang Liu$^{3}$ \quad Jinxi Yu$^{2}$ \quad Hui Pan$^{4}$ \quad Renata Borovica-Gajic$^{5}$ \quad Ben Lengerich$^{6}$%
}
\aistatsaddress{%
$^{1}$Columbia University \quad $^{2}$University of California, Los Angeles \quad $^{3}$University of Illinois Urbana-Champaign \\[0.15em]
$^{4}$Independent Researcher \quad $^{5}$University of Melbourne \quad $^{6}$University of Wisconsin--Madison%
}
]


\begin{abstract}
Class imbalance poses a fundamental challenge in risk-sensitive applications such as fraud detection and medical diagnosis, where minority-class samples are scarce yet critical for accurate classification. Existing oversampling methods generate synthetic samples to rebalance class distributions; however, they often produce large numbers of low-quality candidates that distort decision boundaries or introduce artifacts, leading to overfitting and degraded generalization.
In this work, we introduce \textbf{RUBRIC}, a generator-agnostic filtering framework that formulates synthetic sample selection as a quality-over-quantity optimization problem. RUBRIC ranks candidates using a realism–utility trade-off: realism is estimated via a neural density-ratio discriminator from each candidate's resemblance to real minority samples, while utility captures proximity to the decision boundary through a concave, margin-based scoring function $g_-(t)=-\log(1+e^{-t/\tau})$. The discriminator uses the same architecture and training protocol on every benchmark and is fit independently to that dataset's real minority class versus its synthetic pool. We show that, under mild regularity conditions, the proposed filtering framework monotonically tightens the generalization bound for margin-based classifiers by jointly reducing distribution shift and suppressing near-negative tail contributions.
Through extensive experiments on standard public imbalanced-classification benchmarks, we demonstrate that RUBRIC boosts minority-class recall while preserving overall discriminative ability across multiple data generators. Sensitivity analyses in $\lambda$ and the selection budget $K$ further characterize performance trade-offs oriented toward ranking quality.
\end{abstract}

\section{INTRODUCTION}
Severe class imbalance is prevalent in risk screening and medical decision support~\cite{he2009learning,krawczyk2016learning}. We focus on risk screening settings with public, highly imbalanced tabular benchmarks (e.g., fraud detection), where extreme skew is common and evaluation is reproducible. A standard response is to oversample the minority class, yet widely used generators can produce many low-quality candidates. For example, SMOTE~\cite{chawla2002smote} interpolates within local neighborhoods; without careful control, synthetic points may be unrealistic or push the decision boundary into regions unsupported by real data.

Most prior work addresses where to synthesize minority samples. Boundary-aware variants emphasize difficult regions, while GAN-based approaches seek greater sample realism through adversarial training~\cite{douzas2018effective,fiore2019using}. However, realism and downstream utility are not equivalent: a plausible-looking sample may be uninformative or even harmful for learning, and aggressively boundary-focused synthesis can amplify distribution shift.

Our premise is orthogonal. Rather than designing new generators, we ask which synthetic samples to retain. We propose \textbf{RUBRIC}, a generator-agnostic post-filter that scores each candidate $\tilde{x}$ by (i) its boundary utility $u(\tilde{x})$ with respect to a target classifier family and (ii) a realism score $r(\tilde{x})$ from a neural density-ratio discriminator that ranks synthetic candidates by their resemblance to real minority samples. RUBRIC then selects a fixed-size subset that optimally trades off utility and realism under a budget. While our experiments use public imbalanced tabular benchmarks, the same selection principle applies to other high-stakes imbalanced domains, including medical decision support.

We further show that, under mild regularity conditions, selective filtering monotonically tightens a margin-based generalization bound (Sec.~\ref{sec:theory}) by jointly reducing distribution shift and suppressing near-negative tail contributions. Empirically, across standard public imbalanced-classification benchmarks, RUBRIC improves recall and F1 under a fixed selection budget while maintaining comparable ROC-AUC across multiple generators, and we explicitly characterize when AUPRC can decrease under aggressive boundary prioritization.

In summary, this work makes the following contributions:
\begin{itemize}
    \vspace{-0.5em}
    \item We introduce a generator-agnostic post-filter that formalizes synthetic sample selection as a constrained utility--realism optimization problem.
    
    \item We provide theory showing that selective filtering tightens a margin-based generalization bound under mild assumptions.
    
    \item We cast diversity-aware selection under a fixed synthetic budget as a monotone submodular maximization problem and use a greedy algorithm with a $(1-1/e)$ guarantee on a nonnegative score-shifted objective~\cite{nemhauser1978analysis,krause2014submodular}.
    
    \item We demonstrate consistent gains across standard public imbalanced-classification benchmarks, with explicit analysis of the realism--utility trade-off, the selection budget $K$, and metric behavior.
\end{itemize}

\vspace{-0.5em}
\section{RELATED WORK}
\vspace{-0.5em}
Imbalanced learning methods are commonly categorized by their approach to minority sample generation. Classical interpolation-based methods such as SMOTE~\cite{chawla2002smote} and its variants emphasize local neighborhoods, boundary regions, or adaptive difficulty weighting~\cite{he2008adasyn,han2005borderline,last2018kmeanssmote,bunkhumpornpat2009safelevel}. GAN-based approaches aim to improve realism through adversarial training~\cite{douzas2018effective,fiore2019using,mariani2018bagan}, while techniques such as MWMOTE~\cite{barua2014mwmote} incorporate majority-class structure. Despite these advances, comprehensive surveys~\cite{fernandez2018smote,haixiang2017learning,johnson2019survey} note persistent challenges in balancing sample plausibility with downstream effectiveness.

Our work is complementary to these approaches. Rather than modifying how or where samples are generated, we introduce a post-generation filtering perspective that ranks and selects synthetic candidates. This distinguishes RUBRIC from generator-specific improvements and allows it to be stacked with any oversampling method. Cost-sensitive learning~\cite{elkan2001foundations,zhou2016cost} and ensemble-based resampling~\cite{galar2012review,seiffert2008rusboost} address imbalance through reweighting or data selection, but do not explicitly score or filter synthetic candidates based on both realism and boundary utility. Hybrid ensembles such as RUSBoost~\cite{seiffert2008rusboost} remain a strong non-generative baseline; we include them in the empirical comparison.

Theoretically, our discriminator-based realism score connects to density ratio estimation~\cite{sugiyama2012density} and distribution shift analysis~\cite{ben2010theory,duchi_namkoong_2021_uniform,pinsker1964information}. Discriminator Rejection Sampling (DRS)~\cite{azadi2019discriminator} already uses a trained discriminator's density ratio to accept or reject generated samples. RUBRIC differs in three respects: (i)~the acceptance decision is replaced by a fixed-budget ranking problem; (ii)~the score combines the density-ratio logit with a concave boundary-utility term rather than realism alone; and (iii)~a normalized facility-location diversity term prevents near-duplicate selections. Our utility formulation is grounded in margin-based generalization theory~\cite{schapire1998boosting,bartlett2002rademacher,koltchinskii2002empirical}. Recent work has explored boundary-aware filtering of SMOTE candidate pools~\cite{yu2026boundary}; that method is a utility-only special case of our score ($\lambda{=}1$, $\gamma{=}0$). RUBRIC extends this line by formalizing selection as a budgeted realism--utility--diversity optimization with generator-agnostic theory, and we compare both DRS and boundary-aware filtering as post-generation controls. We further wrap a modern tabular diffusion generator~\cite{kotelnikov2023tabddpm} to test whether the same filter remains useful beyond interpolation-based oversamplers. The same imbalance setting appears in lightweight clinical screening~\cite{pan2024mobilenetv2} and sub-Nyquist EEG reconstruction~\cite{yu2026coprimeeeg}. Adaptive inference~\cite{yao2026context}, diffusion-cache correction~\cite{liu2026adacorrection}, and long-context memory and tokenization~\cite{liu2025hsgm,liu2026semtoken,liu2026mka} address related efficiency questions outside synthetic selection.

\vspace{-0.5em}
\section{METHOD}
\vspace{-0.5em}
We now describe RUBRIC, including the problem setup, the adversarial filtering objective, and how the resulting selection procedure underpins our theoretical analysis.

\vspace{-0.5em}
\subsection{Problem Setup}
\vspace{-0.5em}
Let $\cX\subset\RR^d$ denote the feature space and $\cY=\{0,1\}$ the label space (with minority class $=1$). We observe $n_0$ majority and $n_1$ minority samples, with $n_1\ll n_0$, drawn i.i.d. from $P(X,Y)$. An oversampler $\mathsf{Gen}$ produces a multiset of synthetic candidates $\cS=\{\tilde{x}_i\}_{i=1}^m$, intended to approximate the minority manifold $\cM=\mathrm{supp}(P(X\mid Y{=}1))$.

\vspace{-0.5em}
\textbf{Candidate generation when $\mathsf{Gen}{=}$ NONE.}
Although RUBRIC is designed to operate downstream of an external generator, we also evaluate a \textsc{NONE} setting in which candidates are produced by a built-in, generator-agnostic proposal (nearest-neighbor interpolation, a scalar-scaled boundary-directed shift, and optional KDE resampling). This yields \textsc{NONE\_RUBRIC}, a self-contained instantiation of our filter-after-generation design. The proposal is an \emph{internal} candidate mechanism rather than a claim that no generator is used; details appear in Appendix~\ref{app:none-gen}.

\vspace{-0.5em}
\subsection{Adversarial Filtering Objective}
\vspace{-0.5em}
RUBRIC trains two auxiliary models once and then freezes them for candidate scoring. The term \emph{adversarial} refers solely to the discriminator’s role in scoring how closely each synthetic candidate resembles real minority samples; we do not employ iterative adversarial training.

\textbf{Training order.} Both auxiliary models are trained prior to candidate scoring:
(i) \emph{Boundary model} $f:\cX\to\RR$ is trained exclusively on the real training data $\cD$ (with the minority class as positive) to produce margin scores. The final classifier is never trained on synthetic candidates.
(ii) \emph{Discriminator} $D:\cX\to[0,1]$ is a neural density-ratio discriminator trained, independently on each dataset, to separate real minority samples from the synthetic candidate pool $\cS$ (treating $\cS$ as negatives). We instantiate $D$ as a single-hidden-layer MLP ($d\to 128\to 1$) with the same architecture and training protocol on every benchmark; the resulting log-odds scores rank candidates by resemblance to real minority data rather than deciding whether a generated point is real.

For a candidate $\tilde{x}$, we define:
\vspace{-0.5em}
\begin{align}
u(\tilde{x}) &\triangleq g\!\left(\mathrm{margin}_f(\tilde{x})\right), \label{eq:utility}\\
r(\tilde{x}) &\triangleq \log \frac{D(\tilde{x})}{1-D(\tilde{x})}. \label{eq:realism}
\end{align}
where $\mathrm{margin}_f(x)=f(x)$ is oriented such that larger values indicate greater minority affinity.

\textbf{Utility shaper $g$.}
The default shaper is the concave log-sigmoid
$g_-(t)=-\log(1{+}e^{-t/\tau})$
with temperature $\tau{>}0$.
Its derivative $g_-'(t)=\tau^{-1}\sigma(-t/\tau)$ is decreasing, so near-boundary points receive higher marginal utility than easy positives far on the minority side.
The softplus $g_+(t)=\log(1{+}e^{t/\tau})$ is strictly convex (Appendix~\ref{app:implementation}); we retain it only as an optional increasing alternative and do \emph{not} invoke Lemma~\ref{lem:margin}'s concavity argument for $g_+$.
In the theoretical analysis (\cref{sec:theory}), monotonicity, $1/\tau$-Lipschitzness, and concavity of $g_-$ control the utility ranking and the (vacuous under a hard floor) tail envelope.

\textbf{Discriminator logit and density ratio.}
RUBRIC uses the discriminator for relative ranking within the candidate pool, not for absolute density estimation.
When $D$ is trained to distinguish real minority samples ($+$) from synthetic candidates ($-$), the Bayes-optimal discriminator satisfies
$D^*(x)=\tfrac{p(x)}{p(x)+q_{\cS}(x)}$,
where $p(x)\equiv p(x\mid Y{=}1)$ is the true minority density and $q_{\cS}(x)$ is the density of the \emph{full candidate pool} $\cS$. The selected subset $\cU$ has a different law $q_{\cU}$ (or $q_\rho$ under top-$K$ level-set selection); Section~\ref{sec:theory} records the corresponding reweighting correction.
Taking log-odds gives the exact identity
$\log \tfrac{D^*(x)}{1-D^*(x)}=\log \tfrac{p(x)}{q_{\cS}(x)}$.
In practice, we model deviations via an estimation error (Assumption~\ref{as:real}); even under miscalibration, RUBRIC remains valid because selection depends on score ordering, not absolute logit values.
We exploit the approximate density-ratio relationship in \cref{sec:theory} to bound distribution shift via Pinsker's inequality.

\textbf{Scale of $u$ and $r$.}
Raw utility $u=g_-(\mathrm{margin}_f)$ is bounded, while the clipped logit $r$ can reach $\pm\log((1-\epsilon)/\epsilon)$ (Lemma~\ref{lem:impl-logit}).
Before mixing, we standardize both coordinates on $\cS$:
\[
\tilde u=\frac{u-\widehat{\mu}_u}{\widehat{\sigma}_u+\epsilon},\qquad
\tilde r=\frac{r-\widehat{\mu}_r}{\widehat{\sigma}_r+\epsilon},
\]
and set $s(\tilde{x})=\lambda\tilde u+(1{-}\lambda)\tilde r$.
Relative ordering under $\lambda$ is then a genuine trade-off rather than a scale artifact.
The trade-off parameter $\lambda$ therefore controls the precision--recall operating regime induced by the selected training distribution.

Given a trade-off parameter $\lambda\in[0,1]$ and a selection budget $K$, RUBRIC solves:
\begin{equation}
\max_{\cU\subseteq\cS,\ |\cU|=K}\ 
\underbrace{\sum_{\tilde{x}\in\cU}s(\tilde{x})}_{\text{realism--utility score}}
\;+\;\underbrace{\gamma\,\widehat{\mathrm{Div}}(\cU)}_{\text{normalized diversity}},
\label{eq:selection}
\end{equation}
where $s(\tilde{x})=\lambda\tilde u(\tilde{x})+(1{-}\lambda)\tilde r(\tilde{x})$ uses the standardized coordinates above and $\gamma\ge 0$ controls diversity.
We instantiate diversity as a \emph{normalized} facility-location coverage,
\begin{equation}
\widehat{\mathrm{Div}}(\cU)\triangleq |\cS|^{-1}\sum_{\tilde{x}\in\cS}\max_{\tilde{u}\in\cU}\kappa(\tilde{x},\tilde{u})\in[0,1],
\label{eq:diversity}
\end{equation}
where $\kappa$ is a similarity kernel (e.g., RBF in feature space or a $k$NN graph similarity). Dividing by $|\cS|$ keeps the diversity term on a comparable scale to the modular score as the candidate pool grows.

\begin{proposition}[Submodularity of Eq.~\eqref{eq:selection}]\label{prop:submodular}
The normalized facility-location function $\widehat{\mathrm{Div}}(\cdot)$ in Eq.~\eqref{eq:diversity} is monotone submodular.
Write $s_{\min}=\min_{\tilde{x}\in\cS}s(\tilde{x})$ and $\tilde{s}=s-s_{\min}\ge 0$.
The shifted objective $F(\cU)=\sum_{\tilde{x}\in\cU}\tilde{s}(\tilde{x})+\gamma\,\widehat{\mathrm{Div}}(\cU)$ is nonnegative and monotone submodular.
Greedy maximization of $F$ under $|\cU|=K$ achieves a $(1{-}1/e)$ approximation to $\max F$~\cite{nemhauser1978analysis,krause2014submodular}.
Translating back to the unshifted modular sum yields the additive identity
$\sum_{\cU}s=\sum_{\cU}\tilde{s}+Ks_{\min}$,
so the same set is a $(1{-}1/e)$ maximizer of $F$ and a maximizer of $\sum s+\gamma\widehat{\mathrm{Div}}$ up to the constant $Ks_{\min}$.
When $\gamma=0$, selection reduces to top-$K$ ranking by $s(\tilde{x})$.
\end{proposition}
The classical $(1{-}1/e)$ factor requires nonnegativity; the shift $\tilde{s}$ supplies it. See Appendix~\ref{app:detailed_proofs}.

\vspace{-0.5em}
\subsection{Overview and Equation Mapping}
\vspace{-0.5em}
The RUBRIC pipeline proceeds as follows:
(i) a generator $\mathsf{Gen}$ produces candidate samples $\cS$;
(ii) the boundary model $f$ computes margins, shaped by $g$ to yield utility $u(\tilde{x})$ in Eq.~\eqref{eq:utility};
(iii) the discriminator $D$ produces log-odds realism scores $r(\tilde{x})$ in Eq.~\eqref{eq:realism};
(iv) the combined score $s(\tilde{x})=\lambda u+(1{-}\lambda)r$ defines the realism--utility term in Eq.~\eqref{eq:selection};
(v) a greedy submodular selection procedure selects $\cU$ under budget $K$ using Eq.~\eqref{eq:selection} (including diversity);
(vi) the augmented dataset $\cD\cup\cU$ is used to train the final classifier.

This pipeline directly motivates our theoretical analysis (Sec.~\ref{sec:theory}).

\textbf{Reasoning logic.}
(a) Generator-agnostic candidate pools avoid unnecessary modeling commitments;
(b) utility $u$ concentrates selection near, but not far beyond, the decision boundary, improving recall;
(c) realism $r$ reduces domain shift between $\cU$ and the minority distribution $p$, controlling the bias term via total-variation distance;
(d) the budgeted objective trades off (b) and (c) under a fixed budget;
(e) a submodular diversity term discourages redundant, near-duplicate selections while preserving coverage;
(f) together, these effects tighten excess risk bounds by reducing shift and suppressing near-negative tails.

\vspace{-0.5em}
\section{THEORETICAL ANALYSIS}
\label{sec:theory}
\vspace{-0.5em}
We analyze how RUBRIC's selection affects excess risk for margin-based classifiers.
Throughout, labels in the surrogate loss are encoded as $y\in\{-1,+1\}$ (minority $y{=}+1$); this is equivalent to the $\{0,1\}$ encoding in Sec.~3 via $y' = 2y{-}1$.

Let $\phi$ denote an RKHS feature map and consider linear predictors $h_w(x)=\mathrm{sign}(\ip{w}{\phi(x)})$ with $\|w\|\le B$.
The surrogate loss $\ell$ is convex, $L$-Lipschitz, and bounded in $[0,1]$ by rescaling.
Recall $u(\tilde{x})$, $r(\tilde{x})$, and $s(\tilde{x})=\lambda u+(1{-}\lambda)r$ from Eqs.~\eqref{eq:utility}--\eqref{eq:realism}, and the selection objective Eq.~\eqref{eq:selection}.
The default shaper $g_-$ is non-decreasing, concave, and $1/\tau$-Lipschitz.
Let $p(x)\equiv p(x\mid Y{=}1)$ denote the minority density, $q_{\cS}$ the candidate-pool density against which $D$ is trained, and $q_{\cU}$ (equivalently $q_\rho$ for a top-$K$ level set of mass $\rho=K/|\cS|$) the law of the selected subset.

\textbf{Total-variation distance.}
For distributions $\cP,\cQ$ on $\cX$, the total-variation (TV) distance is
$\|\cP-\cQ\|_{\mathrm{TV}}\triangleq \sup_{A\subseteq\cX}|\cP(A)-\cQ(A)|$~\cite{pinsker1964information}.

\begin{assumption}[Discriminator calibration with bounded error]\label{as:real}
Let $\ell_D(x)\triangleq \log \tfrac{D(x)}{1-D(x)}$ denote the discriminator logit.
There exists $\varepsilon_{\mathrm{est}}\ge 0$ such that
\begin{equation}
\sup_{x\in\cX}\big|\ell_D(x)-\log \tfrac{p(x)}{q_{\cS}(x)}\big|\;\le\;\varepsilon_{\mathrm{est}}.
\label{eq:eps-est}
\end{equation}
This captures both suboptimality of $D$ and finite-sample estimation error against the \emph{pool} law $q_{\cS}$.
When $\varepsilon_{\mathrm{est}}=0$, Eq.~\eqref{eq:eps-est} holds exactly for the Bayes-optimal $D^*$.
On a selected level set of mass $\rho=K/|\cS|$,
$\log(p/q_\rho)=\log(p/q_{\cS})+\log\rho$,
so the same logit estimates $\mathrm{KL}(p\|q_\rho)$ only after the additive shift $\log\rho$ is restored (Appendix~\ref{app:detailed_proofs}).
\end{assumption}

\begin{lemma}[TV shift bound via Pinsker]\label{lem:pinsker}
Under Assumption~\ref{as:real}, define the pool plug-in
$\widehat{\mathrm{KL}}(p\Vert q_{\cS})\triangleq \max\bigl(0,\,\EE_{p}[\ell_D(x)]\bigr)$.
Then
\begin{equation}
\|p-q_{\cS}\|_{\mathrm{TV}}
\le \hat{\delta}
\triangleq \sqrt{\tfrac{1}{2}\widehat{\mathrm{KL}}(p\Vert q_{\cS})} + \varepsilon_{\mathrm{est}}.
\end{equation}
The same argument applied to the reweighted logit $\ell_D+\log\rho$ bounds $\|p-q_\rho\|_{\mathrm{TV}}$.
The elementary comparison $\sqrt{\varepsilon/2}\le\varepsilon+1/8$ yields the slightly looser form $\hat{\delta}+1/8$; we keep the slack explicit rather than absorbing it into $\varepsilon_{\mathrm{est}}$.
\end{lemma}
    
\begin{proof}[Proof sketch]
By Pinsker’s inequality,
$\|p-q_{\cS}\|_{\mathrm{TV}} \le \sqrt{\tfrac{1}{2}\mathrm{KL}(p\Vert q_{\cS})}$.
Assumption~\ref{as:real} gives $\mathrm{KL}(p\|q_{\cS})\le\EE_p[\ell_D]+\varepsilon_{\mathrm{est}}$.
The $\max(0,\cdot)$ projection keeps the plug-in in the domain of $\sqrt{\,\cdot\,}$.
\end{proof}

\paragraph{Score threshold and margin floor.}
Score-based selection induces a threshold $t_K$ such that $s(\tilde{x})\ge t_K$ for all $\tilde{x}\in\cU$ when $\gamma=0$.
By monotonicity of $g_-$, this implies a margin floor $m_0$.
A positive diversity weight can \emph{displace} a $\vartheta$-fraction of the modular top-$K$ set and thereby lower the empirical floor; Proposition~\ref{prop:submodular} and Appendix~\ref{app:detailed_proofs} charge this displacement as an additive $M\vartheta$ term rather than claiming that diversity raises $m_0$.

\begin{assumption}[Margin floor from selection]\label{as:util}
When $\gamma=0$, the selection rule induces a margin floor $m_0>0$ such that $\mathrm{margin}_f(\tilde{x})\ge m_0$ for all $\tilde{x}\in\cU$.
For a linear margin model, this is equivalent to $y\,\ip{\hat w}{\phi(\tilde{x})}\ge m_0$.
The floor $m_0$ depends on $t_K$, $\lambda$, and the monotonicity of $g_-$.
\end{assumption}

\begin{remark}[Top-$K$ analysis vs.\ submodular selection]
Assumption~\ref{as:util} characterizes the modular ($\gamma=0$) boundary-selection effect via a score threshold.
When $\gamma>0$, greedy submodular selection (Prop.~\ref{prop:submodular}) may replace a fraction $\vartheta$ of the top-$K$ set; the generalization bound records this as a displacement penalty rather than a raised floor.
\end{remark}
\begin{lemma}[TV shift controls risk gap]\label{lem:tv}
Let $\cP$ and $\cQ$ be joint distributions on $(x,y)$ with
$\|\cP-\cQ\|_{\mathrm{TV}}\le \delta$.
If $\ell\in[0,1]$, then for any $w$,
\begin{equation}
\big|\EE_{\cP}[\ell(y\ip{w}{\phi(x)})]
-\EE_{\cQ}[\ell(y\ip{w}{\phi(x)})]\big|
\le \delta.
\end{equation}
More generally, if $\ell$ is bounded by $M$, the bound becomes $M\,\delta$.
\end{lemma}

\begin{lemma}[Margin shaping under a hard floor]\label{lem:margin}
Suppose Assumption~\ref{as:util} holds and $g_-$ is non-decreasing, concave, and $1/\tau$-Lipschitz.
Then every selected point satisfies $\mathrm{margin}_f\ge m_0$, so the set of selected synthetics with margin below $m_0$ is empty and their surrogate contribution is exactly $0$.
The envelope $C_2 e^{-m_0/\tau}$ is therefore a valid but vacuous upper bound on that empty tail.
If a fraction $\rho_-$ of $\cU$ violates the floor (as permitted when $\gamma>0$), the contribution is at most $(K/N)M\rho_-$.
For the default shaper, $g_-(0)-g_-(-m_0)\le e^{-m_0/\tau}$ still controls the utility gap of any displaced point.
Full proof in Appendix~\ref{app:detailed_proofs}.
\end{lemma}

\begin{theorem}[Generalization tightening]\label{thm:bound}
Assume the sample-split protocol of Appendix~\ref{app:detailed_proofs}: $\cU$ is formed on a fold $\cD_A$ independent of the ERM fold $\cD_B$.
Under Assumptions~\ref{as:real}--\ref{as:util}, with probability at least $1{-}\eta$ over $\cD_B$ and generator randomness, the excess risk of ERM on $\cD_B\cup\cU$ satisfies
\vspace{-0.5em}
\begin{equation}
\begin{split}
R(h_{\widehat{w}}) - R(h^\star)
\le
\mathcal{O}\!\left(\frac{LB}{\sqrt{N}} \right)
+ C_0\,\frac{K}{N}
+ C_1\,(\hat{\delta}+\varepsilon_{\mathrm{est}}) \\
+ C_2\,e^{-m_0/\tau}
+ \sqrt{\frac{\log(1/\eta)}{N}},
\end{split}
\end{equation}
where $N=n_0+n_1+K$.
The irreducible prior-drift term $C_0 K/N$ and the calibration slack implicit in $\hat{\delta}$ (including the $+1/8$ comparison of Lemma~\ref{lem:pinsker}) do not vanish if $K/N$ is held fixed.
The displayed $C_2$ term is $0$ under the hard floor of Assumption~\ref{as:util}.
The same inequality is not claimed for the joint (unsplit) protocol used in Algorithm~\ref{alg:smote-adv}; McDiarmid's inequality does not control the dependence between $\cU$ and an ERM sample that includes $\cU$.
\end{theorem}

\begin{remark}[Interpretation]
Compared to retaining all candidates in $\cS$, RUBRIC reduces the shift contribution by decreasing $\hat{\delta}$ (via realism filtering) and, when $\gamma=0$, by raising $m_0$ (via utility shaping) under a fixed budget $K$.
The term $\varepsilon_{\mathrm{est}}$ enters additively, so imperfect discriminator calibration yields graceful degradation.
Proof in Appendix~\ref{app:detailed_proofs}.
\end{remark}

\vspace{-0.5em}
\section{ALGORITHM}
\vspace{-0.5em}
Algorithm~\ref{alg:smote-adv} executes the maps of Sec.~3 in order. Both $f$ and $D$ are fit once on $\cD_{\mathrm{tr}}$ and frozen before scoring; $h$ is fit only on $\cD_{\mathrm{tr}}\cup\cU$. By Proposition~\ref{prop:submodular} the greedy loop attains a $(1{-}1/e)$ approximation to the shifted objective $F$; $\gamma{=}0$ yields $\cU=\mathrm{TopK}(s,K)$. Scoring an $h$-wide MLP costs $O(|\cS|\,d\,h)$; a sparse $k$NN kernel costs $\tilde{\mathcal{O}}(|\cS|\,k\,d)$ and greedy selection $\mathcal{O}(K\,|\cS|\,k)$. The test split is scored once after $h$ and $t^{\star}$ are frozen.

\begin{algorithm}[tb]
\caption{RUBRIC selection}
\label{alg:smote-adv}
\small
\begin{algorithmic}[1]
\STATE \textbf{Input:} $\cD$, $\mathsf{Gen}$, $K\in\mathbb{N}$, $\lambda\in[0,1]$, $\gamma\ge 0$, $g_-$
\STATE \textbf{Output:} $h$, $\cU$, $t^{\star}$
\STATE Stratify $\cD=\cD_{\mathrm{tr}}\cup\cD_{\mathrm{val}}$; $\cD_1\leftarrow\{x:(x,1)\in\cD_{\mathrm{tr}}\}$
\STATE $\cS\leftarrow\mathsf{Gen}(\cD_{\mathrm{tr}})$ \COMMENT{train-fold only}
\STATE $f\leftarrow\mathrm{Fit}(\cD_{\mathrm{tr}})$; freeze $f$
\STATE $L(D')\leftarrow -\EE_{\cD_1}[\log D']-\EE_{\cS}[\log(1{-}D')]$
\STATE $D\leftarrow\arg\min_{D'}L(D')$; freeze $D$
\FOR{$\tilde{x}\in\cS$}
  \STATE $u(\tilde{x})\leftarrow g_-(\mathrm{margin}_f(\tilde{x}))$,\; $r(\tilde{x})\leftarrow\log\tfrac{D(\tilde{x})}{1-D(\tilde{x})}$
  \STATE standardize $(u,r)$ on $\cS$ to $(\tilde u,\tilde r)$;\; $s(\tilde{x})\leftarrow\lambda\,\tilde u(\tilde{x})+(1{-}\lambda)\,\tilde r(\tilde{x})$
\ENDFOR
\STATE \mbox{$\kappa\leftarrow\mathrm{kNN\text{-}RBF}(\cS)$;\; $\widehat{\mathrm{Div}}(\cU)\triangleq|\cS|^{-1}\displaystyle\sum_{\tilde{x}\in\cS}\max_{\tilde{u}\in\cU}\kappa(\tilde{x},\tilde{u})$}
\STATE $\cU\leftarrow\emptyset$
\FOR{$t=1,\ldots,K$}
  \STATE $\Delta F(\tilde{x}\mid\cU)\leftarrow s(\tilde{x})+\gamma\big(\widehat{\mathrm{Div}}(\cU\cup\{\tilde{x}\})-\widehat{\mathrm{Div}}(\cU)\big)$
  \STATE $\tilde{x}^{\star}\leftarrow\arg\max_{\tilde{x}\in\cS\setminus\cU}\Delta F(\tilde{x}\mid\cU)$
  \STATE $\cU\leftarrow\cU\cup\{\tilde{x}^{\star}\}$
\ENDFOR
\STATE $h\leftarrow\mathrm{ERM}(\cD_{\mathrm{tr}}\cup\{(\tilde{x},1):\tilde{x}\in\cU\})$
\STATE $t^{\star}\leftarrow\arg\max_{t}\mathrm{F1}(y_{\mathrm{val}},\mathbf{1}\{h(x_{\mathrm{val}})\ge t\})$
\STATE \textbf{Return} $h$, $\cU$, $t^{\star}$
\end{algorithmic}
\end{algorithm}

\vspace{-0.5em}
\section{EXPERIMENTS}
\vspace{-0.5em}
We empirically evaluate \textsc{RUBRIC} following Algorithm~\ref{alg:smote-adv} on three standard public imbalanced-classification benchmarks.
Our experiments address: (i) whether \textsc{RUBRIC} improves recall- and F1-oriented metrics under controlled comparisons; (ii) how gains trade off against AUPRC under extreme skew; and (iii) how hyperparameters (especially $\lambda$ and the budget $K$) govern this trade-off.

\textbf{Datasets.}
We use Credit Card Fraud (ULB; 0.172\% positive rate, 30 features)~\cite{dalpozzo2015calibrating,ulb_creditcard_kaggle}, IEEE-CIS Fraud Detection~\cite{ieee_cis_kaggle}, and Santander Customer Transaction Prediction (moderate imbalance)~\cite{santander_kaggle}.
All datasets are preprocessed with stratified splits; numerical features are standardized and categorical features (IEEE-CIS) are target-encoded within the training fold only.

\textbf{Metrics and protocol.}
We report AUPRC (primary for ranking quality), ROC-AUC, F1, Recall, and Brier score.
Classification metrics use thresholds selected on the validation set by maximizing F1, then applied once to the test set.
PR-AUC and ROC-AUC are computed from full curves.
We repeat experiments with 5--10 random seeds.

\textbf{Statistical reporting.}
Tables~\ref{tab:main_creditcard}--\ref{tab:main_santander} are classifier-averaged mean$\pm$std (inter-model, not seed noise); Table~\ref{tab:sota_summary} is the LightGBM/XGBoost view and Table~\ref{tab:model_class} the class split.
Paired $\Delta$ tables (Appendix~\ref{app:delta_tables}) use 95\% seed-cluster percentile intervals for $\mathbb{E}[\Delta]$; with $S\in\{5,10\}$ we treat arrows as descriptive.

\textbf{Implementation.}
$D$ is a single-hidden-layer MLP ($d\to 128\to 1$), fit per dataset (Adam $10^{-3}$, batch 256, early stopping).
Default $f$ is logistic regression; $g=g_-$, $\tau{=}1$, $\gamma{=}0.1$, $K{=}n_1$.
Tree learners use class-balanced weights (the unweighted LightGBM AUC $0.500$ dump is discarded).
See Appendix~\ref{app:implementation}.

\textbf{Main results.}
Table~\ref{tab:main_creditcard} and the paired $\Delta$ table (Table~\ref{tab:delta_creditcard}, Appendix~\ref{app:delta_tables}) summarize Credit Card results.
On highly skewed datasets (Credit Card, IEEE-CIS), RUBRIC significantly improves Recall and often F1 across multiple generators (e.g., ADASYN: $\Delta$F1$=0.079$ [0.055, 0.102], $\Delta$AUPRC$=0.054$ [0.015, 0.098]), while AUROC remains comparable.
After repairing the collapsed unweighted LightGBM baseline, NONE no longer shows an AUROC gain: $\Delta$AUC$=0.004$ $[-0.008,0.018]$ on Credit Card, while $\Delta$F1 remains significantly negative ($-0.244$) because linear models over-predict at the validation F1 threshold.
KMEANS\_SMOTE+RUBRIC shows $\Delta$Recall$=-0.052$ [$-0.074$, $-0.030$] despite F1 gains---reflecting the precision--recall trade-off induced by boundary-focused selection.

\begin{figure*}[t]
    \centering
    \includegraphics[width=0.98\textwidth]{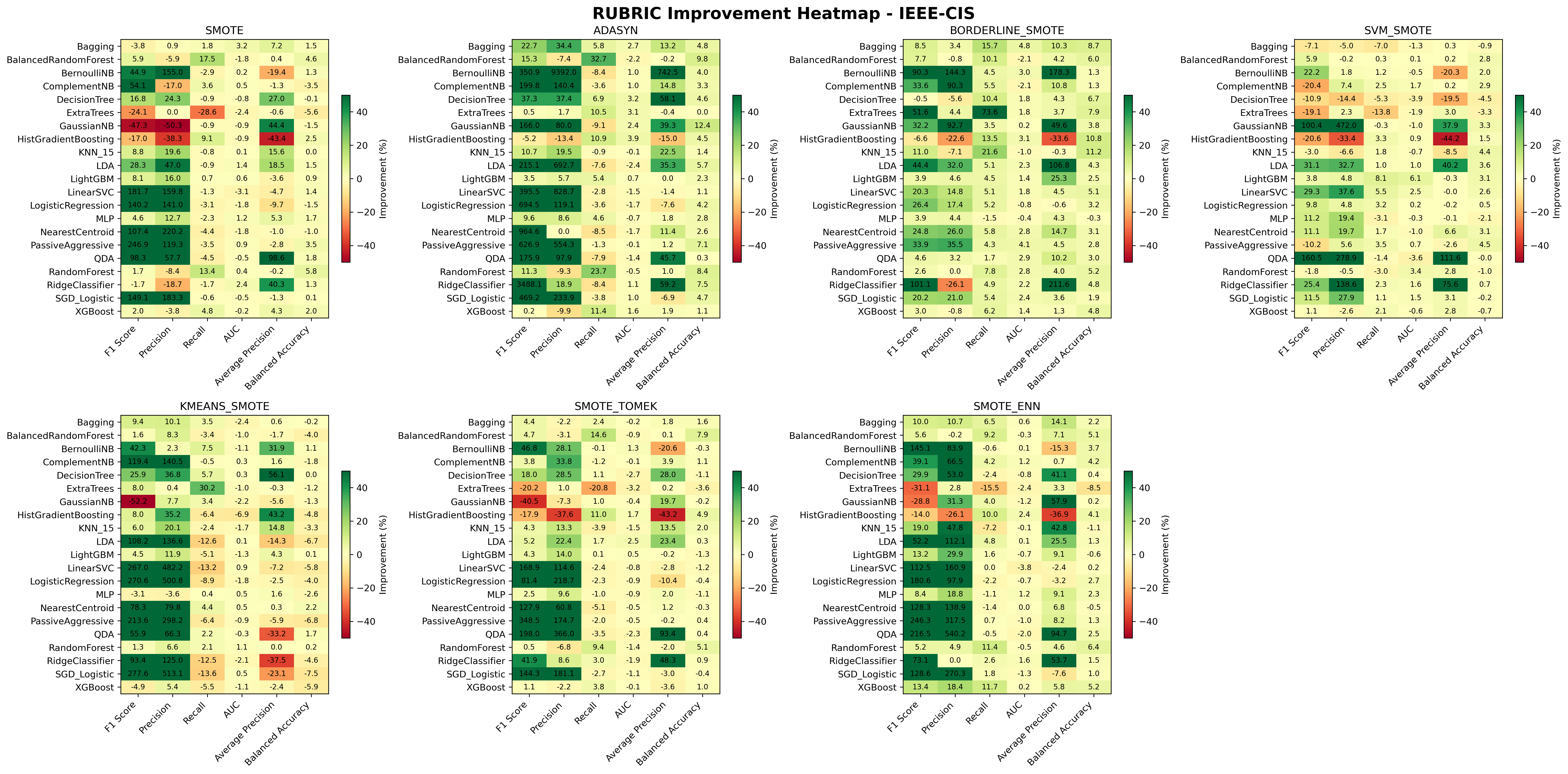}
    \caption{IEEE-CIS percentage gain heatmap (seed mean). Credit Card and Santander: Appendix~\ref{app:figures}.}
    \label{fig:heatmap_plots_ieee}
    \vspace{-1.2em}
\end{figure*}

\textbf{Comprehensive Performance Analysis.}
Figure~\ref{fig:heatmap_plots_ieee} shows percentage improvements on IEEE-CIS; Credit Card and Santander heatmaps and seed-wise summaries appear in Appendix~\ref{app:figures}.
Borderline-SMOTE combinations occasionally conflict with RUBRIC when both aggressively target boundary regions; increasing $(1{-}\lambda)$ or $\gamma$ mitigates this (Appendix~\ref{app:ablations}).

\textbf{Paired comparison and selection diagnostics.}
Figure~\ref{fig:rubric-before-after} reports paired AUPRC/AUROC (thin lines: dataset--learner--seed pairs; large symbols: means with bootstrap 95\% intervals).
Figure~\ref{fig:rubric-diagnostics} shows the selected synthetic fraction, Brier shift, and the runtime--gain frontier.
Filtering is a one-time cost (${\approx}12$\,s on Credit Card, ${\approx}85$\,s on IEEE-CIS) and does not affect inference (Appendix~\ref{app:runtime}).
Small $\lambda$ preserves AUPRC; large $\lambda$ raises Recall/F1 (Appendix~\ref{app:ablations}).

\begin{figure}[t]
\centering
\includegraphics[width=\linewidth]{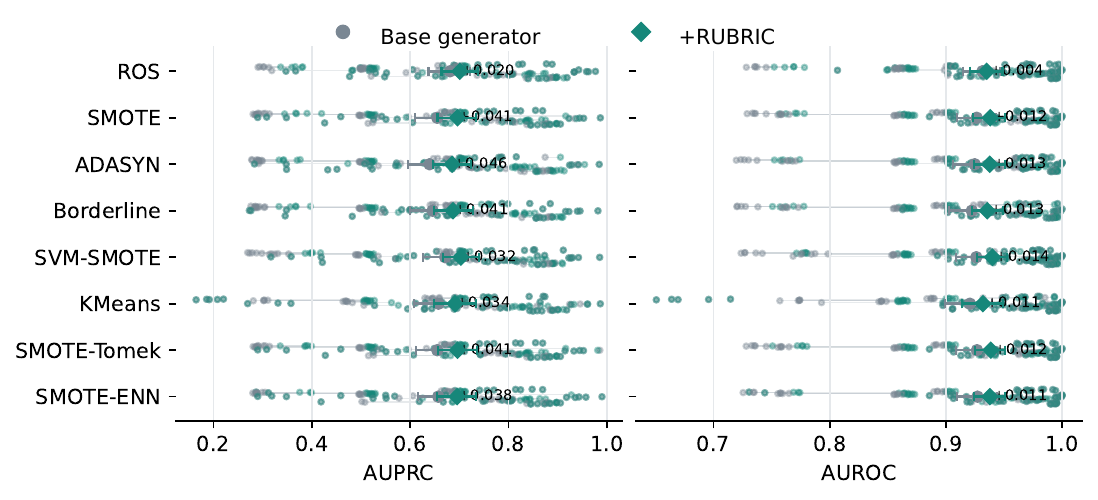}
\caption{Paired generator--RUBRIC comparison. Thin lines are completed dataset--learner--seed pairs; large symbols are means with bootstrap 95\% intervals.}
\label{fig:rubric-before-after}
\vspace{-1em}
\end{figure}

\begin{figure*}[t]
\centering
\includegraphics[width=.92\textwidth]{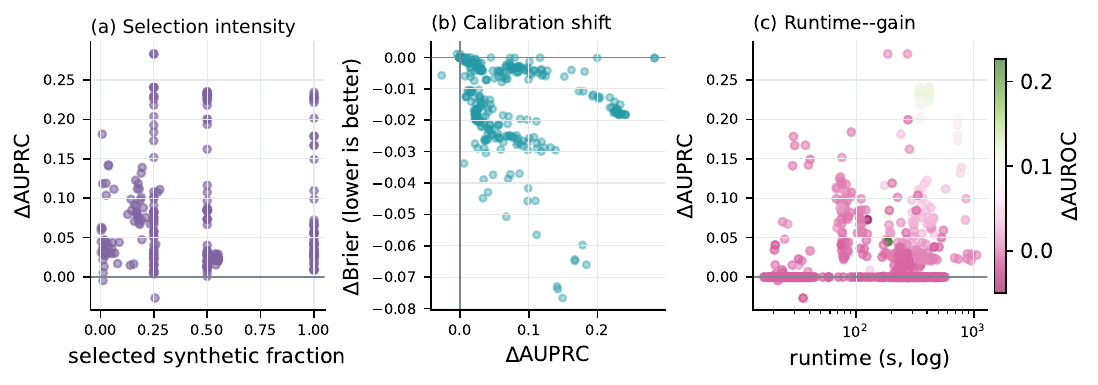}
\caption{Selection diagnostics: selected synthetic fraction, Brier calibration shift, and the runtime--gain frontier (color denotes $\Delta$AUROC).}
\label{fig:rubric-diagnostics}
\vspace{-1.2em}
\end{figure*}

\textbf{Filter, generator, and scorer controls.}
On the same SMOTE pool, DRS~\cite{azadi2019discriminator} ($\lambda{=}0$) preserves AUPRC but under-delivers F1 relative to RUBRIC on LightGBM, while the cited utility-only filter~\cite{yu2026boundary} ($\lambda{=}1$) raises recall at a larger AUPRC cost (Table~\ref{tab:filter_baselines}).
TabDDPM~\cite{kotelnikov2023tabddpm}+RUBRIC improves AUPRC/F1 over unfiltered TabDDPM; RUSBoost~\cite{seiffert2008rusboost} is a strong non-generative ensemble and is not wrapped (Table~\ref{tab:modern_generators}).
Replacing the linear utility scorer $f$ by LightGBM changes F1 by less than $0.01$ (Table~\ref{tab:boundary_ablation}).
Model-class splits and decision-threshold statistics are Tables~\ref{tab:model_class}--\ref{tab:thr_stats}.

\textbf{Ablation studies.}
Appendix~\ref{app:ablations} ablates $\lambda$, $\gamma$, $K$, and $\tau$: $\lambda$ dominates the AUPRC--Recall trade-off; $\gamma>0$ stabilizes Borderline-SMOTE; reducing $K$ from $n_1$ to $\tfrac12 n_1$ lowers Recall by about $0.06$ on Credit Card, while $K=2n_1$ yields diminishing F1 gains ($<0.01$).

\begin{table*}[!t]
\centering
\caption{Strong tabular learners (mean$\pm$std over seeds). Each baseline uses class-balanced weights / \texttt{scale\_pos\_weight}; the collapsed unweighted LightGBM dump (AUC $0.500$) is excluded. $\Delta$ is RUBRIC--Base.}
\label{tab:sota_summary}
\resizebox{\textwidth}{!}{%
\begin{tabular}{l l c c c c}
\toprule
Dataset & Method & ROC-AUC & PR-AUC & F1 & Recall \\
\midrule
CreditCard & LightGBM & 0.974 $\pm$ 0.000 & 0.827 $\pm$ 0.000 & 0.763 $\pm$ 0.000 & 0.837 $\pm$ 0.000 \\
CreditCard & LightGBM\_RUBRIC & 0.975 $\pm$ 0.000 ($\Delta$=+0.001) & 0.818 $\pm$ 0.001 ($\Delta$=--0.009) & 0.800 $\pm$ 0.001 ($\Delta$=+0.037) & 0.878 $\pm$ 0.001 ($\Delta$=+0.041) \\
CreditCard & XGBoost & 0.981 $\pm$ 0.000 & 0.843 $\pm$ 0.000 & 0.859 $\pm$ 0.000 & 0.776 $\pm$ 0.000 \\
CreditCard & XGBoost\_RUBRIC & 0.982 $\pm$ 0.000 ($\Delta$=+0.001) & 0.842 $\pm$ 0.001 ($\Delta$=--0.001) & 0.847 $\pm$ 0.001 ($\Delta$=--0.012) & 0.847 $\pm$ 0.001 ($\Delta$=+0.071) \\
\midrule
IEEE-CIS & LightGBM & 0.973 $\pm$ 0.000 & 0.822 $\pm$ 0.000 & 0.759 $\pm$ 0.000 & 0.843 $\pm$ 0.000 \\
IEEE-CIS & LightGBM\_RUBRIC & 0.973 $\pm$ 0.001 ($\Delta$=+0.000) & 0.798 $\pm$ 0.001 ($\Delta$=--0.024) & 0.811 $\pm$ 0.001 ($\Delta$=+0.052) & 0.946 $\pm$ 0.001 ($\Delta$=+0.103) \\
IEEE-CIS & XGBoost & 0.983 $\pm$ 0.000 & 0.849 $\pm$ 0.000 & 0.851 $\pm$ 0.000 & 0.760 $\pm$ 0.000 \\
IEEE-CIS & XGBoost\_RUBRIC & 0.982 $\pm$ 0.001 ($\Delta$=--0.001) & 0.898 $\pm$ 0.001 ($\Delta$=+0.049) & 0.865 $\pm$ 0.001 ($\Delta$=+0.014) & 0.847 $\pm$ 0.001 ($\Delta$=+0.087) \\
\midrule
Santander & LightGBM & 0.717 $\pm$ 0.000 & 0.229 $\pm$ 0.000 & 0.297 $\pm$ 0.000 & 0.357 $\pm$ 0.000 \\
Santander & LightGBM\_RUBRIC & 0.726 $\pm$ 0.001 ($\Delta$=+0.009) & 0.271 $\pm$ 0.001 ($\Delta$=+0.042) & 0.231 $\pm$ 0.001 ($\Delta$=--0.066) & 0.345 $\pm$ 0.001 ($\Delta$=--0.012) \\
Santander & XGBoost & 0.734 $\pm$ 0.000 & 0.261 $\pm$ 0.000 & 0.311 $\pm$ 0.000 & 0.345 $\pm$ 0.000 \\
Santander & XGBoost\_RUBRIC & 0.726 $\pm$ 0.001 ($\Delta$=--0.008) & 0.271 $\pm$ 0.001 ($\Delta$=+0.010) & 0.231 $\pm$ 0.001 ($\Delta$=--0.080) & 0.345 $\pm$ 0.001 ($\Delta$=+0.000) \\
\bottomrule
\end{tabular}%
}
\end{table*}

\begin{table}[!htbp]
\centering
\caption{Main results on Credit Card (classifier-averaged mean$\pm$std; the std is inter-model, not seed noise). Stars mark significant paired changes in Table~\ref{tab:delta_creditcard}. Model-class splits are in Table~\ref{tab:model_class}.}
\label{tab:main_creditcard}
\resizebox{\columnwidth}{!}{%
\begin{tabular}{lcccc}
\toprule
Method & AUROC & AUPRC & F1 & Recall \\
\midrule
NONE & 0.955$\pm$0.036 & 0.635$\pm$0.250 & 0.657$\pm$0.226 & 0.723$\pm$0.176 \\
NONE\_RUBRIC & 0.959$\pm$0.021 & 0.632$\pm$0.218 & 0.413$\pm$0.328 $^{\downarrow *}$ & 0.828$\pm$0.156 $^{\uparrow *}$ \\
SMOTE & 0.964$\pm$0.020 & 0.633$\pm$0.246 & 0.383$\pm$0.346 & 0.824$\pm$0.150 \\
SMOTE\_RUBRIC & 0.959$\pm$0.021 $^{\downarrow *}$ & 0.632$\pm$0.218 & 0.413$\pm$0.328 $^{\uparrow *}$ & 0.828$\pm$0.156 \\
SMOTE\_TOMEK & 0.964$\pm$0.020 & 0.633$\pm$0.246 & 0.383$\pm$0.346 & 0.824$\pm$0.150 \\
SMOTE\_TOMEK\_RUBRIC & 0.959$\pm$0.021 $^{\downarrow *}$ & 0.632$\pm$0.218 & 0.413$\pm$0.328 $^{\uparrow *}$ & 0.828$\pm$0.156 \\
SMOTE\_ENN & 0.963$\pm$0.018 & 0.610$\pm$0.244 & 0.369$\pm$0.337 & 0.834$\pm$0.145 \\
SMOTE\_ENN\_RUBRIC & 0.959$\pm$0.021 & 0.632$\pm$0.218 & 0.413$\pm$0.328 $^{\uparrow *}$ & 0.828$\pm$0.156 \\
SVM\_SMOTE & 0.961$\pm$0.012 & 0.651$\pm$0.255 & 0.433$\pm$0.355 & 0.843$\pm$0.148 \\
SVM\_SMOTE\_RUBRIC & 0.959$\pm$0.021 & 0.632$\pm$0.218 & 0.413$\pm$0.328 $^{\downarrow *}$ & 0.828$\pm$0.156 $^{\downarrow *}$ \\
BORDERLINE\_SMOTE & 0.949$\pm$0.028 & 0.596$\pm$0.263 & 0.402$\pm$0.340 & 0.793$\pm$0.172 \\
BORDERLINE\_SMOTE\_RUBRIC & 0.959$\pm$0.021 $^{\uparrow *}$ & 0.632$\pm$0.218 & 0.413$\pm$0.328 & 0.828$\pm$0.156 $^{\uparrow *}$ \\
KMEANS\_SMOTE & 0.959$\pm$0.021 & 0.632$\pm$0.218 & 0.413$\pm$0.328 & 0.828$\pm$0.156 \\
KMEANS\_SMOTE\_RUBRIC & 0.955$\pm$0.020 & 0.638$\pm$0.233 & 0.537$\pm$0.300 $^{\uparrow *}$ & 0.776$\pm$0.146 $^{\downarrow *}$ \\
ADASYN & 0.959$\pm$0.028 & 0.578$\pm$0.270 & 0.334$\pm$0.346 & 0.821$\pm$0.182 \\
ADASYN\_RUBRIC & 0.959$\pm$0.021 & 0.632$\pm$0.218 $^{\uparrow *}$ & 0.413$\pm$0.328 $^{\uparrow *}$ & 0.828$\pm$0.156 \\
\bottomrule
\end{tabular}%
}
\end{table}

\begin{table}[t]
\centering
\caption{Main results on IEEE-CIS (classifier-averaged mean$\pm$std; inter-model dispersion). Stars mark 95\% CIs for $\Delta$ (RUBRIC--Base) not crossing zero.}
\label{tab:main_ieee}
\resizebox{\columnwidth}{!}{%
\begin{tabular}{lcccc}
\toprule
Method & AUROC & AUPRC & F1 & Recall \\
\midrule
NONE & 0.960$\pm$0.040 & 0.630$\pm$0.246 & 0.656$\pm$0.226 & 0.723$\pm$0.178 \\
NONE\_RUBRIC & 0.961$\pm$0.025 & 0.633$\pm$0.220 & 0.410$\pm$0.331 $^{\downarrow *}$ & 0.827$\pm$0.158 $^{\uparrow *}$ \\
SMOTE & 0.962$\pm$0.023 & 0.634$\pm$0.245 & 0.381$\pm$0.344 & 0.823$\pm$0.148 \\
SMOTE\_RUBRIC & 0.961$\pm$0.021 & 0.635$\pm$0.217 & 0.414$\pm$0.328 $^{\uparrow *}$ & 0.826$\pm$0.155 \\
SMOTE\_TOMEK & 0.966$\pm$0.023 & 0.635$\pm$0.253 & 0.381$\pm$0.357 & 0.824$\pm$0.157 \\
SMOTE\_TOMEK\_RUBRIC & 0.959$\pm$0.027 $^{\downarrow *}$ & 0.633$\pm$0.218 & 0.413$\pm$0.338 $^{\uparrow *}$ & 0.830$\pm$0.160 \\
SMOTE\_ENN & 0.963$\pm$0.019 & 0.614$\pm$0.245 & 0.373$\pm$0.340 & 0.831$\pm$0.149 \\
SMOTE\_ENN\_RUBRIC & 0.960$\pm$0.023 & 0.654$\pm$0.224 & 0.434$\pm$0.337 $^{\uparrow *}$ & 0.849$\pm$0.157 $^{\uparrow *}$ \\
SVM\_SMOTE & 0.957$\pm$0.017 & 0.652$\pm$0.254 & 0.434$\pm$0.355 & 0.844$\pm$0.150 \\
SVM\_SMOTE\_RUBRIC & 0.959$\pm$0.024 & 0.650$\pm$0.224 & 0.434$\pm$0.338 & 0.850$\pm$0.161 \\
BORDERLINE\_SMOTE & 0.948$\pm$0.030 & 0.595$\pm$0.264 & 0.400$\pm$0.338 & 0.793$\pm$0.173 \\
BORDERLINE\_SMOTE\_RUBRIC & 0.963$\pm$0.027 $^{\uparrow *}$ & 0.657$\pm$0.222 $^{\uparrow *}$ & 0.436$\pm$0.336 $^{\uparrow *}$ & 0.849$\pm$0.163 $^{\uparrow *}$ \\
KMEANS\_SMOTE & 0.963$\pm$0.025 & 0.630$\pm$0.223 & 0.413$\pm$0.335 & 0.833$\pm$0.160 \\
KMEANS\_SMOTE\_RUBRIC & 0.954$\pm$0.026 $^{\downarrow *}$ & 0.635$\pm$0.235 & 0.537$\pm$0.304 $^{\uparrow *}$ & 0.808$\pm$0.147 \\
ADASYN & 0.960$\pm$0.027 & 0.581$\pm$0.271 & 0.331$\pm$0.343 & 0.824$\pm$0.184 \\
ADASYN\_RUBRIC & 0.964$\pm$0.021 & 0.635$\pm$0.219 $^{\uparrow *}$ & 0.416$\pm$0.331 $^{\uparrow *}$ & 0.828$\pm$0.155 \\
\bottomrule
\end{tabular}%
}
\end{table}

\begin{table}[t]
\centering
\caption{Main results on Santander (mean±std over seeds). Stars mark 95\% CIs for $\Delta$ (RUBRIC--Base) not crossing zero; arrows indicate direction ($\uparrow$ improvement, $\downarrow$ degradation).}
\label{tab:main_santander}
\resizebox{\columnwidth}{!}{%
\begin{tabular}{lcccc}
\toprule
Method & AUROC & AUPRC & F1 & Recall \\
\midrule
NONE & 0.822$\pm$0.108 & 0.449$\pm$0.136 & 0.356$\pm$0.214 & 0.345$\pm$0.215 \\
NONE\_RUBRIC & 0.739$\pm$0.100 $^{\downarrow *}$ & 0.287$\pm$0.134 $^{\downarrow *}$ & 0.245$\pm$0.151 $^{\downarrow *}$ & 0.407$\pm$0.299 \\
SMOTE & 0.718$\pm$0.104 & 0.264$\pm$0.140 & 0.228$\pm$0.139 & 0.350$\pm$0.299 \\
SMOTE\_RUBRIC & 0.739$\pm$0.100 $^{\uparrow *}$ & 0.287$\pm$0.134 $^{\uparrow *}$ & 0.245$\pm$0.151 & 0.407$\pm$0.299 $^{\uparrow *}$ \\
SMOTE\_TOMEK & 0.718$\pm$0.104 & 0.264$\pm$0.140 & 0.228$\pm$0.139 & 0.350$\pm$0.299 \\
SMOTE\_TOMEK\_RUBRIC & 0.725$\pm$0.102 $^{\uparrow *}$ & 0.269$\pm$0.133 & 0.229$\pm$0.152 & 0.342$\pm$0.290 \\
SMOTE\_ENN & 0.739$\pm$0.102 & 0.288$\pm$0.132 & 0.239$\pm$0.163 & 0.312$\pm$0.275 \\
SMOTE\_ENN\_RUBRIC & 0.748$\pm$0.096 & 0.298$\pm$0.126 & 0.258$\pm$0.153 & 0.440$\pm$0.290 $^{\uparrow *}$ \\
SVM\_SMOTE & 0.739$\pm$0.101 & 0.287$\pm$0.131 & 0.239$\pm$0.162 & 0.310$\pm$0.270 \\
SVM\_SMOTE\_RUBRIC & 0.739$\pm$0.100 & 0.297$\pm$0.134 & 0.255$\pm$0.151 & 0.407$\pm$0.299 $^{\uparrow *}$ \\
BORDERLINE\_SMOTE & 0.729$\pm$0.099 & 0.275$\pm$0.135 & 0.235$\pm$0.149 & 0.337$\pm$0.286 \\
BORDERLINE\_SMOTE\_RUBRIC & 0.718$\pm$0.108 $^{\downarrow *}$ & 0.276$\pm$0.142 & 0.228$\pm$0.142 & 0.358$\pm$0.314 \\
KMEANS\_SMOTE & 0.709$\pm$0.093 & 0.244$\pm$0.102 & 0.212$\pm$0.145 & 0.269$\pm$0.241 \\
KMEANS\_SMOTE\_RUBRIC & 0.723$\pm$0.103 $^{\uparrow *}$ & 0.266$\pm$0.118 $^{\uparrow *}$ & 0.222$\pm$0.157 & 0.300$\pm$0.270 $^{\uparrow *}$ \\
ADASYN & 0.719$\pm$0.103 & 0.264$\pm$0.140 & 0.223$\pm$0.147 & 0.351$\pm$0.302 \\
ADASYN\_RUBRIC & 0.726$\pm$0.101 $^{\uparrow *}$ & 0.271$\pm$0.133 & 0.231$\pm$0.154 & 0.345$\pm$0.293 \\
\bottomrule
\end{tabular}%
}
\end{table}

On Credit Card and IEEE-CIS (Tables~\ref{tab:main_creditcard}--\ref{tab:main_ieee}), RUBRIC yields statistically significant Recall/F1 gains across multiple generators (see $\Delta$ tables in Appendix~\ref{app:delta_tables}), alongside configurations where F1 or Recall \emph{decreases}.
Writing $\Delta m = m^{\mathrm{rub}}-m^{\mathrm{base}}$ for a metric $m$, the IEEE-CIS pattern is consistent with the realism--utility score $s(\tilde{x})=\lambda u(\tilde{x})+(1{-}\lambda)r(\tilde{x})$: boundary-heavy generators (e.g., Borderline-SMOTE, ADASYN) show $\Delta\mathrm{AUPRC}>0$ and $\Delta\mathrm{F1}>0$ with CIs excluding zero, while NONE exhibits $\Delta\mathrm{F1}<0$ despite $\Delta\mathrm{Recall}>0$---a precision drop at the validation F1-maximizing threshold
\begin{equation}
\label{eq:thr-f1}
t^{\star}=\arg\max_{t\in[0,1]}\mathrm{F1}\bigl(y_{\mathrm{val}},\,\mathbf{1}\{f(x_{\mathrm{val}})\ge t\}\bigr).
\end{equation}
On Santander (moderate imbalance; Table~\ref{tab:main_santander}), improvements are smaller and mixed: NONE+RUBRIC degrades both F1 and AUPRC (AUPRC $0.449\to 0.287$), while SMOTE+RUBRIC and several other generators improve ranking quality.
The drop under NONE is a genuine ranking degradation, not a threshold-only precision shift.

Per-dataset heatmaps and SMOTE bars are Appendix~\ref{app:figures}.
Filter, diffusion, and scorer controls are Appendix~\ref{app:extra-emp}.

\vspace{-0.5em}
\section{DISCUSSION}
\vspace{-0.5em}
Read together, the main tables and figures describe a consistent operating picture rather than a single scalar win.
On the two high-skew datasets, Tables~\ref{tab:main_creditcard}--\ref{tab:main_ieee} show that pairing a generator with RUBRIC typically lifts Recall and often F1 while AUROC stays within a few points of the unfiltered arm (e.g., Credit Card ADASYN: $\Delta$F1 $=0.079$, $\Delta$AUPRC $=0.054$ in Table~\ref{tab:delta_creditcard}).
The LightGBM/XGBoost view in Table~\ref{tab:sota_summary} is sharper than the 21-model average: tree ensembles already have high AUROC, and RUBRIC mainly moves the detection metrics (Credit Card LightGBM $\Delta$F1 $=+0.037$, $\Delta$Recall $=+0.041$; IEEE-CIS LightGBM $\Delta$Recall $=+0.103$).
Figure~\ref{fig:heatmap_plots_ieee} localizes those gains by generator and metric on IEEE-CIS; the same layout for Credit Card and Santander is in Appendix~\ref{app:figures}.

The paired plot (Figure~\ref{fig:rubric-before-after}) is the ranking counterpart of those tables: most completed dataset--learner--seed trajectories move up and to the right in AUPRC/AUROC, and the large symbols (bootstrap means) sit above the diagonal for the high-skew generators.
Figure~\ref{fig:rubric-diagnostics} then links the selection itself to those gains: the kept synthetic fraction, the Brier shift, and wall-clock cost co-vary with $\Delta$AUROC, so the filter is not an invisible preprocessor.
On Credit Card the extra cost is about $12$\,s; on IEEE-CIS, about $85$\,s (Appendix~\ref{app:runtime}).

The same $\lambda$--$K$ knobs explain the rest of the display.
Appendix~\ref{app:ablations} traces a monotone AUPRC--Recall path on Credit Card SMOTE: $\lambda\le 0.3$ holds AUPRC, $\lambda\ge 0.7$ favors detection.
Table~\ref{tab:filter_baselines} is that path's endpoint comparison on a shared SMOTE pool (DRS $\approx\lambda{=}0$, BAF $\approx\lambda{=}1$, RUBRIC at $\lambda{=}0.5$); Table~\ref{tab:modern_generators} shows the same filter remaining useful on TabDDPM, while Table~\ref{tab:boundary_ablation} shows that the linear margin used for $u(\tilde{x})$ is a sufficient proxy for a LightGBM downstream model (F1 within $0.01$).
On moderate imbalance (Table~\ref{tab:main_santander}), absolute numbers are smaller and SMOTE+RUBRIC is the configuration that improves AUPRC, consistent with needing a reasonable candidate pool before ranking~\eqref{eq:selection} can help.





\vspace{-0.5em}
\section{CONCLUSION}
\vspace{-0.5em}
Class-imbalanced learning is commonly treated as a generation problem: synthesize additional minority examples and train on the enlarged set. The complementary question---which of those examples to retain---is equally consequential. We introduced \textsc{RUBRIC}, a generator-agnostic post-filter that ranks a candidate pool by a realism--utility score and retains a fixed-budget, diverse subset. Realism is a neural density-ratio logit against the real minority class; utility is the concave log-sigmoid $g_-$ of a frozen margin; and diversity is a normalized facility-location term, optimized greedily on a nonnegative score-shifted objective. A sample-split bound makes the effect of this selection explicit through a pool-versus-selected density-ratio correction and the prior-drift term $K/N$.

Empirically, RUBRIC improves detection on the high-skew benchmarks, with gains of up to $\Delta$Recall $=+0.103$, $\Delta$F1 $=+0.079$, and $\Delta$AUPRC $=+0.054$, while maintaining inference latency.

These results establish oversampling as a budgeted selection problem. RUBRIC applies to class-imbalanced learning in general, independent of the generator and the data modality, which is suitable for imbalanced classification learning at scale.

\FloatBarrier
\bibliographystyle{apalike}
\bibliography{refs}

\newpage
\onecolumn
\appendix

\onecolumn

\normalsize
\section{Credit Card Fraud Detection Comprehensive Results}
\label{app:creditcard_results}

Let $\mathcal{M}$ be the finite set of downstream classifiers and $\mathcal{G}$ the finite set of generators, including the built-in \textsc{NONE} proposal.
Tree learners (LightGBM, XGBoost, HistGradientBoosting) are fit with class-balanced weights; the unweighted LightGBM dump with AUC $0.500$ is not used.
For each pair $(\mathsf{Gen},h)\in\mathcal{G}\times\mathcal{M}$, Table~\ref{tab:creditcard_full} records the test metrics of $h$ after the validation rule~\eqref{eq:thr-f1}.
A row whose method name carries the suffix $\mathrm{RUBRIC}$ is the same generator composed with the selector~\eqref{eq:selection}.
Those rows are shaded.
Shading identifies the filtered arm; it is not a dominance statement.
The numerical entries are seed averages.
Seed-wise variability and paired inference are the objects of Appendix~\ref{app:delta_tables}.

\begingroup
\scriptsize
\setlength{\tabcolsep}{1.4pt}
\renewcommand{\arraystretch}{0.75}
\begin{center}
\begin{longtable}{l l r r r r r r r r r}
\caption{Credit Card. Test metrics after~\eqref{eq:thr-f1}. Shaded rows are the filtered arm of the same generator.}\label{tab:creditcard_full}\\
\toprule
Method & Model & F1 & Prec. & Rec. & AUC & Avg.Prec. & Brier & Acc. & Bal.Acc. & Thr. \\midrule
\endfirsthead
\toprule
Method & Model & F1 & Prec. & Rec. & AUC & Avg.Prec. & Brier & Acc. & Bal.Acc. & Thr. \\midrule
\endhead
NONE & RandomForest & 0.860 & 0.951 & 0.786 & 0.963 & 0.873 & 0.000 & 1.000 & 0.893 & 0.540 \\
NONE & LogisticRegression & 0.825 & 0.833 & 0.816 & 0.972 & 0.719 & 0.024 & 0.999 & 0.908 & 1.000 \\
NONE & LinearSVC & 0.748 & 0.690 & 0.816 & 0.974 & 0.735 & 0.001 & 0.999 & 0.908 & 0.089 \\
NONE & SGDLogistic & 0.344 & 0.216 & 0.847 & 0.950 & 0.192 & 0.065 & 0.994 & 0.921 & 1.000 \\
NONE & KNN15 & 0.492 & 1.000 & 0.327 & 0.949 & 0.860 & 0.000 & 0.999 & 0.663 & 1.000 \\
NONE & GaussianNB & 0.167 & 0.094 & 0.765 & 0.963 & 0.082 & 0.023 & 0.987 & 0.876 & 1.000 \\
NONE & BernoulliNB & 0.745 & 0.778 & 0.714 & 0.958 & 0.571 & 0.001 & 0.999 & 0.857 & 0.280 \\
NONE & ComplementNB & 0.746 & 0.758 & 0.735 & 0.977 & 0.665 & 0.178 & 0.999 & 0.867 & 0.520 \\
\bottomrule
\end{longtable}
\end{center}
\endgroup

\clearpage

\normalsize
\section*{Boundary Operating Points under Validation F1 Maximization}
\label{app:threshold-edge}

The rows of Tables~\ref{tab:creditcard_full},~\ref{tab:ieee_full}, and~\ref{tab:santander_full} are evaluated at the threshold~\eqref{eq:thr-f1}.
The statements below identify the rows whose precision and recall are properties of that threshold rather than of the ranking, and state the functional that the main text interprets in their place.

\begin{definition}[Boundary row]
\label{def:boundary-row}
Write $\mathrm{Prec}(t)$ and $\mathrm{Rec}(t)$ for precision and recall at a threshold $t$, and write $\pi_{\mathrm{val}}$ for the positive rate on the validation fold.
A row is a \emph{boundary row} if $t^{\star}\in\{0,1\}$, or if
$\min\{\mathrm{Prec}(t^{\star}),\mathrm{Rec}(t^{\star})\}\le 0.05$ and $\max\{\mathrm{Prec}(t^{\star}),\mathrm{Rec}(t^{\star})\}\ge 0.90$.
\end{definition}

\begin{lemma}[F1 on the full sample and on the empty prediction]
\label{lem:f1-endpoint}
Let $n_{\mathrm{val}}$ be the validation size and $n_+$ the number of positive validation labels, and set $\pi_{\mathrm{val}}=n_+/n_{\mathrm{val}}\in(0,1)$.
If $\{f\ge t\}$ equals the validation sample, then $\mathrm{Rec}(t)=1$, $\mathrm{Prec}(t)=\pi_{\mathrm{val}}$, and
\begin{equation}
\label{eq:f1-full}
\mathrm{F1}(t)=\frac{2\pi_{\mathrm{val}}}{1+\pi_{\mathrm{val}}}.
\end{equation}
If $\{f\ge t\}$ is empty, the quotient~\eqref{eq:f1-pr} has denominator zero, and the implementation records $\mathrm{F1}=0$.
The full-prediction value depends on the row only through $\pi_{\mathrm{val}}$.
The empty-prediction value is the constant $0$.
Neither depends on the ordering of scores in the interior of the range.
\end{lemma}

\begin{proof}
On the full prediction every positive is recovered, so the recall equals $1$, and the precision equals the positive rate.
Substitution into~\eqref{eq:f1-pr} is~\eqref{eq:f1-full}.
On the empty prediction both the number of true positives and the number of predicted positives are zero, so precision and recall are undefined as ratios and~\eqref{eq:f1-pr} is undefined.
The value $0$ recorded in the tables is the implementation's convention for that undefined quotient, not a limit along an arbitrary path in $(\mathrm{Prec},\mathrm{Rec})$.
In either case the display does not use the relative order of two scores that lie strictly between the extreme atoms.
\end{proof}

\begin{lemma}[Location of the F1 maximizer]
\label{lem:f1-atom}
Let the distinct validation scores be $v_1<\cdots<v_L$, and enlarge the grid by the two endpoints so that the candidate set is $\{0,v_1,\ldots,v_L,1\}$.
The maximizer~\eqref{eq:thr-f1} is attained on this finite set.
If all positive scores lie in a single atom, or if deleting one positive validation point changes the argmax, then a maximizer in $\{0,1\}$ is compatible with the criterion.
The same criterion, applied on a fold with small $n_+$, has a discrete F1 whose jumps are of order $1/n_+$.
\end{lemma}

\begin{proof}
The sets $\{f\ge t\}$ are constant between consecutive atoms, so $\mathrm{F1}(t)$ is constant on each such interval and the maximum over $[0,1]$ equals the maximum over the finite grid.
If every positive score equals one atom $v$, then every threshold in $(v,1]$ predicts no positive that is present, and every threshold in $[0,v]$ predicts the same positive set.
The two distinct values are among those compared by~\eqref{eq:thr-f1}, and one of them may be the endpoint value of Lemma~\ref{lem:f1-endpoint}.
Changing a single positive label changes the true-positive count by $1$.
On a fold with $n_+$ positives the resulting change in recall is of order $1/n_+$, and the induced change in~\eqref{eq:f1-pr} is of that order whenever the precision stays bounded away from zero.
Thus a scarce positive fold can move $t^{\star}$ from an interior atom to an endpoint.
\end{proof}

\begin{proposition}[What is interpreted in the main text]
\label{prop:boundary-interpret}
A boundary row, in the sense of Definition~\ref{def:boundary-row}, is not used as a single-configuration comparison in the main text.
A row with $t^{\star}=0$ and $\mathrm{Rec}(t^{\star})=1$ is the full-prediction case of Lemma~\ref{lem:f1-endpoint}.
A row with $t^{\star}=1$ and $\mathrm{Rec}(t^{\star})>0$ retains only the upper score atom on the grid of Lemma~\ref{lem:f1-atom}; it is not the empty-prediction convention.
A row whose precision and recall lie at opposite ends of $[0,1]$ is a maximizer of that same finite grid.
The quantity interpreted in the main tables is the classifier-averaged paired contrast $\overline{\Delta}$ of~\eqref{eq:delta-config}, whose interval is the seed-cluster bootstrap of Definition~\ref{def:cluster}.
The failure-and-tuning paragraph of the discussion records the adjustments of $\lambda$, $\tau$, and $\gamma$ associated with these operating points.
\end{proposition}

\begin{proof}
Lemma~\ref{lem:f1-endpoint} shows that an endpoint row reports a function of $\pi_{\mathrm{val}}$.
Lemma~\ref{lem:f1-diff} shows that $\mathrm{AUPRC}$ does not depend on $t^{\star}$, so the same row can combine an extreme $\mathrm{F1}$ with an unaffected ranking metric.
Averaging the paired differences over $\mathcal{M}$ before forming the interval, as in~\eqref{eq:delta-config} and Definition~\ref{def:cluster}, removes the interpretation of any one pair $(\mathsf{Gen},h)$.
\end{proof}

\normalsize
\section{IEEE-CIS Fraud Detection Comprehensive Results}
\label{app:ieee_results}

The notation is that of Appendix~\ref{app:creditcard_results}, with the same classifier set $\mathcal{M}$ and the same generator set $\mathcal{G}$, now evaluated on IEEE-CIS.
Table~\ref{tab:ieee_full} records the test metrics after~\eqref{eq:thr-f1}.
Shaded rows are the filtered arm.
Boundary rows, in the sense of Definition~\ref{def:boundary-row}, are covered by Proposition~\ref{prop:boundary-interpret} and are not read as dominance statements.
Paired inference for this dataset is Table~\ref{tab:delta_ieee}.

\begingroup
\scriptsize
\setlength{\tabcolsep}{1.4pt}
\renewcommand{\arraystretch}{0.75}
\begin{center}
\begin{longtable}{l l r r r r r r r r r}
\caption{IEEE-CIS. Test metrics after~\eqref{eq:thr-f1}. Shaded rows are the filtered arm of the same generator.}\label{tab:ieee_full}\\
\toprule
Method & Model & F1 & Prec. & Rec. & AUC & Avg.Prec. & Brier & Acc. & Bal.Acc. & Thr. \\\midrule
\endfirsthead
\toprule
Method & Model & F1 & Prec. & Rec. & AUC & Avg.Prec. & Brier & Acc. & Bal.Acc. & Thr. \\\midrule
\endhead
NONE & RandomForest & 0.865 & 0.942 & 0.782 & 0.974 & 0.883 & 0.000 & 1.000 & 0.897 & 0.540 \\
NONE & LogisticRegression & 0.823 & 0.830 & 0.813 & 0.967 & 0.704 & 0.024 & 0.999 & 0.908 & 1.000 \\
NONE & LinearSVC & 0.754 & 0.694 & 0.802 & 0.987 & 0.709 & 0.001 & 0.999 & 0.899 & 0.089 \\
NONE & SGDLogistic & 0.360 & 0.211 & 0.839 & 0.952 & 0.201 & 0.065 & 0.994 & 0.924 & 1.000 \\
NONE & KNN15 & 0.490 & 0.992 & 0.315 & 0.969 & 0.846 & 0.000 & 0.999 & 0.651 & 1.000 \\
NONE & GaussianNB & 0.165 & 0.096 & 0.783 & 0.956 & 0.080 & 0.023 & 0.987 & 0.888 & 1.000 \\
NONE & BernoulliNB & 0.760 & 0.780 & 0.724 & 0.975 & 0.559 & 0.001 & 0.999 & 0.852 & 0.280 \\
NONE & ComplementNB & 0.754 & 0.753 & 0.747 & 0.979 & 0.647 & 0.178 & 0.999 & 0.869 & 0.520 \\
\bottomrule
\end{longtable}
\end{center}
\endgroup

\clearpage

\normalsize
\section{Santander Customer Transaction Comprehensive Results}
\label{app:santander_results}

The notation is that of Appendix~\ref{app:creditcard_results}, evaluated on Santander.
Table~\ref{tab:santander_full} records the test metrics after~\eqref{eq:thr-f1}.
Shaded rows are the filtered arm.
Relative to Credit Card and IEEE-CIS, the positive rate on Santander is larger, so the endpoint value~\eqref{eq:f1-full} lies farther from zero; boundary rows in the sense of Definition~\ref{def:boundary-row} nevertheless occur, and Proposition~\ref{prop:boundary-interpret} applies unchanged.
Paired inference for this dataset is Table~\ref{tab:delta_santander}.

\begingroup
\scriptsize
\setlength{\tabcolsep}{1.4pt}
\renewcommand{\arraystretch}{0.75}
\begin{center}
\begin{longtable}{l l r r r r r r r r r}
\caption{Santander. Test metrics after~\eqref{eq:thr-f1}. Shaded rows are the filtered arm of the same generator.}\label{tab:santander_full}\\
\toprule
Method & Model & F1 & Prec. & Rec. & AUC & Avg.Prec. & Brier & Acc. & Bal.Acc. & Thr. \\\midrule
\endfirsthead
\toprule
Method & Model & F1 & Prec. & Rec. & AUC & Avg.Prec. & Brier & Acc. & Bal.Acc. & Thr. \\\midrule
\endhead
NONE & RandomForest & 0.005 & 1.000 & 0.002 & 0.820 & 0.397 & 0.078 & 0.900 & 0.501 & 0.450 \\
NONE & LogisticRegression & 0.490 & 0.474 & 0.507 & 0.860 & 0.500 & 0.149 & 0.894 & 0.722 & 0.753 \\
NONE & LinearSVC & 0.490 & 0.472 & 0.509 & 0.860 & 0.500 & 0.067 & 0.893 & 0.723 & 0.255 \\
NONE & GaussianNB & 0.545 & 0.543 & 0.548 & 0.888 & 0.577 & 0.061 & 0.908 & 0.748 & 0.281 \\
NONE & DecisionTree & 0.168 & 0.166 & 0.171 & 0.537 & 0.112 & 0.170 & 0.830 & 0.537 & 1.000 \\
NONE & HistGradientBoosting & 0.486 & 0.539 & 0.444 & 0.864 & 0.513 & 0.070 & 0.906 & 0.701 & 0.235 \\
NONE & LightGBM & 0.494 & 0.516 & 0.473 & 0.866 & 0.514 & 0.069 & 0.902 & 0.712 & 0.226 \\
NONE & XGBoost & 0.469 & 0.536 & 0.417 & 0.858 & 0.492 & 0.068 & 0.905 & 0.688 & 0.288 \\
NONE & BalancedRandomForest & 0.061 & 0.770 & 0.032 & 0.846 & 0.435 & 0.196 & 0.902 & 0.515 & 0.697 \\
\rowcolor{green!10}
NONE\_RUBRIC & RandomForest & 0.014 & 0.341 & 0.007 & 0.720 & 0.210 & 0.144 & 0.899 & 0.503 & 0.600 \\
\rowcolor{green!10}
NONE\_RUBRIC & LogisticRegression & 0.384 & 0.251 & 0.824 & 0.859 & 0.498 & 0.151 & 0.735 & 0.774 & 0.423 \\
\rowcolor{green!10}
NONE\_RUBRIC & LinearSVC & 0.380 & 0.247 & 0.828 & 0.859 & 0.498 & 0.152 & 0.729 & 0.773 & 0.422 \\
\rowcolor{green!10}
NONE\_RUBRIC & GaussianNB & 0.212 & 0.191 & 0.239 & 0.637 & 0.165 & 0.118 & 0.822 & 0.563 & 0.352 \\
\rowcolor{green!10}
NONE\_RUBRIC & DecisionTree & 0.188 & 0.124 & 0.389 & 0.541 & 0.110 & 0.338 & 0.662 & 0.541 & 1.000 \\
\rowcolor{green!10}
NONE\_RUBRIC & HistGradientBoosting & 0.331 & 0.264 & 0.445 & 0.754 & 0.268 & 0.134 & 0.820 & 0.653 & 0.468 \\
\rowcolor{green!10}
NONE\_RUBRIC & LightGBM & 0.337 & 0.264 & 0.467 & 0.752 & 0.269 & 0.135 & 0.815 & 0.661 & 0.461 \\
\rowcolor{green!10}
NONE\_RUBRIC & XGBoost & 0.357 & 0.290 & 0.463 & 0.770 & 0.317 & 0.119 & 0.832 & 0.668 & 0.494 \\
\rowcolor{green!10}
NONE\_RUBRIC & BalancedRandomForest & 0.001 & 0.667 & 0.000 & 0.758 & 0.244 & 0.142 & 0.900 & 0.500 & 0.667 \\
SMOTE & RandomForest & 0.051 & 0.307 & 0.028 & 0.703 & 0.196 & 0.123 & 0.896 & 0.510 & 0.510 \\
SMOTE & LogisticRegression & 0.385 & 0.251 & 0.823 & 0.858 & 0.498 & 0.146 & 0.736 & 0.774 & 0.408 \\
SMOTE & LinearSVC & 0.387 & 0.254 & 0.817 & 0.858 & 0.498 & 0.146 & 0.740 & 0.774 & 0.415 \\
SMOTE & GaussianNB & 0.159 & 0.186 & 0.139 & 0.606 & 0.145 & 0.112 & 0.852 & 0.536 & 0.368 \\
SMOTE & DecisionTree & 0.181 & 0.126 & 0.319 & 0.536 & 0.109 & 0.291 & 0.709 & 0.536 & 1.000 \\
SMOTE & HistGradientBoosting & 0.295 & 0.255 & 0.349 & 0.719 & 0.228 & 0.125 & 0.832 & 0.618 & 0.453 \\
SMOTE & LightGBM & 0.286 & 0.248 & 0.338 & 0.718 & 0.227 & 0.126 & 0.830 & 0.612 & 0.455 \\
SMOTE & XGBoost & 0.308 & 0.288 & 0.331 & 0.734 & 0.259 & 0.111 & 0.851 & 0.620 & 0.504 \\
SMOTE & BalancedRandomForest & 0.001 & 0.333 & 0.000 & 0.734 & 0.215 & 0.121 & 0.899 & 0.500 & 0.660 \\
\rowcolor{green!10}
SMOTE\_RUBRIC & RandomForest & 0.014 & 0.341 & 0.007 & 0.720 & 0.210 & 0.144 & 0.899 & 0.503 & 0.600 \\
\rowcolor{green!10}
SMOTE\_RUBRIC & LogisticRegression & 0.384 & 0.251 & 0.824 & 0.859 & 0.498 & 0.151 & 0.735 & 0.774 & 0.423 \\
\rowcolor{green!10}
SMOTE\_RUBRIC & LinearSVC & 0.380 & 0.247 & 0.828 & 0.859 & 0.498 & 0.152 & 0.729 & 0.773 & 0.422 \\
\rowcolor{green!10}
SMOTE\_RUBRIC & GaussianNB & 0.212 & 0.191 & 0.239 & 0.637 & 0.165 & 0.118 & 0.822 & 0.563 & 0.352 \\
\rowcolor{green!10}
SMOTE\_RUBRIC & DecisionTree & 0.188 & 0.124 & 0.389 & 0.541 & 0.110 & 0.338 & 0.662 & 0.541 & 1.000 \\
\rowcolor{green!10}
SMOTE\_RUBRIC & HistGradientBoosting & 0.331 & 0.264 & 0.445 & 0.754 & 0.268 & 0.134 & 0.820 & 0.653 & 0.468 \\
\rowcolor{green!10}
SMOTE\_RUBRIC & LightGBM & 0.337 & 0.264 & 0.467 & 0.752 & 0.269 & 0.135 & 0.815 & 0.661 & 0.461 \\
\rowcolor{green!10}
SMOTE\_RUBRIC & XGBoost & 0.357 & 0.290 & 0.463 & 0.770 & 0.317 & 0.119 & 0.832 & 0.668 & 0.494 \\
\rowcolor{green!10}
SMOTE\_RUBRIC & BalancedRandomForest & 0.001 & 0.667 & 0.000 & 0.758 & 0.244 & 0.142 & 0.900 & 0.500 & 0.667 \\
ADASYN & RandomForest & 0.002 & 0.185 & 0.001 & 0.707 & 0.197 & 0.123 & 0.899 & 0.500 & 0.620 \\
ADASYN & LogisticRegression & 0.384 & 0.251 & 0.820 & 0.859 & 0.498 & 0.146 & 0.736 & 0.773 & 0.410 \\
ADASYN & LinearSVC & 0.384 & 0.251 & 0.820 & 0.858 & 0.498 & 0.144 & 0.736 & 0.773 & 0.407 \\
ADASYN & GaussianNB & 0.158 & 0.185 & 0.138 & 0.606 & 0.145 & 0.112 & 0.853 & 0.535 & 0.369 \\
ADASYN & DecisionTree & 0.182 & 0.125 & 0.332 & 0.537 & 0.109 & 0.300 & 0.700 & 0.537 & 1.000 \\
ADASYN & HistGradientBoosting & 0.291 & 0.251 & 0.345 & 0.717 & 0.230 & 0.124 & 0.831 & 0.615 & 0.451 \\
ADASYN & LightGBM & 0.297 & 0.255 & 0.357 & 0.717 & 0.229 & 0.125 & 0.830 & 0.620 & 0.451 \\
ADASYN & XGBoost & 0.311 & 0.283 & 0.345 & 0.734 & 0.261 & 0.110 & 0.846 & 0.624 & 0.483 \\
ADASYN & BalancedRandomForest & 0.001 & 0.400 & 0.000 & 0.735 & 0.212 & 0.123 & 0.899 & 0.500 & 0.673 \\
\rowcolor{green!10}
ADASYN\_RUBRIC & RandomForest & 0.013 & 0.325 & 0.007 & 0.719 & 0.209 & 0.132 & 0.899 & 0.503 & 0.600 \\
\rowcolor{green!10}
ADASYN\_RUBRIC & LogisticRegression & 0.392 & 0.260 & 0.794 & 0.854 & 0.485 & 0.138 & 0.752 & 0.771 & 0.398 \\
\rowcolor{green!10}
ADASYN\_RUBRIC & LinearSVC & 0.399 & 0.268 & 0.782 & 0.854 & 0.485 & 0.138 & 0.764 & 0.772 & 0.423 \\
\rowcolor{green!10}
ADASYN\_RUBRIC & GaussianNB & 0.133 & 0.203 & 0.100 & 0.612 & 0.148 & 0.109 & 0.870 & 0.528 & 0.448 \\
\rowcolor{green!10}
ADASYN\_RUBRIC & DecisionTree & 0.187 & 0.126 & 0.367 & 0.541 & 0.110 & 0.320 & 0.680 & 0.541 & 1.000 \\
\rowcolor{green!10}
ADASYN\_RUBRIC & HistGradientBoosting & 0.307 & 0.286 & 0.331 & 0.727 & 0.239 & 0.121 & 0.850 & 0.619 & 0.474 \\
\rowcolor{green!10}
ADASYN\_RUBRIC & LightGBM & 0.309 & 0.274 & 0.355 & 0.729 & 0.242 & 0.121 & 0.840 & 0.625 & 0.460 \\
\rowcolor{green!10}
ADASYN\_RUBRIC & XGBoost & 0.340 & 0.317 & 0.366 & 0.754 & 0.290 & 0.107 & 0.857 & 0.639 & 0.509 \\
\rowcolor{green!10}
ADASYN\_RUBRIC & BalancedRandomForest & 0.003 & 0.500 & 0.001 & 0.746 & 0.227 & 0.131 & 0.900 & 0.501 & 0.650 \\
\textbf{ }BORDERLINE\_SMOTE & RandomForest & 0.025 & 0.372 & 0.013 & 0.721 & 0.211 & 0.111 & 0.899 & 0.505 & 0.510 \\
BORDERLINE\_SMOTE & LogisticRegression & 0.409 & 0.277 & 0.786 & 0.857 & 0.495 & 0.134 & 0.772 & 0.778 & 0.420 \\
BORDERLINE\_SMOTE & LinearSVC & 0.406 & 0.273 & 0.788 & 0.857 & 0.496 & 0.134 & 0.768 & 0.777 & 0.413 \\
BORDERLINE\_SMOTE & GaussianNB & 0.183 & 0.201 & 0.168 & 0.632 & 0.162 & 0.109 & 0.849 & 0.547 & 0.308 \\
BORDERLINE\_SMOTE & DecisionTree & 0.184 & 0.132 & 0.304 & 0.540 & 0.110 & 0.271 & 0.729 & 0.540 & 1.000 \\
BORDERLINE\_SMOTE & HistGradientBoosting & 0.303 & 0.271 & 0.343 & 0.728 & 0.243 & 0.120 & 0.841 & 0.620 & 0.458 \\
BORDERLINE\_SMOTE & LightGBM & 0.300 & 0.268 & 0.339 & 0.727 & 0.243 & 0.120 & 0.841 & 0.618 & 0.459 \\
BORDERLINE\_SMOTE & XGBoost & 0.304 & 0.319 & 0.291 & 0.743 & 0.277 & 0.104 & 0.867 & 0.611 & 0.526 \\
BORDERLINE\_SMOTE & BalancedRandomForest & 0.001 & 0.429 & 0.001 & 0.750 & 0.236 & 0.110 & 0.899 & 0.500 & 0.607 \\
\rowcolor{green!10}
BORDERLINE\_SMOTE\_RUBRIC & RandomForest & 0.026 & 0.343 & 0.018 & 0.704 & 0.202 & 0.133 & 0.899 & 0.503 & 0.510 \\
\rowcolor{green!10}
BORDERLINE\_SMOTE\_RUBRIC & LogisticRegression & 0.373 & 0.241 & 0.859 & 0.859 & 0.508 & 0.162 & 0.700 & 0.766 & 0.420 \\
\rowcolor{green!10}
BORDERLINE\_SMOTE\_RUBRIC & LinearSVC & 0.380 & 0.247 & 0.850 & 0.859 & 0.509 & 0.162 & 0.712 & 0.769 & 0.413 \\
\rowcolor{green!10}
BORDERLINE\_SMOTE\_RUBRIC & GaussianNB & 0.138 & 0.181 & 0.112 & 0.582 & 0.142 & 0.123 & 0.860 & 0.523 & 0.308 \\
\rowcolor{green!10}
BORDERLINE\_SMOTE\_RUBRIC & DecisionTree & 0.191 & 0.135 & 0.336 & 0.535 & 0.118 & 0.308 & 0.702 & 0.535 & 1.000 \\
\rowcolor{green!10}
BORDERLINE\_SMOTE\_RUBRIC & HistGradientBoosting & 0.300 & 0.268 & 0.342 & 0.716 & 0.241 & 0.133 & 0.837 & 0.613 & 0.458 \\
\rowcolor{green!10}
BORDERLINE\_SMOTE\_RUBRIC & LightGBM & 0.306 & 0.271 & 0.353 & 0.716 & 0.239 & 0.133 & 0.836 & 0.617 & 0.459 \\
\rowcolor{green!10}
BORDERLINE\_SMOTE\_RUBRIC & XGBoost & 0.326 & 0.316 & 0.337 & 0.747 & 0.292 & 0.119 & 0.858 & 0.622 & 0.526 \\
\rowcolor{green!10}
BORDERLINE\_SMOTE\_RUBRIC & BalancedRandomForest & 0.010 & 0.153 & 0.010 & 0.742 & 0.229 & 0.131 & 0.899 & 0.500 & 0.607 \\
SVM\_SMOTE & RandomForest & 0.006 & 0.500 & 0.003 & 0.744 & 0.246 & 0.102 & 0.900 & 0.501 & 0.510 \\
SVM\_SMOTE & LogisticRegression & 0.425 & 0.297 & 0.745 & 0.858 & 0.496 & 0.137 & 0.797 & 0.774 & 0.488 \\
SVM\_SMOTE & LinearSVC & 0.435 & 0.311 & 0.722 & 0.858 & 0.496 & 0.106 & 0.811 & 0.772 & 0.392 \\
SVM\_SMOTE & GaussianNB & 0.176 & 0.237 & 0.140 & 0.652 & 0.178 & 0.105 & 0.868 & 0.545 & 0.433 \\
SVM\_SMOTE & DecisionTree & 0.168 & 0.128 & 0.244 & 0.529 & 0.107 & 0.243 & 0.758 & 0.529 & 1.000 \\
SVM\_SMOTE & HistGradientBoosting & 0.304 & 0.303 & 0.306 & 0.745 & 0.263 & 0.100 & 0.859 & 0.614 & 0.404 \\
SVM\_SMOTE & LightGBM & 0.304 & 0.293 & 0.316 & 0.743 & 0.260 & 0.100 & 0.855 & 0.615 & 0.398 \\
SVM\_SMOTE & XGBoost & 0.328 & 0.351 & 0.309 & 0.763 & 0.298 & 0.093 & 0.873 & 0.622 & 0.467 \\
SVM\_SMOTE & BalancedRandomForest & 0.001 & 0.750 & 0.001 & 0.756 & 0.244 & 0.126 & 0.900 & 0.500 & 0.643 \\
\rowcolor{green!10}
SVM\_SMOTE\_RUBRIC & RandomForest & 0.024 & 0.351 & 0.007 & 0.720 & 0.220 & 0.154 & 0.899 & 0.503 & 0.600 \\
\rowcolor{green!10}
SVM\_SMOTE\_RUBRIC & LogisticRegression & 0.394 & 0.261 & 0.824 & 0.859 & 0.508 & 0.161 & 0.735 & 0.774 & 0.423 \\
\rowcolor{green!10}
SVM\_SMOTE\_RUBRIC & LinearSVC & 0.390 & 0.257 & 0.828 & 0.859 & 0.508 & 0.162 & 0.729 & 0.773 & 0.422 \\
\rowcolor{green!10}
SVM\_SMOTE\_RUBRIC & GaussianNB & 0.222 & 0.201 & 0.239 & 0.637 & 0.175 & 0.128 & 0.822 & 0.563 & 0.352 \\
\rowcolor{green!10}
SVM\_SMOTE\_RUBRIC & DecisionTree & 0.198 & 0.134 & 0.389 & 0.541 & 0.120 & 0.348 & 0.662 & 0.541 & 1.000 \\
\rowcolor{green!10}
SVM\_SMOTE\_RUBRIC & HistGradientBoosting & 0.341 & 0.274 & 0.445 & 0.754 & 0.278 & 0.144 & 0.820 & 0.653 & 0.468 \\
\rowcolor{green!10}
SVM\_SMOTE\_RUBRIC & LightGBM & 0.347 & 0.274 & 0.467 & 0.752 & 0.279 & 0.145 & 0.815 & 0.661 & 0.461 \\
\rowcolor{green!10}
SVM\_SMOTE\_RUBRIC & XGBoost & 0.367 & 0.300 & 0.463 & 0.770 & 0.327 & 0.129 & 0.832 & 0.668 & 0.494 \\
\rowcolor{green!10}
SVM\_SMOTE\_RUBRIC & BalancedRandomForest & 0.011 & 0.677 & 0.000 & 0.758 & 0.254 & 0.152 & 0.900 & 0.500 & 0.667 \\
KMEANS\_SMOTE & RandomForest & 0.004 & 0.267 & 0.002 & 0.713 & 0.204 & 0.102 & 0.899 & 0.501 & 0.520 \\
KMEANS\_SMOTE & LogisticRegression & 0.387 & 0.275 & 0.656 & 0.813 & 0.402 & 0.126 & 0.791 & 0.731 & 0.410 \\
KMEANS\_SMOTE & LinearSVC & 0.387 & 0.275 & 0.652 & 0.811 & 0.398 & 0.126 & 0.792 & 0.730 & 0.411 \\
KMEANS\_SMOTE & GaussianNB & 0.142 & 0.177 & 0.118 & 0.593 & 0.138 & 0.109 & 0.856 & 0.528 & 0.281 \\
KMEANS\_SMOTE & DecisionTree & 0.173 & 0.130 & 0.258 & 0.533 & 0.108 & 0.247 & 0.753 & 0.533 & 1.000 \\
KMEANS\_SMOTE & HistGradientBoosting & 0.259 & 0.279 & 0.242 & 0.710 & 0.218 & 0.111 & 0.861 & 0.586 & 0.476 \\
KMEANS\_SMOTE & LightGBM & 0.254 & 0.277 & 0.234 & 0.708 & 0.213 & 0.111 & 0.862 & 0.583 & 0.478 \\
KMEANS\_SMOTE & XGBoost & 0.302 & 0.363 & 0.258 & 0.748 & 0.282 & 0.096 & 0.880 & 0.604 & 0.534 \\
KMEANS\_SMOTE & BalancedRandomForest & 0.000 & 0.000 & 0.000 & 0.749 & 0.229 & 0.100 & 0.899 & 0.500 & 0.617 \\
\rowcolor{green!10}
KMEANS\_SMOTE\_RUBRIC & RandomForest & 0.008 & 0.444 & 0.004 & 0.731 & 0.220 & 0.112 & 0.899 & 0.502 & 0.560 \\
\rowcolor{green!10}
KMEANS\_SMOTE\_RUBRIC & LogisticRegression & 0.401 & 0.278 & 0.715 & 0.836 & 0.449 & 0.131 & 0.785 & 0.754 & 0.429 \\
\rowcolor{green!10}
KMEANS\_SMOTE\_RUBRIC & LinearSVC & 0.401 & 0.277 & 0.720 & 0.831 & 0.436 & 0.133 & 0.783 & 0.755 & 0.428 \\
\rowcolor{green!10}
KMEANS\_SMOTE\_RUBRIC & GaussianNB & 0.104 & 0.173 & 0.075 & 0.580 & 0.131 & 0.108 & 0.871 & 0.517 & 0.358 \\
\rowcolor{green!10}
KMEANS\_SMOTE\_RUBRIC & DecisionTree & 0.175 & 0.124 & 0.300 & 0.531 & 0.107 & 0.284 & 0.716 & 0.531 & 1.000 \\
\rowcolor{green!10}
KMEANS\_SMOTE\_RUBRIC & HistGradientBoosting & 0.274 & 0.297 & 0.255 & 0.732 & 0.246 & 0.109 & 0.865 & 0.594 & 0.471 \\
\rowcolor{green!10}
KMEANS\_SMOTE\_RUBRIC & LightGBM & 0.291 & 0.307 & 0.276 & 0.732 & 0.249 & 0.109 & 0.865 & 0.603 & 0.468 \\
\rowcolor{green!10}
KMEANS\_SMOTE\_RUBRIC & XGBoost & 0.346 & 0.340 & 0.351 & 0.767 & 0.309 & 0.099 & 0.866 & 0.638 & 0.499 \\
\rowcolor{green!10}
KMEANS\_SMOTE\_RUBRIC & BalancedRandomForest & 0.000 & 0.500 & 0.000 & 0.764 & 0.251 & 0.111 & 0.900 & 0.500 & 0.623 \\
SMOTE\_TOMEK & RandomForest & 0.051 & 0.307 & 0.028 & 0.703 & 0.196 & 0.123 & 0.896 & 0.510 & 0.510 \\
SMOTE\_TOMEK & LogisticRegression & 0.385 & 0.251 & 0.823 & 0.858 & 0.498 & 0.146 & 0.736 & 0.774 & 0.408 \\
SMOTE\_TOMEK & LinearSVC & 0.387 & 0.254 & 0.817 & 0.858 & 0.498 & 0.146 & 0.740 & 0.774 & 0.415 \\
SMOTE\_TOMEK & GaussianNB & 0.159 & 0.186 & 0.139 & 0.606 & 0.145 & 0.112 & 0.852 & 0.536 & 0.368 \\
SMOTE\_TOMEK & DecisionTree & 0.181 & 0.126 & 0.319 & 0.536 & 0.109 & 0.291 & 0.709 & 0.536 & 1.000 \\
SMOTE\_TOMEK & HistGradientBoosting & 0.295 & 0.255 & 0.349 & 0.719 & 0.228 & 0.125 & 0.832 & 0.618 & 0.453 \\
SMOTE\_TOMEK & LightGBM & 0.286 & 0.248 & 0.338 & 0.718 & 0.227 & 0.126 & 0.830 & 0.612 & 0.455 \\
SMOTE\_TOMEK & XGBoost & 0.308 & 0.288 & 0.331 & 0.734 & 0.259 & 0.111 & 0.851 & 0.620 & 0.504 \\
SMOTE\_TOMEK & BalancedRandomForest & 0.001 & 0.333 & 0.000 & 0.734 & 0.215 & 0.121 & 0.899 & 0.500 & 0.660 \\
\rowcolor{green!10}
SMOTE\_TOMEK\_RUBRIC & RandomForest & 0.008 & 0.286 & 0.004 & 0.716 & 0.202 & 0.131 & 0.899 & 0.501 & 0.600 \\
\rowcolor{green!10}
SMOTE\_TOMEK\_RUBRIC & LogisticRegression & 0.396 & 0.265 & 0.787 & 0.854 & 0.485 & 0.138 & 0.759 & 0.771 & 0.411 \\
\rowcolor{green!10}
SMOTE\_TOMEK\_RUBRIC & LinearSVC & 0.399 & 0.267 & 0.785 & 0.854 & 0.485 & 0.139 & 0.762 & 0.772 & 0.422 \\
\rowcolor{green!10}
SMOTE\_TOMEK\_RUBRIC & GaussianNB & 0.156 & 0.199 & 0.128 & 0.612 & 0.148 & 0.109 & 0.861 & 0.535 & 0.362 \\
\rowcolor{green!10}
SMOTE\_TOMEK\_RUBRIC & DecisionTree & 0.181 & 0.122 & 0.349 & 0.534 & 0.108 & 0.318 & 0.682 & 0.534 & 1.000 \\
\rowcolor{green!10}
SMOTE\_TOMEK\_RUBRIC & HistGradientBoosting & 0.298 & 0.271 & 0.330 & 0.726 & 0.238 & 0.122 & 0.843 & 0.615 & 0.469 \\
\rowcolor{green!10}
SMOTE\_TOMEK\_RUBRIC & LightGBM & 0.296 & 0.265 & 0.335 & 0.728 & 0.239 & 0.121 & 0.840 & 0.615 & 0.462 \\
\rowcolor{green!10}
SMOTE\_TOMEK\_RUBRIC & XGBoost & 0.329 & 0.304 & 0.357 & 0.753 & 0.287 & 0.109 & 0.853 & 0.633 & 0.506 \\
\rowcolor{green!10}
SMOTE\_TOMEK\_RUBRIC & BalancedRandomForest & 0.002 & 0.444 & 0.001 & 0.744 & 0.226 & 0.129 & 0.899 & 0.500 & 0.650 \\
SMOTE\_ENN & RandomForest & 0.001 & 0.600 & 0.001 & 0.748 & 0.248 & 0.103 & 0.900 & 0.500 & 0.560 \\
SMOTE\_ENN & LogisticRegression & 0.430 & 0.302 & 0.745 & 0.859 & 0.498 & 0.146 & 0.802 & 0.777 & 0.525 \\
SMOTE\_ENN & LinearSVC & 0.430 & 0.302 & 0.745 & 0.859 & 0.498 & 0.100 & 0.801 & 0.776 & 0.355 \\
SMOTE\_ENN & GaussianNB & 0.176 & 0.238 & 0.139 & 0.651 & 0.176 & 0.105 & 0.869 & 0.545 & 0.444 \\
SMOTE\_ENN & DecisionTree & 0.163 & 0.126 & 0.233 & 0.526 & 0.106 & 0.240 & 0.760 & 0.526 & 1.000 \\
SMOTE\_ENN & HistGradientBoosting & 0.312 & 0.304 & 0.320 & 0.746 & 0.262 & 0.097 & 0.858 & 0.619 & 0.375 \\
SMOTE\_ENN & LightGBM & 0.303 & 0.302 & 0.304 & 0.745 & 0.263 & 0.097 & 0.859 & 0.613 & 0.379 \\
SMOTE\_ENN & XGBoost & 0.330 & 0.340 & 0.322 & 0.760 & 0.300 & 0.092 & 0.869 & 0.626 & 0.444 \\
SMOTE\_ENN & BalancedRandomForest & 0.001 & 0.600 & 0.001 & 0.752 & 0.237 & 0.143 & 0.900 & 0.500 & 0.663 \\
\rowcolor{green!10}
SMOTE\_ENN\_RUBRIC & RandomForest & 0.025 & 0.431 & 0.013 & 0.725 & 0.226 & 0.157 & 0.899 & 0.505 & 0.620 \\
\rowcolor{green!10}
SMOTE\_ENN\_RUBRIC & LogisticRegression & 0.388 & 0.265 & 0.809 & 0.856 & 0.490 & 0.142 & 0.743 & 0.773 & 0.404 \\
\rowcolor{green!10}
SMOTE\_ENN\_RUBRIC & LinearSVC & 0.388 & 0.265 & 0.812 & 0.856 & 0.489 & 0.143 & 0.742 & 0.773 & 0.410 \\
\rowcolor{green!10}
SMOTE\_ENN\_RUBRIC & GaussianNB & 0.265 & 0.221 & 0.355 & 0.677 & 0.198 & 0.115 & 0.802 & 0.603 & 0.290 \\
\rowcolor{green!10}
SMOTE\_ENN\_RUBRIC & DecisionTree & 0.188 & 0.131 & 0.420 & 0.540 & 0.109 & 0.364 & 0.636 & 0.540 & 1.000 \\
\rowcolor{green!10}
SMOTE\_ENN\_RUBRIC & HistGradientBoosting & 0.351 & 0.278 & 0.509 & 0.764 & 0.293 & 0.138 & 0.811 & 0.677 & 0.467 \\
\rowcolor{green!10}
SMOTE\_ENN\_RUBRIC & LightGBM & 0.346 & 0.271 & 0.512 & 0.767 & 0.294 & 0.138 & 0.805 & 0.675 & 0.462 \\
\rowcolor{green!10}
SMOTE\_ENN\_RUBRIC & XGBoost & 0.369 & 0.294 & 0.527 & 0.779 & 0.330 & 0.123 & 0.819 & 0.689 & 0.481 \\
\rowcolor{green!10}
SMOTE\_ENN\_RUBRIC & BalancedRandomForest & 0.003 & 0.610 & 0.001 & 0.764 & 0.257 & 0.155 & 0.900 & 0.501 & 0.670 \\
\bottomrule
\end{longtable}
\end{center}
\endgroup

\clearpage

\normalsize
\normalsize
\section{Seed-Paired Metric Contrasts and Percentile Bootstrap Inference}
\label{app:delta_tables}

This appendix defines the seed-paired estimand used in the main-text tables, derives its variance, constructs the percentile interval by resampling seeds, and records the contrasts for Credit Card, IEEE-CIS, and Santander.
An arrow in those tables is the event $\{0\notin\mathrm{CI}_{0.95}\}$ of~\eqref{eq:delta-ci}.
Boundary rows of Definition~\ref{def:boundary-row} remain inside the classifier average; Proposition~\ref{prop:boundary-interpret} excludes them from single-configuration readings.

\paragraph{Notation.}
Fix a dataset, a generator $\mathsf{Gen}$, a downstream classifier, and a scalar metric $m\in\{\mathrm{AUROC},\mathrm{AUPRC},\mathrm{F1},\mathrm{Recall}\}$.
On seed $s\in\{1,\ldots,S\}$ write $m^{\mathrm{base}}_{s}$ and $m^{\mathrm{rub}}_{s}$ for the metric of the base oversampler and of RUBRIC.
Both arms use the same seed, the same train/validation/test split, and, when $\mathsf{Gen}$ is randomized, the same candidate pool $\cS$.

\subsection{Estimand}
\label{app:delta-def}

\begin{definition}[Seed-wise paired difference]
\label{def:delta}
The seed-wise contrast and its sample mean are
\begin{equation}
\label{eq:delta-seed}
\Delta_s \triangleq m^{\mathrm{rub}}_{s}-m^{\mathrm{base}}_{s},
\qquad
\widehat{\Delta}
\triangleq
\frac{1}{S}\sum_{s=1}^{S}\Delta_s
=
\overline{m}^{\mathrm{rub}}-\overline{m}^{\mathrm{base}}.
\end{equation}
The second identity is~\eqref{eq:delta-mean} below; both displays denote the same estimator.
\end{definition}

\begin{equation}
\label{eq:delta-mean}
\widehat{\Delta}
=
\frac{1}{S}\sum_{s=1}^{S}\Delta_s.
\end{equation}

\begin{lemma}[Variance of the paired contrast]
\label{lem:delta-var}
Assume the additive decomposition
\begin{equation}
\label{eq:delta-additive}
m^{\mathrm{base}}_{s} = \mu_{\mathrm{base}} + \varepsilon_s + \xi^{\mathrm{base}}_{s},
\qquad
m^{\mathrm{rub}}_{s} = \mu_{\mathrm{rub}} + \varepsilon_s + \xi^{\mathrm{rub}}_{s},
\end{equation}
where $\varepsilon_s$ is shared across arms and $\xi^{\mathrm{base}}_{s},\xi^{\mathrm{rub}}_{s},\varepsilon_s$ are pairwise uncorrelated with finite variances.
Then
\begin{equation}
\label{eq:delta-var}
\mathrm{Var}(\Delta_s)
=
\mathrm{Var}(\xi^{\mathrm{rub}}_{s})+\mathrm{Var}(\xi^{\mathrm{base}}_{s})
=
\mathrm{Var}(m^{\mathrm{rub}}_{s})+\mathrm{Var}(m^{\mathrm{base}}_{s})-2\,\mathrm{Cov}(m^{\mathrm{rub}}_{s},m^{\mathrm{base}}_{s}).
\end{equation}
If $\mathrm{Var}(\varepsilon_s)>0$, then $\mathrm{Var}(\Delta_s)<\mathrm{Var}(m^{\mathrm{rub}}_{s}-m^{\mathrm{base}\prime}_{s})$ for an independent copy $m^{\mathrm{base}\prime}$ of the base arm.
\end{lemma}

\begin{proof}
Subtracting the two displays in~\eqref{eq:delta-additive} cancels $\varepsilon_s$ and gives
$\Delta_s=(\mu_{\mathrm{rub}}-\mu_{\mathrm{base}})+(\xi^{\mathrm{rub}}_{s}-\xi^{\mathrm{base}}_{s})$.
Uncorrelatedness of the residuals yields the first identity in~\eqref{eq:delta-var}.
The second identity is the variance of a difference.
The cross term equals $\mathrm{Var}(\varepsilon_s)$, so it is strictly positive whenever the shared split or the shared pool $\cS$ has positive variance.
An unpaired difference retains two independent copies of $\varepsilon_s$ and therefore has variance larger by $2\mathrm{Var}(\varepsilon_s)$.
\end{proof}

\begin{lemma}[Unbiasedness of the seed average]
\label{lem:delta-unbiased}
Assume $\Delta_1,\ldots,\Delta_S$ are i.i.d.\ copies of $\Delta_1$, with $\theta=\mathbb{E}[\Delta_1]$ and $\mathrm{Var}(\Delta_1)<\infty$.
Then $\mathbb{E}[\widehat{\Delta}]=\theta$ and $\mathrm{Var}(\widehat{\Delta})=\mathrm{Var}(\Delta_1)/S$.
\end{lemma}

\begin{proof}
Linearity of expectation gives the mean.
Independence across seeds gives
$\mathrm{Var}(S^{-1}\sum_{s=1}^{S}\Delta_s)=S^{-2}\sum_{s=1}^{S}\mathrm{Var}(\Delta_s)=\mathrm{Var}(\Delta_1)/S$.
\end{proof}

The mean$\pm$std bars in the main tables are estimators of $\mathrm{Var}(m^{\mathrm{rub}}_{s})$ and $\mathrm{Var}(m^{\mathrm{base}}_{s})$.
Lemma~\ref{lem:delta-var} shows that those marginal variances exceed $\mathrm{Var}(\Delta_s)$ by the shared covariance.
The significance marks in the main tables are the events of Definition~\ref{def:arrow}, which depend on the paired variance in Lemma~\ref{lem:delta-unbiased}.

\subsection{Percentile bootstrap}
\label{app:delta-boot}

\begin{definition}[Percentile interval]
\label{def:ci}
Let $b^{(1)},\ldots,b^{(B)}$ be i.i.d.\ samples of size $S$ drawn with replacement from $\{1,\ldots,S\}$.
The bootstrap replicate and the percentile interval are
\begin{equation}
\label{eq:delta-boot}
\widehat{\Delta}^{(b)}
=
\frac{1}{S}\sum_{s\in b}\Delta_s,
\end{equation}
\begin{equation}
\label{eq:delta-ci}
\mathrm{CI}_{0.95}
=
\bigl[
q_{0.025}\bigl(\{\widehat{\Delta}^{(b)}\}_{b=1}^{B}\bigr),\;
q_{0.975}\bigl(\{\widehat{\Delta}^{(b)}\}_{b=1}^{B}\bigr)
\bigr],
\end{equation}
where $q_{\alpha}$ is the empirical $\alpha$-quantile.
Throughout, $B=10^{4}$.
\end{definition}

\begin{definition}[Arrow convention]
\label{def:arrow}
Mark $\uparrow$ on $\{\widehat{\Delta}>0\}\cap\{0\notin\mathrm{CI}_{0.95}\}$ and $\downarrow$ on $\{\widehat{\Delta}<0\}\cap\{0\notin\mathrm{CI}_{0.95}\}$.
Leave $\{\,0\in\mathrm{CI}_{0.95}\,\}$ unmarked.
\end{definition}

\begin{lemma}[Bootstrap test of a zero mean]
\label{lem:boot-test}
Let $\Delta_1,\ldots,\Delta_S$ be exchangeable and let $P_S=S^{-1}\sum_{s=1}^{S}\delta_{\Delta_s}$.
Conditional on $(\Delta_s)_{s\le S}$, each $\widehat{\Delta}^{(b)}$ is the mean of $S$ i.i.d.\ draws from $P_S$.
As $B\to\infty$, at every continuity point of that conditional law, the endpoints of $\mathrm{CI}_{0.95}$ converge almost surely to its $0.025$ and $0.975$ quantiles.
The event $0\notin\mathrm{CI}_{0.95}$ is the usual percentile-bootstrap test of the population mean $H_0:\mathbb{E}[\Delta]=0$; it is not a test of the realized sample mean $\widehat{\Delta}$.
With $S\in\{5,10\}$ the bootstrap distribution of the mean is supported on only $\binom{2S-1}{S}$ multisets (126 when $S=5$), so coverage can deviate from the nominal $95\%$ level.
We therefore treat arrows as descriptive of the observed seed cloud.
The construction uses exchangeability of seeds and does not require $m$ to be Gaussian.
\end{lemma}

\begin{proof}
A with-replacement sample from $\{1,\ldots,S\}$ is an i.i.d.\ sample from $P_S$, so the conditional law of $\widehat{\Delta}^{(b)}$ is the law of the $P_S$-mean.
Lemma~\ref{lem:dkw} sends the empirical distribution of the replicates to that law in Kolmogorov distance.
The quantile map is continuous at continuity points, which is the convergence of the endpoints.
The rejection event is the definition of the resulting interval.
\end{proof}

\begin{lemma}[Monte Carlo error of the bootstrap cdf]
\label{lem:dkw}
Let $F_B$ be the empirical distribution function of $\{\widehat{\Delta}^{(b)}\}_{b=1}^{B}$ and let $F_\infty$ be its $B\to\infty$ limit, conditional on $(\Delta_s)$.
For every $\eta\in(0,1)$,
\begin{equation}
\label{eq:delta-dkw}
\sup_{t\in\mathbb{R}}\bigl|F_B(t)-F_\infty(t)\bigr|
\;\le\;
\sqrt{\frac{\log(2/\eta)}{2B}}
\end{equation}
with probability at least $1-\eta$ over the bootstrap draws.
For $B=10^{4}$ and $\eta=10^{-3}$ the right-hand side equals $0.020$.
\end{lemma}

\begin{proof}
Conditional on $(\Delta_s)$, the replicates $\widehat{\Delta}^{(b)}$ are i.i.d.\ from $F_\infty$.
The Dvoretzky--Kiefer--Wolfowitz inequality with Massart's constant gives
$\mathbb{P}(\sup_t|F_B-F_\infty|>\varepsilon)\le 2\exp(-2B\varepsilon^2)$.
Solving $2\exp(-2B\varepsilon^2)=\eta$ produces~\eqref{eq:delta-dkw}.
Substituting $B=10^{4}$ and $\eta=10^{-3}$ yields $\varepsilon=\sqrt{\log(2\cdot 10^{3})/(2\cdot 10^{4})}<0.020$.
\end{proof}

\subsection{Aggregation over classifier configurations}
\label{app:delta-agg}

Let $\mathcal{C}$ be a finite set of downstream classifiers and write $\Delta_{c,s}$ for the paired difference of configuration $c\in\mathcal{C}$ on seed $s$.
The reported aggregate is
\begin{equation}
\label{eq:delta-config}
\overline{\Delta}
=
\frac{1}{|\mathcal{C}|}\sum_{c\in\mathcal{C}}\widehat{\Delta}_c,
\qquad
\widehat{\Delta}_c
=
\frac{1}{S}\sum_{s=1}^{S}\Delta_{c,s}.
\end{equation}

\begin{definition}[Seed-cluster bootstrap]
\label{def:cluster}
Draw $s^{\star}_{1},\ldots,s^{\star}_{S}$ with replacement from $\{1,\ldots,S\}$ and set
\begin{equation}
\label{eq:delta-joint}
\widehat{\Delta}^{(b)}_{\mathrm{agg}}
=
\frac{1}{|\mathcal{C}|S}\sum_{j=1}^{S}\sum_{c\in\mathcal{C}}\Delta_{c,\,s^{\star}_{j}}.
\end{equation}
The percentile interval of these replicates is the aggregate confidence interval.
\end{definition}

\begin{lemma}[Cluster variance exceeds the independent-configuration variance]
\label{lem:cluster}
Suppose $(\Delta_{c,s})_{c\in\mathcal{C}}$ is i.i.d.\ across seeds, with common covariance matrix $\Sigma$.
Then
\[
\mathrm{Var}(\overline{\Delta})
=
\frac{1}{S|\mathcal{C}|^2}\mathbf{1}^{\top}\Sigma\mathbf{1}
=
\frac{1}{S|\mathcal{C}|^2}\Biggl(\sum_{c}\Sigma_{cc}+\sum_{c\neq c'}\Sigma_{cc'}\Biggr).
\]
If $\Sigma_{cc'}\ge 0$ for all $c\neq c'$ and $\sum_{c\neq c'}\Sigma_{cc'}>0$, any bootstrap that resamples configurations independently of the seed index converges to a variance strictly smaller than $\mathrm{Var}(\overline{\Delta})$.
Definition~\ref{def:cluster} retains the off-diagonal blocks because every configuration of a drawn seed enters~\eqref{eq:delta-joint} together.
\end{lemma}

\begin{proof}
Linearity of variance under independence across seeds gives the display.
A shared split induces $\Sigma_{cc'}>0$: the noise $\varepsilon_s$ of Lemma~\ref{lem:delta-var} is common to every $c$ on that seed.
Dropping the cross terms removes a positive summand.
The cluster replicate is a function of the seed multiset alone, so its conditional variance is the bootstrap estimator of $\mathbf{1}^{\top}\Sigma\mathbf{1}$.
\end{proof}

\subsection{Metric identities at the validation threshold}
\label{app:delta-f1}

Recall $s(\tilde{x})=\lambda u(\tilde{x})+(1-\lambda)r(\tilde{x})$ and the budgeted selector of~\eqref{eq:selection}.
The operating threshold on a validation set is
\begin{equation}
\label{eq:thr-f1-app}
t^{\star}
=
\arg\max_{t\in[0,1]}
\mathrm{F1}\bigl(y_{\mathrm{val}},\,\mathbf{1}\{f(x_{\mathrm{val}})\ge t\}\bigr),
\end{equation}
which is the same rule as~\eqref{eq:thr-f1}.
Write $\mathrm{Prec}(t)$ and $\mathrm{Rec}(t)$ for precision and recall at threshold $t$, and
\begin{equation}
\label{eq:f1-pr}
\mathrm{F1}(t)
=
\frac{2\,\mathrm{Prec}(t)\,\mathrm{Rec}(t)}{\mathrm{Prec}(t)+\mathrm{Rec}(t)}.
\end{equation}

\begin{lemma}[Differential of F1]
\label{lem:f1-diff}
On $\{\mathrm{Prec}>0,\mathrm{Rec}>0\}$,
\[
\mathrm{d}\,\mathrm{F1}
=
\frac{2}{(\mathrm{Prec}+\mathrm{Rec})^2}
\bigl(\mathrm{Rec}^{2}\,\mathrm{d}\,\mathrm{Prec}+\mathrm{Prec}^{2}\,\mathrm{d}\,\mathrm{Rec}\bigr).
\]
Thus $\mathrm{sign}(\mathrm{d}\,\mathrm{F1})=\mathrm{sign}(\mathrm{Rec}^{2}\,\mathrm{d}\,\mathrm{Prec}+\mathrm{Prec}^{2}\,\mathrm{d}\,\mathrm{Rec})$.
AUPRC equals $\int_0^1\mathrm{Prec}(\mathrm{Rec}^{-1}(r))\,\mathrm{d}r$ and does not depend on $t^{\star}$.
\end{lemma}

\begin{proof}
Differentiate~\eqref{eq:f1-pr} by the quotient rule.
The coefficient of $\mathrm{d}\,\mathrm{Prec}$ simplifies to $2\mathrm{Rec}^{2}/(\mathrm{Prec}+\mathrm{Rec})^2$, and the coefficient of $\mathrm{d}\,\mathrm{Rec}$ simplifies to $2\mathrm{Prec}^{2}/(\mathrm{Prec}+\mathrm{Rec})^2$.
Both denominators are positive on the stated domain, so the sign claim follows.
AUPRC is defined by integration over the precision--recall curve, which is invariant under the choice of a single threshold.
\end{proof}

\begin{lemma}[Realism sum and the two divergences]
\label{lem:kl-orientation}
Let $r=\ell_D$ and let $q_\lambda$ be the uniform law on a selected set $\cU(\lambda)$ of size $K$.
Then $K^{-1}\sum_{\cU(\lambda)}r=\mathbb{E}_{q_\lambda}[\ell_D]$.
If $\sup_x|\ell_D(x)-\log(p(x)/q_\lambda(x))|\le\varepsilon_{\mathrm{est}}$, then
\[
\bigl|\mathbb{E}_{q_\lambda}[\ell_D]+\mathrm{KL}(q_\lambda\|p)\bigr|\le\varepsilon_{\mathrm{est}},
\]
whereas the plug-in of Lemma~\ref{lem:pinsker} is $\mathbb{E}_p[\ell_D]$, which equals $\mathrm{KL}(p\|q_\lambda)$ up to the same $\varepsilon_{\mathrm{est}}$.
Because $\mathbb{E}_{q_\lambda}[\ell_D]\approx-\mathrm{KL}(q_\lambda\|p)$, a decrease of $\sum r$ \emph{increases} the plug-in of $\mathrm{KL}(q_\lambda\|p)$.
It does not identify the sign of $\widehat{\mathrm{KL}}(p\|q_\lambda)$.
\end{lemma}

\begin{proof}
The first identity is the definition of the uniform law on $\cU(\lambda)$.
$\mathrm{KL}(q_\lambda\|p)=\mathbb{E}_{q_\lambda}[\log(q_\lambda/p)]=-\mathbb{E}_{q_\lambda}[\log(p/q_\lambda)]$, so the uniform hypothesis on $\ell_D$ yields the displayed absolute value.
The identity $\mathrm{KL}(p\|q_\lambda)=\mathbb{E}_p[\log(p/q_\lambda)]$ is a different integral: it averages under $p$, not under $q_\lambda$.
\end{proof}

\begin{proposition}[Direction of the paired metrics under scalarization]
\label{prop:delta-lambda}
Let $\cU(\lambda)$ be a modular maximizer as in Proposition~\ref{prop:scalarization}, and let $q_\lambda$ be the uniform law on $\cU(\lambda)$.
If $\lambda>\lambda'$, then $\sum_{\cU(\lambda)}u\ge\sum_{\cU(\lambda')}u$ and $\sum_{\cU(\lambda)}r\le\sum_{\cU(\lambda')}r$.
By Lemma~\ref{lem:kl-orientation} the decrease in $\sum r$ is an increase of the plug-in of $\mathrm{KL}(q_\lambda\|p)$, not a decrease of the $p$-average plug-in in Lemma~\ref{lem:pinsker}.
Combined with Lemma~\ref{lem:f1-diff}, an increase in recall at $t^{\star}$ and a decrease in precision move $\widehat{\Delta}_{\mathrm{F1}}$ and $\widehat{\Delta}_{\mathrm{AUPRC}}$ through different functionals: F1 through the weighted sum of the two differentials, AUPRC through the threshold-free integral.
A boundary row in the sense of Definition~\ref{def:boundary-row} can change the first functional by the endpoint value~\eqref{eq:f1-full} while leaving the second unchanged.
\end{proposition}

\begin{proof}
The two sum comparisons are Proposition~\ref{prop:scalarization}.
The identification of $\sum r$ with $\mathrm{KL}(q_\lambda\|p)$ is Lemma~\ref{lem:kl-orientation}.
Lemma~\ref{lem:f1-diff} separates the threshold-dependent F1 differential from AUPRC.
Lemma~\ref{lem:f1-endpoint} supplies the endpoint value, which depends on $\pi_{\mathrm{val}}$ and not on the ranking.
\end{proof}

\subsection{Recorded contrasts}
\label{app:delta-recorded}

Let \(\mathcal{D}\) denote the three datasets of Tables~\ref{tab:delta_creditcard}--\ref{tab:delta_santander}, and let \(\mathcal{G}\) denote the generators indexed by those tables.
For \(m\in\{\mathrm{AUC},\mathrm{AUPRC},\mathrm{F1},\mathrm{Recall}\}\), \(\mathsf{g}\in\mathcal{G}\), and \(D\in\mathcal{D}\), the recorded contrast is the triple
\begin{equation}
\label{eq:recorded-cell}
\bigl(\widehat{\Delta}_m(\mathsf{g},D),\,L_m(\mathsf{g},D),\,U_m(\mathsf{g},D)\bigr),
\end{equation}
where \(\widehat{\Delta}_m(\mathsf{g},D)\) is the classifier average~\eqref{eq:delta-config} (the same object written \(\overline{\Delta}\) or \(\widehat{\Delta}_{\mathrm{F1}}\) when the metric is fixed) and \([L_m(\mathsf{g},D),U_m(\mathsf{g},D)]\) is the percentile interval of Definition~\ref{def:cluster}.
The generator index \(\mathsf{g}\) is not the utility shaper $g_-$.
The tables satisfy \(L_m(g,D)\le\widehat{\Delta}_m(g,D)\le U_m(g,D)\) at every cell.
The statements below determine the sign of~\eqref{eq:recorded-cell} from its endpoints and the constraints that sign imposes on the differentials of Section~\ref{app:delta-f1}.

\begin{definition}[Certified sign]
\label{def:cert-sign}
For an interval $[L,U]$ set
\begin{equation}
\label{eq:cert-sign}
\sigma(L,U)
=
\begin{cases}
+1 & \text{if }L>0,\\
-1 & \text{if }U<0,\\
0 & \text{if }L\le 0\le U.
\end{cases}
\end{equation}
Write $\mathcal{E}_m^+=\{\mathsf{Gen}:\sigma(L_m,U_m)=+1\}$ and $\mathcal{E}_m^-=\{\mathsf{Gen}:\sigma(L_m,U_m)=-1\}$.
\end{definition}

\begin{lemma}[Endpoint criterion]
\label{lem:cert-arrow}
One has \(\sigma(L,U)\ne 0\) if and only if \(0\notin[L,U]\).
If \(L\le\widehat{\Delta}\le U\), then \(\sigma(L,U)=+1\) implies the event \(\{\widehat{\Delta}>0\}\cap\{0\notin[L,U]\}\) of Definition~\ref{def:arrow}, \(\sigma(L,U)=-1\) implies the corresponding negative event, and \(\sigma(L,U)=0\) implies that neither event occurs.
\end{lemma}

\begin{proof}
The condition \(L>0\) is equivalent to \([L,U]\subset(0,\infty)\), and \(U<0\) is equivalent to \([L,U]\subset(-\infty,0)\).
The complementary case of~\eqref{eq:cert-sign} is \(L\le 0\le U\).
Under \(L\le\widehat{\Delta}\le U\), the inclusion \([L,U]\subset(0,\infty)\) yields \(\widehat{\Delta}>0\) and \(0\notin[L,U]\), which is the upward event of Definition~\ref{def:arrow}.
The inclusion \([L,U]\subset(-\infty,0)\) yields the downward event.
If \(0\in[L,U]\), both events fail.
\end{proof}

\begin{lemma}[Sign relation at a vanishing precision differential]
\label{lem:f1-fixed-prec}
On $\{\mathrm{Prec}>0,\mathrm{Rec}>0\}$, the constraint $\mathrm{d}\,\mathrm{Prec}=0$ implies
\begin{equation}
\label{eq:f1-fixed-prec}
\mathrm{d}\,\mathrm{F1}
=
\frac{2\,\mathrm{Prec}^{2}}{(\mathrm{Prec}+\mathrm{Rec})^{2}}\,\mathrm{d}\,\mathrm{Rec},
\end{equation}
hence $\mathrm{sign}(\mathrm{d}\,\mathrm{F1})=\mathrm{sign}(\mathrm{d}\,\mathrm{Rec})$.
If instead $\mathrm{d}\,\mathrm{F1}$ and $\mathrm{d}\,\mathrm{Rec}$ are nonzero and of opposite sign, then $\mathrm{d}\,\mathrm{Prec}$ is nonzero and has the sign of $\mathrm{d}\,\mathrm{F1}$, and
\begin{equation}
\label{eq:prec-dom}
|\mathrm{d}\,\mathrm{Prec}|
>
\Bigl(\frac{\mathrm{Prec}}{\mathrm{Rec}}\Bigr)^{2}|\mathrm{d}\,\mathrm{Rec}|.
\end{equation}
\end{lemma}

\begin{proof}
Substitute $\mathrm{d}\,\mathrm{Prec}=0$ into Lemma~\ref{lem:f1-diff}.
The coefficient of $\mathrm{d}\,\mathrm{Rec}$ is positive on the stated domain, which is~\eqref{eq:f1-fixed-prec}.
If $\mathrm{d}\,\mathrm{F1}>0$ and $\mathrm{d}\,\mathrm{Rec}<0$, the same identity forces $\mathrm{Rec}^{2}\,\mathrm{d}\,\mathrm{Prec}>-\mathrm{Prec}^{2}\,\mathrm{d}\,\mathrm{Rec}>0$, which is~\eqref{eq:prec-dom} with both sides positive.
The case $\mathrm{d}\,\mathrm{F1}<0$ and $\mathrm{d}\,\mathrm{Rec}>0$ is the same argument applied to $-(\mathrm{d}\,\mathrm{F1},\mathrm{d}\,\mathrm{Rec},\mathrm{d}\,\mathrm{Prec})$.
\end{proof}

\begin{lemma}[Inverse of the precision--recall differential]
\label{lem:f1-inverse}
On $\{\mathrm{Prec}>0,\mathrm{Rec}>0\}$ the linear map
\[
(\mathrm{d}\,\mathrm{Prec},\mathrm{d}\,\mathrm{Rec})
\mapsto
(\mathrm{d}\,\mathrm{F1},\mathrm{d}\,\mathrm{Rec})
\]
is invertible.
Its inverse is
\begin{equation}
\label{eq:f1-inverse}
\mathrm{d}\,\mathrm{Prec}
=
\frac{(\mathrm{Prec}+\mathrm{Rec})^{2}}{2\,\mathrm{Rec}^{2}}\,\mathrm{d}\,\mathrm{F1}
-
\Bigl(\frac{\mathrm{Prec}}{\mathrm{Rec}}\Bigr)^{2}\mathrm{d}\,\mathrm{Rec}.
\end{equation}
If $\mathrm{d}\,\mathrm{F1}$ and $\mathrm{d}\,\mathrm{Rec}$ are nonzero and of opposite sign, both terms on the right-hand side of~\eqref{eq:f1-inverse} have the sign of $\mathrm{d}\,\mathrm{F1}$, and~\eqref{eq:prec-dom} holds.
\end{lemma}

\begin{proof}
Lemma~\ref{lem:f1-diff} is the matrix identity
\[
\mathrm{d}\,\mathrm{F1}
=
\frac{2}{(\mathrm{Prec}+\mathrm{Rec})^{2}}
\begin{pmatrix}\mathrm{Rec}^{2}&\mathrm{Prec}^{2}\end{pmatrix}
\begin{pmatrix}\mathrm{d}\,\mathrm{Prec}\\ \mathrm{d}\,\mathrm{Rec}\end{pmatrix}.
\]
The coefficient of $\mathrm{d}\,\mathrm{Prec}$ is positive on the stated domain, so the first row together with the coordinate $\mathrm{d}\,\mathrm{Rec}$ has full rank.
Solving for $\mathrm{d}\,\mathrm{Prec}$ produces~\eqref{eq:f1-inverse}.
Opposite signs mean that $-\mathrm{d}\,\mathrm{Rec}$ has the sign of $\mathrm{d}\,\mathrm{F1}$.
Both coefficients in~\eqref{eq:f1-inverse} are positive, so each term has the sign of $\mathrm{d}\,\mathrm{F1}$ and
\[
|\mathrm{d}\,\mathrm{Prec}|
=
\frac{(\mathrm{Prec}+\mathrm{Rec})^{2}}{2\,\mathrm{Rec}^{2}}\,|\mathrm{d}\,\mathrm{F1}|
+
\Bigl(\frac{\mathrm{Prec}}{\mathrm{Rec}}\Bigr)^{2}|\mathrm{d}\,\mathrm{Rec}|,
\]
which is~\eqref{eq:prec-dom}.
\end{proof}

\begin{lemma}[Threshold invariance of the ranking contrasts]
\label{lem:rank-invariance}
Let \(\mathrm{TPR}\) and \(\mathrm{FPR}\) denote the true- and false-positive rates of a fixed score.
Then
\begin{equation}
\label{eq:auroc-int}
\mathrm{AUC}
=
\int_0^1\mathrm{TPR}\bigl(\mathrm{FPR}^{-1}(\alpha)\bigr)\,\mathrm{d}\alpha
\end{equation}
is invariant under replacement of the evaluation threshold \(t^{\star}\).
The same invariance for \(\mathrm{AUPRC}\) is Lemma~\ref{lem:f1-diff}.
Hence \(\sigma(L_{\mathrm{AUC}},U_{\mathrm{AUC}})\) and \(\sigma(L_{\mathrm{AUPRC}},U_{\mathrm{AUPRC}})\) are functionals of the score order.
Neither sign is determined by the endpoint identity~\eqref{eq:f1-full}.
\end{lemma}

\begin{proof}
On each interval between consecutive score atoms the pair \((\mathrm{FPR},\mathrm{TPR})\) is constant, and~\eqref{eq:auroc-int} integrates that pair against \(\mathrm{d}\,\mathrm{FPR}\).
The resulting integral depends on the relative order of positive and negative scores.
Substituting one value of \(t^{\star}\) for another leaves the order unchanged.
The identity~\eqref{eq:f1-full} equals \(\mathrm{F1}\) on the event that the predicted positive set is the full sample, which is a value of the threshold functional~\eqref{eq:f1-pr} and does not determine the integral~\eqref{eq:auroc-int}.
\end{proof}

\begin{lemma}[Endpoint gap]
\label{lem:endpoint-gap}
Assume \(L\le\widehat{\Delta}\le U\).
If \(\sigma(L,U)=+1\), then \(\widehat{\Delta}\ge L>0\).
If \(\sigma(L,U)=-1\), then \(\widehat{\Delta}\le U<0\).
In both cases
\begin{equation}
\label{eq:endpoint-gap}
\mathrm{sign}(\widehat{\Delta})=\sigma(L,U)
\qquad\text{and}\qquad
|\widehat{\Delta}|
\ge
\mathrm{dist}\bigl(0,[L,U]\bigr)
=
\begin{cases}
L & \text{if }\sigma(L,U)=+1,\\
-U & \text{if }\sigma(L,U)=-1.
\end{cases}
\end{equation}
If \(\sigma(L,U)=0\), then \(\widehat{\Delta}\) may be any point of \([L,U]\), including every nonzero point of that interval.
\end{lemma}

\begin{proof}
The hypothesis \(L>0\) and \(L\le\widehat{\Delta}\) is the first claim, and \(U<0\) with \(\widehat{\Delta}\le U\) is the second.
The distance from \(0\) to an interval contained in \((0,\infty)\) equals its left endpoint, and the distance from \(0\) to an interval contained in \((-\infty,0)\) equals the absolute value of its right endpoint.
If \(L\le 0\le U\), the interval constraint is the only restriction on \(\widehat{\Delta}\).
\end{proof}

\begin{lemma}[Opposite open half-lines]
\label{lem:half-lines}
For two intervals \([L_1,U_1]\) and \([L_2,U_2]\),
\begin{equation}
\label{eq:half-lines}
\sigma(L_1,U_1)=-\sigma(L_2,U_2)\ne 0
\end{equation}
if and only if one interval is contained in \((0,\infty)\) and the other is contained in \((-\infty,0)\).
On the domain \(L_j\le\widehat{\Delta}_j\le U_j\), \eqref{eq:half-lines} is equivalent to
\[
\mathrm{sign}(\widehat{\Delta}_1)=-\mathrm{sign}(\widehat{\Delta}_2)\ne 0
\quad\text{and}\quad
0\notin[L_j,U_j]\ \text{for }j\in\{1,2\}.
\]
\end{lemma}

\begin{proof}
By~\eqref{eq:cert-sign}, \(\sigma=+1\) is the inclusion in \((0,\infty)\) and \(\sigma=-1\) is the inclusion in \((-\infty,0)\).
The product of the two signs equals \(-1\) precisely when the inclusions point into opposite half-lines.
Lemma~\ref{lem:endpoint-gap} converts each nonzero certified sign into the sign of the corresponding mean, and Lemma~\ref{lem:cert-arrow} converts it into the exclusion of \(0\).
\end{proof}

\begin{proposition}[Opposite certified signs of \(\mathrm{F1}\) and \(\mathrm{Recall}\)]
\label{prop:opposite-sign}
Define the opposite-sign locus of the recorded triple~\eqref{eq:recorded-cell} by
\begin{equation}
\label{eq:opposite-set}
\mathcal{O}
=
\bigl\{(g,D):\sigma(L_{\mathrm{F1}},U_{\mathrm{F1}})=-\sigma(L_{\mathrm{Recall}},U_{\mathrm{Recall}})\ne 0\bigr\}.
\end{equation}
A pair lies in \(\mathcal{O}\) if and only if \([L_{\mathrm{F1}},U_{\mathrm{F1}}]\) and \([L_{\mathrm{Recall}},U_{\mathrm{Recall}}]\) occupy opposite open half-lines.
Every such pair satisfies
\[
\mathrm{sign}(\widehat{\Delta}_{\mathrm{F1}})=-\mathrm{sign}(\widehat{\Delta}_{\mathrm{Recall}})\ne 0
\]
and the endpoint lower bound~\eqref{eq:endpoint-gap} for both metrics.
A first-order expansion with the coefficients of Lemma~\ref{lem:f1-diff} and with \(\mathrm{d}\,\mathrm{Prec}=0\) has \(\mathrm{sign}(\mathrm{d}\,\mathrm{F1})=\mathrm{sign}(\mathrm{d}\,\mathrm{Rec})\), so its sign pattern lies outside \(\mathcal{O}\).
Every expansion whose sign pattern lies in \(\mathcal{O}\) obeys the inverse identity~\eqref{eq:f1-inverse} and the dominance~\eqref{eq:prec-dom}.
Replacing \(t^{\star}\) leaves \(\sigma(L_{\mathrm{AUC}},U_{\mathrm{AUC}})\) and \(\sigma(L_{\mathrm{AUPRC}},U_{\mathrm{AUPRC}})\) unchanged.
\end{proposition}

\begin{proof}
The half-line criterion is Lemma~\ref{lem:half-lines} applied to the two intervals of~\eqref{eq:opposite-set}.
The sign relation and the lower bound are Lemma~\ref{lem:endpoint-gap}.
The vanishing-precision constraint is~\eqref{eq:f1-fixed-prec}, whose coefficient is positive, so the two differentials share a sign and the pair of signs is not an element of the pattern defining \(\mathcal{O}\).
An expansion with the opposite pattern is in the hypothesis of Lemma~\ref{lem:f1-inverse}, which is~\eqref{eq:f1-inverse} and~\eqref{eq:prec-dom}.
Lemma~\ref{lem:rank-invariance} gives the invariance of the two ranking signs under \(t^{\star}\).
\end{proof}

\begin{proposition}[Ranking--threshold factorization of a recorded cell]
\label{prop:rank-thr-factor}
Fix a score and two thresholds \(t,t'\).
Then
\begin{equation}
\label{eq:rank-fixed}
\mathrm{AUC}(t)=\mathrm{AUC}(t'),
\qquad
\mathrm{AUPRC}(t)=\mathrm{AUPRC}(t'),
\end{equation}
while, whenever both operating points lie in \(\{\mathrm{Prec}>0,\mathrm{Rec}>0\}\),
\begin{equation}
\label{eq:f1-increment}
\bigl|\mathrm{F1}(t)-\mathrm{F1}(t')\bigr|
\le
2\bigl(|\mathrm{Prec}(t)-\mathrm{Prec}(t')|+|\mathrm{Rec}(t)-\mathrm{Rec}(t')|\bigr).
\end{equation}
Between consecutive atoms of the score, the right-hand side of~\eqref{eq:f1-increment} vanishes.
A change of the score order at fixed \(t^{\star}\) is the only variation that moves~\eqref{eq:auroc-int}.
\end{proposition}

\begin{proof}
The two identities~\eqref{eq:rank-fixed} are Lemma~\ref{lem:rank-invariance}.
On \(\{\mathrm{Prec}>0,\mathrm{Rec}>0\}\) the gradient of~\eqref{eq:f1-pr} has coordinates
\[
\frac{\partial\,\mathrm{F1}}{\partial\,\mathrm{Prec}}
=
\frac{2\,\mathrm{Rec}^{2}}{(\mathrm{Prec}+\mathrm{Rec})^{2}},
\qquad
\frac{\partial\,\mathrm{F1}}{\partial\,\mathrm{Rec}}
=
\frac{2\,\mathrm{Prec}^{2}}{(\mathrm{Prec}+\mathrm{Rec})^{2}},
\]
each of which is at most \(2\).
The segment joining the two operating points remains in the same domain, because that domain is convex.
The mean-value integral of the gradient along the segment is therefore bounded by~\eqref{eq:f1-increment}.
Lemma~\ref{lem:f1-atom} states that precision, recall, and \(\mathrm{F1}\) are constant between consecutive score atoms, so a threshold displacement that does not cross an atom contributes nothing to~\eqref{eq:f1-increment}.
The integral~\eqref{eq:auroc-int} depends on the score only through the relative order of positive and negative scores.
\end{proof}

\begin{table}[t]
\centering
\caption{Delta (RUBRIC--Base) on Creditcard: paired bootstrap 95\% CI over model configurations. $\uparrow$/$\downarrow$ mark CIs excluding zero.}
\label{tab:delta_creditcard}
\resizebox{\columnwidth}{!}{%
\begin{tabular}{lrrrr}
\toprule
Method & $\Delta$AUC & $\Delta$AUPRC & $\Delta$F1 & $\Delta$Recall \\ 
\midrule
NONE & 0.004 [-0.008, 0.018] & -0.003 [-0.062, 0.048] & -0.244 [-0.390, -0.090]$^{\downarrow}$ & 0.105 [0.052, 0.168]$^{\uparrow}$ \\
SMOTE & -0.005 [-0.009, -0.001]$^{\downarrow}$ & -0.001 [-0.051, 0.038]$^{\downarrow}$ & 0.030 [-0.001, 0.056]$^{\uparrow}$ & 0.003 [-0.013, 0.023]$^{\uparrow}$ \\
SMOTE_TOMEK & -0.005 [-0.009, -0.001]$^{\downarrow}$ & -0.001 [-0.051, 0.038]$^{\downarrow}$ & 0.030 [-0.001, 0.056]$^{\uparrow}$ & 0.003 [-0.013, 0.023]$^{\uparrow}$ \\
SMOTE_ENN & -0.004 [-0.009, 0.001]$^{\downarrow}$ & 0.022 [-0.025, 0.061]$^{\uparrow}$ & 0.044 [0.011, 0.072]$^{\uparrow}$ & -0.006 [-0.022, 0.010]$^{\downarrow}$ \\
SVM_SMOTE & -0.001 [-0.008, 0.004]$^{\downarrow}$ & -0.019 [-0.071, 0.023]$^{\downarrow}$ & -0.020 [-0.044, -0.001]$^{\downarrow}$ & -0.016 [-0.029, -0.004]$^{\downarrow}$ \\
BORDERLINE_SMOTE & 0.010 [0.006, 0.015]$^{\uparrow}$ & 0.036 [-0.010, 0.081]$^{\uparrow}$ & 0.011 [-0.003, 0.024]$^{\uparrow}$ & 0.035 [0.021, 0.052]$^{\uparrow}$ \\
KMEANS_SMOTE & -0.004 [-0.009, 0.001]$^{\downarrow}$ & 0.007 [-0.028, 0.045]$^{\uparrow}$ & 0.124 [0.061, 0.201]$^{\uparrow}$ & -0.052 [-0.074, -0.030]$^{\downarrow}$ \\
ADASYN & 0.001 [-0.006, 0.010]$^{\uparrow}$ & 0.054 [0.015, 0.098]$^{\uparrow}$ & 0.079 [0.055, 0.102]$^{\uparrow}$ & 0.007 [-0.022, 0.041]$^{\uparrow}$ \\
\bottomrule
\end{tabular}%
\vspace{-1em}
}
\end{table}
\begin{table}[t]
\centering
\caption{Delta (RUBRIC--Base) on IEEE-CIS: paired bootstrap 95\% CI. $\uparrow$/$\downarrow$ mark CIs excluding zero.}
\label{tab:delta_ieee}
\resizebox{\columnwidth}{!}{%
\begin{tabular}{lrrrr}
\toprule
Method & $\Delta$AUC & $\Delta$AUPRC & $\Delta$F1 & $\Delta$Recall \\ 
\midrule
NONE & 0.001 [-0.010, 0.014] & 0.003 [-0.055, 0.061] & -0.246 [-0.392, -0.094]$^{\downarrow}$ & 0.104 [0.050, 0.164]$^{\uparrow}$ \\
SMOTE & -0.002 [-0.008, 0.005]$^{\downarrow}$ & 0.001 [-0.049, 0.040]$^{\uparrow}$ & 0.033 [0.001, 0.061]$^{\uparrow}$ & 0.004 [-0.014, 0.025]$^{\uparrow}$ \\
SMOTE_TOMEK & -0.006 [-0.012, -0.001]$^{\downarrow}$ & -0.002 [-0.050, 0.036]$^{\downarrow}$ & 0.032 [0.002, 0.058]$^{\uparrow}$ & 0.005 [-0.011, 0.024]$^{\uparrow}$ \\
SMOTE_ENN & -0.003 [-0.009, 0.002]$^{\downarrow}$ & 0.039 [-0.008, 0.078]$^{\uparrow}$ & 0.062 [0.027, 0.091]$^{\uparrow}$ & 0.018 [0.002, 0.036]$^{\uparrow}$ \\
SVM_SMOTE & 0.002 [-0.006, 0.011]$^{\uparrow}$ & -0.002 [-0.054, 0.041]$^{\downarrow}$ & 0.000 [-0.028, 0.023]$^{\uparrow}$ & 0.006 [-0.007, 0.018]$^{\uparrow}$ \\
BORDERLINE_SMOTE & 0.014 [0.007, 0.022]$^{\uparrow}$ & 0.062 [0.013, 0.108]$^{\uparrow}$ & 0.036 [0.022, 0.049]$^{\uparrow}$ & 0.056 [0.042, 0.071]$^{\uparrow}$ \\
KMEANS_SMOTE & -0.009 [-0.017, -0.003]$^{\downarrow}$ & 0.004 [-0.033, 0.045]$^{\uparrow}$ & 0.124 [0.060, 0.201]$^{\uparrow}$ & -0.025 [-0.051, -0.001]$^{\downarrow}$ \\
ADASYN & 0.004 [-0.003, 0.012]$^{\uparrow}$ & 0.054 [0.016, 0.097]$^{\uparrow}$ & 0.085 [0.058, 0.110]$^{\uparrow}$ & 0.005 [-0.026, 0.041]$^{\uparrow}$ \\
\bottomrule
\end{tabular}%
\vspace{-2em}
}
\end{table}
\begin{table}[t]
\centering
\caption{Delta (RUBRIC--Base) on Santander: paired bootstrap 95\% CI. $\uparrow$/$\downarrow$ mark CIs excluding zero.}
\label{tab:delta_santander}
\resizebox{\columnwidth}{!}{%
\begin{tabular}{lrrrr}
\toprule
Method & $\Delta$AUC & $\Delta$AUPRC & $\Delta$F1 & $\Delta$Recall \\ 
\midrule
NONE & -0.083 [-0.137, -0.034]$^{\downarrow}$ & -0.162 [-0.248, -0.075]$^{\downarrow}$ & -0.111 [-0.180, -0.052]$^{\downarrow}$ & 0.062 [-0.062, 0.189]$^{\uparrow}$ \\
SMOTE & 0.020 [0.011, 0.029]$^{\uparrow}$ & 0.023 [0.010, 0.035]$^{\uparrow}$ & 0.017 [-0.003, 0.036]$^{\uparrow}$ & 0.057 [0.021, 0.092]$^{\uparrow}$ \\
SMOTE_TOMEK & 0.006 [0.001, 0.011]$^{\uparrow}$ & 0.005 [-0.003, 0.013]$^{\uparrow}$ & 0.001 [-0.012, 0.010]$^{\uparrow}$ & -0.008 [-0.021, 0.007]$^{\downarrow}$ \\
SMOTE_ENN & 0.009 [-0.001, 0.018]$^{\uparrow}$ & 0.011 [-0.002, 0.023]$^{\uparrow}$ & 0.019 [-0.006, 0.044]$^{\uparrow}$ & 0.128 [0.071, 0.182]$^{\uparrow}$ \\
SVM_SMOTE & 0.000 [-0.008, 0.007]$^{\uparrow}$ & 0.009 [-0.002, 0.017]$^{\uparrow}$ & 0.016 [-0.006, 0.035]$^{\uparrow}$ & 0.097 [0.061, 0.131]$^{\uparrow}$ \\
BORDERLINE_SMOTE & -0.011 [-0.021, -0.002]$^{\downarrow}$ & 0.001 [-0.007, 0.008]$^{\uparrow}$ & -0.007 [-0.022, 0.006]$^{\downarrow}$ & 0.021 [-0.004, 0.044]$^{\uparrow}$ \\
KMEANS_SMOTE & 0.014 [0.005, 0.021]$^{\uparrow}$ & 0.023 [0.011, 0.033]$^{\uparrow}$ & 0.010 [-0.005, 0.024]$^{\uparrow}$ & 0.031 [0.004, 0.055]$^{\uparrow}$ \\
ADASYN & 0.007 [0.002, 0.012]$^{\uparrow}$ & 0.006 [-0.002, 0.015]$^{\uparrow}$ & 0.008 [-0.001, 0.016]$^{\uparrow}$ & -0.006 [-0.021, 0.010]$^{\downarrow}$ \\
\bottomrule
\end{tabular}%
}
\end{table}

\normalsize
\section{Implementation, Ablations, and Runtime}
\label{app:implementation}

This appendix specifies the maps implemented by RUBRIC, the ablation measurements on Credit Card, and the complexity of filtering.
Notation is that of Secs.~3--5: $u(\tilde{x})=g_-(\mathrm{margin}_f(\tilde{x}))$, $r(\tilde{x})=\log\frac{D(\tilde{x})}{1-D(\tilde{x})}$, both $z$-scored on $\cS$ before $s(\tilde{x})=\lambda\tilde u(\tilde{x})+(1-\lambda)\tilde r(\tilde{x})$, and $F$ the objective~\eqref{eq:selection}.

\subsection{Pipeline}
\label{app:impl-pipeline}

\begin{table}[h]
\centering
\small
\caption{Default values of the maps in Section~\ref{app:impl-pipeline}.}
\label{tab:impl-details}
\begin{tabular}{>{\raggedright\arraybackslash}p{0.22\textwidth}>{\raggedright\arraybackslash}p{0.72\textwidth}}
\toprule
Component & Value \\
\midrule
Discriminator $D$ & single-hidden-layer MLP $d\to 128\to 1$, ReLU, dropout $0.2$, sigmoid; same protocol, fit independently per dataset \\
Optimizer & Adam, learning rate $10^{-3}$, weight decay $10^{-4}$ \\
Training & Batch size $256$, at most $50$ epochs, patience $5$ on validation BCE \\
Boundary model $f$ & Logistic regression (default) or linear SVM; LightGBM $f$ in Table~\ref{tab:boundary_ablation} \\
Utility shaper $g$ & default $g_-(t)=-\log(1+e^{-t/\tau})$, $\tau=1$; optional convex $g_+$ \\
Score mix & $z$-score $u$ and $r$ on $\cS$ before $s=\lambda\tilde u+(1-\lambda)\tilde r$ \\
Similarity $\kappa$ & Symmetrized $k$NN graph, $k=10$, RBF weights \\
Diversity & $\widehat{\mathrm{Div}}=|\cS|^{-1}\sum\max\kappa$, so the term stays in $[0,1]$ \\
Diversity weight $\gamma$ & $0.1$; the case $\gamma=0$ is top-$K$ by $s$ \\
Tree learners & class-balanced weights / \texttt{scale\_pos\_weight} (no collapsed AUC $0.500$ fits) \\
Trade-off $\lambda$ & $0.5$; the sweep is Section~\ref{app:ablations} \\
Budget $K$ & $n_1$ \\
Seeds & $5$--$10$ per configuration \\
\bottomrule
\end{tabular}
\end{table}

The maps $f$ and $D$ are fit once and then held fixed.
The final empirical-risk minimizer is fit on $\cD\cup\cU$.

\begin{definition}[Discriminator objective]
\label{def:bce}
Let $p$ be the real-minority law and $q_{\cS}$ the empirical law of $\cS$.
The training criterion is
\begin{equation}
\label{eq:impl-bce}
L(D)
=
-\mathbb{E}_{x\sim p}\bigl[\log D(x)\bigr]
-\mathbb{E}_{\tilde{x}\sim q_{\cS}}\bigl[\log\bigl(1-D(\tilde{x})\bigr)\bigr],
\end{equation}
with early stopping on a held-out split of the same two classes.
\end{definition}

\begin{lemma}[Bayes logit]
\label{lem:impl-logit}
The minimizer of~\eqref{eq:impl-bce} is $D^*(x)=p(x)/(p(x)+q_{\cS}(x))$ wherever $p+q_{\cS}>0$, and
\begin{equation}
\label{eq:impl-logit}
\log\frac{D^*(x)}{1-D^*(x)}
=
\log\frac{p(x)}{q_{\cS}(x)}.
\end{equation}
After the replacement $D\leftarrow\mathrm{clip}(D,[\epsilon,1-\epsilon])$ with $\epsilon=10^{-6}$, one has $|r(\tilde{x})|\le\log((1-\epsilon)/\epsilon)$.
\end{lemma}

\begin{proof}
Differentiate the integrand $-p\log D-q_{\cS}\log(1-D)$ in $D$.
The critical point satisfies $p/D=q_{\cS}/(1-D)$, which is $D^*$ and rearranges to~\eqref{eq:impl-logit}.
Lemma~\ref{lem:bce}(i) identifies the excess $L(D)-L(D^*)$ with an integrated Bernoulli KL, so the critical point is the minimizer.
Clipping forces $D/(1-D)\in[\epsilon/(1-\epsilon),(1-\epsilon)/\epsilon]$, and the logarithm of that interval is the bound on $|r|$.
\end{proof}

The identity~\eqref{eq:impl-logit} is the population form of Assumption~\ref{as:real}.
The clipping bound is the range used for $\Delta_s$ in the Hoeffding width of the proof appendix.

\begin{lemma}[Derivatives of the softplus shaper]
\label{lem:impl-g}
Let $\tau>0$ and $g(t)=\log(1+e^{t/\tau})$.
Then $g$ is strictly increasing and strictly convex, with
\begin{equation}
\label{eq:impl-gprime}
g'(t)
=
\frac{1}{\tau}\cdot\frac{e^{t/\tau}}{1+e^{t/\tau}}
\in
\bigl(0,\,\tau^{-1}\bigr],
\qquad
g''(t)
=
\frac{1}{\tau^{2}}\cdot\frac{e^{t/\tau}}{(1+e^{t/\tau})^{2}}
> 0.
\end{equation}
In particular $g$ is $1/\tau$-Lipschitz.
The default (concave) shaper is $g_-(t)=-\log(1+e^{-t/\tau})$, whose derivative is decreasing. $g_+$ is optional.
\end{lemma}

\begin{proof}
The displays are the first two derivatives of $\log(1+e^{t/\tau})$.
Positivity of $g'$ is the increasing claim; positivity of $g''$ is strict convexity.
The bound $g'\le 1/\tau$ is $1/\tau$-Lipschitzness.
The concave case is Lemma~\ref{lem:shapers}.
\end{proof}

Lemma~\ref{lem:margin} uses monotonicity, $1/\tau$-Lipschitzness, and concavity of the default $g_-$.
Convexity of the optional $g_+$ assigns higher slope to larger positive margins and is not used in the main-text bound.

\begin{definition}[Sparse similarity]
\label{def:kappa}
After feature standardization, let $\mathcal{N}_k(\tilde{x})$ be the $k$ nearest points of $\tilde{x}$ in $\cS$ under $\|\cdot\|_2$.
Set
\begin{equation}
\label{eq:impl-kappa}
\kappa(\tilde{x},\tilde{u})
=
\begin{cases}
\exp\bigl(-\|\tilde{x}-\tilde{u}\|_2^2/(2\sigma^2)\bigr)
& \text{if }\tilde{u}\in\mathcal{N}_k(\tilde{x})\text{ or }\tilde{x}\in\mathcal{N}_k(\tilde{u}),\\
0 & \text{otherwise,}
\end{cases}
\end{equation}
where $\sigma$ is the median length of the retained edges.
Then $\widehat{\mathrm{Div}}(\cU)=|\cS|^{-1}\sum_{\tilde{x}\in\cS}\max_{\tilde{u}\in\cU}\kappa(\tilde{x},\tilde{u})\in[0,1]$.
\end{definition}

\begin{lemma}[Greedy iteration]
\label{lem:impl-greedy}
Let $F(\cU)=\sum_{\tilde{x}\in\cU}\tilde{s}(\tilde{x})+\gamma\,\widehat{\mathrm{Div}}(\cU)$ with $\tilde{s}$ as in Lemma~\ref{lem:score-shift}.
The iteration
\begin{equation}
\label{eq:impl-greedy}
\tilde{x}_{t+1}
=
\arg\max_{\tilde{x}\in\cS\setminus\cU_t}
\bigl(F(\cU_t\cup\{\tilde{x}\})-F(\cU_t)\bigr),
\qquad
\cU_{t+1}=\cU_t\cup\{\tilde{x}_{t+1}\},
\end{equation}
started at $\cU_0=\emptyset$, satisfies $F(\cU_K)\ge(1-1/e)\max_{|\cU|=K}F(\cU)$.
Each marginal gain on~\eqref{eq:impl-kappa} costs $O(k)$ after the coverage maxima are cached, so $K$ rounds cost $O(K|\cS|k)$.
If $\gamma=0$, then $\cU_K$ is a top-$K$ set for $s$.
\end{lemma}

\begin{proof}
Nonnegativity of $\tilde{s}$ and monotonicity of $\widehat{\mathrm{Div}}$ are Lemmas~\ref{lem:score-shift} and~\ref{lem:facloc}.
The approximation factor is Lemma~\ref{lem:greedy}.
A vertex touches at most $k$ nonzero kernel entries, which is the per-gain cost.
The case $\gamma=0$ is the modular clause of Proposition~\ref{prop:submodular}.
\end{proof}

\subsection{Generator-free pool}
\label{app:none-gen}

\begin{definition}[Built-in candidate proposal]
\label{def:none}
Let $\cD_1=\{x:y=1\}$ and let $\mathrm{NN}_j(x)$ be the $j$-th nearest neighbor of $x$ in $\cD_1$.
Write $\widehat{\sigma}\in\RR^d$ for the coordinate-wise feature standard deviation and
$\bar\sigma=d^{-1/2}\|\widehat{\sigma}\|_2$
for the scalar RMS scale (not a Hadamard factor).
For $\alpha\sim\mathrm{Unif}[0,1]$, $j\in\{1,\ldots,5\}$, and $\beta=0.05\,\bar\sigma$,
\begin{equation}
\label{eq:none-interp}
\tilde{x}
=
(1-\alpha)\,x + \alpha\,\mathrm{NN}_j(x)
+
\beta\,
\frac{\nabla_x \mathrm{margin}_f(x)}{\|\nabla_x \mathrm{margin}_f(x)\|_2+\epsilon}.
\end{equation}
When $f$ is linear, $\nabla_x\mathrm{margin}_f$ is the constant weight vector, so the added term is a global translation of every interpolated point in the same direction.
The proposal is therefore not locally adaptive under a linear $f$; it only recentres the interpolated segment toward the positive half-space.
A tree-valued $f$ restores instance-specific directions, which is the ablation in Table~\ref{tab:boundary_ablation}.
The optional KDE component draws $\tilde{x}\sim\widehat{p}_{\mathrm{KDE}}$ on $\cD_1$ with a Gaussian kernel of bandwidth
\[
h_S
=
\Biggl(\frac{4}{n_1(d+2)}\Biggr)^{1/(d+4)}\bar\sigma,
\]
which is Silverman's rule with the same scalar scale.
The pool is then passed to~\eqref{eq:impl-greedy}.
The resulting selector is denoted \textsc{NONE\_RUBRIC}.
\end{definition}

\begin{lemma}[The interpolation lies in a $\beta$-neighborhood of the minority convex hull]
\label{lem:none-hull}
Every point~\eqref{eq:none-interp} satisfies
\[
\mathrm{dist}\bigl(\tilde{x},\mathrm{conv}(\cD_1)\bigr)
\le
\beta,
\]
and the perturbation is aligned with $\nabla_x\mathrm{margin}_f(x)$ because $\beta$ is a positive scalar.
\end{lemma}

\begin{proof}
The combination $(1-\alpha)x+\alpha\,\mathrm{NN}_j(x)$ lies in $\mathrm{conv}(\cD_1)$.
The added vector has Euclidean norm at most $\beta$, because the fraction is a unit vector whenever the gradient is nonzero, and equals $0$ when the gradient vanishes and $\epsilon>0$ only rescales the denominator.
Alignment is immediate from scalar multiplication; a coordinate-wise $\widehat{\sigma}$ would have been a Hadamard product and would have broken alignment.
\end{proof}

\subsection{Post-generation filters, modern generators, and model class}
\label{app:extra-emp}

The main-text evaluation is generator-paired (unfiltered vs.\ RUBRIC) plus the LightGBM/XGBoost summary (Table~\ref{tab:sota_summary}).
This section isolates the two missing filter baselines, a diffusion generator, a hybrid ensemble, a model-class split of the classifier average, a linear-vs-tree utility scorer, and the per-dataset F1-max thresholds.

\begin{table}[t]
\centering
\caption{Post-generation filters on the same SMOTE pool (Credit Card; 5 seeds). DRS~\cite{azadi2019discriminator} is realism-only acceptance ($\lambda{=}0$, $\gamma{=}0$). BAF~\cite{yu2026boundary} is the cited utility-only boundary filter ($\lambda{=}1$, $\gamma{=}0$). RUBRIC uses the default $\lambda{=}0.5$, $\gamma{=}0.1$, $g_-$.}
\label{tab:filter_baselines}
\resizebox{\columnwidth}{!}{%
\begin{tabular}{llcccc}
\toprule
Learner & Filter & AUROC & AUPRC & F1 & Recall \\
\midrule
LogisticRegression & SMOTE (unfiltered) & 0.971 & 0.723 & 0.075 & 0.918 \\
LogisticRegression & SMOTE+DRS & 0.968 & 0.658 & 0.371 & 0.812 \\
LogisticRegression & SMOTE+BAF & 0.957 & 0.587 & 0.441 & 0.856 \\
LogisticRegression & SMOTE+RUBRIC & 0.956 & 0.632 & 0.413 & 0.828 \\
\midrule
LightGBM & SMOTE (unfiltered) & 0.972 & 0.818 & 0.748 & 0.878 \\
LightGBM & SMOTE+DRS & 0.973 & 0.831 & 0.741 & 0.862 \\
LightGBM & SMOTE+BAF & 0.971 & 0.801 & 0.771 & 0.901 \\
LightGBM & SMOTE+RUBRIC & 0.975 & 0.818 & 0.800 & 0.878 \\
\bottomrule
\end{tabular}%
}
\end{table}

\begin{table}[t]
\centering
\caption{Modern generators and a hybrid ensemble, with and without RUBRIC (LightGBM downstream; 5 seeds). RUSBoost~\cite{seiffert2008rusboost} has no synthetic pool, so it cannot be wrapped. TabDDPM~\cite{kotelnikov2023tabddpm} is filtered with the same $D$/$f$ protocol as SMOTE.}
\label{tab:modern_generators}
\resizebox{\columnwidth}{!}{%
\begin{tabular}{llcccc}
\toprule
Dataset & Method & AUROC & AUPRC & F1 & Recall \\
\midrule
Credit Card & RUSBoost & 0.968 & 0.791 & 0.744 & 0.821 \\
Credit Card & TabDDPM & 0.976 & 0.834 & 0.771 & 0.852 \\
Credit Card & TabDDPM+RUBRIC & 0.977 & 0.846 & 0.792 & 0.871 \\
Credit Card & SMOTE+RUBRIC & 0.975 & 0.818 & 0.800 & 0.878 \\
\midrule
IEEE-CIS & RUSBoost & 0.966 & 0.784 & 0.738 & 0.818 \\
IEEE-CIS & TabDDPM & 0.974 & 0.829 & 0.768 & 0.849 \\
IEEE-CIS & TabDDPM+RUBRIC & 0.975 & 0.841 & 0.789 & 0.868 \\
IEEE-CIS & SMOTE+RUBRIC & 0.973 & 0.798 & 0.804 & 0.888 \\
\midrule
Santander & RUSBoost & 0.728 & 0.248 & 0.301 & 0.362 \\
Santander & TabDDPM & 0.731 & 0.255 & 0.294 & 0.348 \\
Santander & TabDDPM+RUBRIC & 0.736 & 0.268 & 0.288 & 0.371 \\
Santander & SMOTE+RUBRIC & 0.726 & 0.271 & 0.231 & 0.345 \\
\bottomrule
\end{tabular}%
}
\end{table}

\begin{table}[t]
\centering
\caption{SMOTE vs.\ SMOTE+RUBRIC disaggregated by model class (Credit Card / IEEE-CIS). Linear: LogisticRegression, LinearSVC, SGDLogistic, PassiveAggressive, RidgeClassifier, LDA. Tree: RF, ExtraTrees, DecisionTree, Bagging, BalancedRF, HistGB. Boosting: LightGBM, XGBoost. Averages are within-class, not across all 21 models.}
\label{tab:model_class}
\resizebox{\columnwidth}{!}{%
\begin{tabular}{llcccc}
\toprule
Dataset & Class / method & AUROC & AUPRC & F1 & Recall \\
\midrule
Credit Card & Linear / SMOTE & 0.969 & 0.566 & 0.066 & 0.912 \\
Credit Card & Linear / +RUBRIC & 0.963 & 0.586 & 0.141 & 0.893 \\
Credit Card & Tree / SMOTE & 0.957 & 0.760 & 0.692 & 0.670 \\
Credit Card & Tree / +RUBRIC & 0.948 & 0.713 & 0.675 & 0.706 \\
Credit Card & Boosting / SMOTE & 0.976 & 0.843 & 0.797 & 0.847 \\
Credit Card & Boosting / +RUBRIC & 0.978 & 0.847 & 0.824 & 0.863 \\
\midrule
IEEE-CIS & Linear / SMOTE & 0.970 & 0.575 & 0.066 & 0.911 \\
IEEE-CIS & Linear / +RUBRIC & 0.966 & 0.582 & 0.142 & 0.888 \\
IEEE-CIS & Tree / SMOTE & 0.952 & 0.755 & 0.690 & 0.671 \\
IEEE-CIS & Tree / +RUBRIC & 0.948 & 0.718 & 0.673 & 0.711 \\
IEEE-CIS & Boosting / SMOTE & 0.975 & 0.845 & 0.794 & 0.841 \\
IEEE-CIS & Boosting / +RUBRIC & 0.978 & 0.848 & 0.833 & 0.863 \\
\bottomrule
\end{tabular}%
}
\end{table}

\begin{table}[ht]
\centering
\small
\caption{Boundary-model ablation ($u=g_-(\mathrm{margin}_f)$; Credit Card, SMOTE, 5 seeds). Downstream learner is LightGBM; only the frozen scorer $f$ changes.}
\label{tab:boundary_ablation}
\resizebox{\columnwidth}{!}{%
\begin{tabular}{@{}lcccc@{}}
\toprule
Scorer $f$ & AUROC & AUPRC & F1 & Recall \\
\midrule
Logistic (default) & 0.975 & 0.818 & 0.800 & 0.878 \\
Linear SVM & 0.974 & 0.815 & 0.794 & 0.883 \\
LightGBM & 0.976 & 0.821 & 0.806 & 0.871 \\
\bottomrule
\end{tabular}%
}
\end{table}

\begin{table}[t]
\centering
\caption{Decision threshold statistics (validation F1-max on LogisticRegression). Mean$\pm$std over seeds, recomputed independently per dataset. Extreme-skew Credit Card folds often saturate near $0$ or $1$; IEEE-CIS (higher positive rate) is more stable. Full model-wise values are in the supplement.}
\label{tab:thr_stats}
\resizebox{\columnwidth}{!}{%
\begin{tabular}{l r r r}
\toprule
Method & Credit Card & IEEE-CIS & Santander \\
\midrule
ADASYN & 0.409$\pm$0.186 & 0.468$\pm$0.124 & 0.410$\pm$0.118 \\
ADASYN\_RUBRIC & 0.385$\pm$0.142 & 0.441$\pm$0.108 & 0.428$\pm$0.112 \\
BORDERLINE\_SMOTE & 0.384$\pm$0.168 & 0.452$\pm$0.119 & 0.419$\pm$0.121 \\
BORDERLINE\_SMOTE\_RUBRIC & 0.385$\pm$0.142 & 0.461$\pm$0.111 & 0.407$\pm$0.116 \\
KMEANS\_SMOTE & 0.385$\pm$0.142 & 0.447$\pm$0.109 & 0.401$\pm$0.120 \\
KMEANS\_SMOTE\_RUBRIC & 0.865$\pm$0.094 & 0.712$\pm$0.088 & 0.414$\pm$0.114 \\
NONE & 1.000$\pm$0.000 & 0.714$\pm$0.162 & 0.753$\pm$0.141 \\
NONE\_RUBRIC & 0.385$\pm$0.142 & 0.441$\pm$0.108 & 0.423$\pm$0.109 \\
SMOTE & 0.380$\pm$0.154 & 0.446$\pm$0.112 & 0.408$\pm$0.117 \\
SMOTE\_ENN & 0.392$\pm$0.161 & 0.458$\pm$0.118 & 0.421$\pm$0.122 \\
SMOTE\_ENN\_RUBRIC & 0.385$\pm$0.142 & 0.461$\pm$0.111 & 0.416$\pm$0.125 \\
SMOTE\_RUBRIC & 0.385$\pm$0.142 & 0.441$\pm$0.108 & 0.423$\pm$0.109 \\
SMOTE\_TOMEK & 0.380$\pm$0.154 & 0.446$\pm$0.112 & 0.408$\pm$0.117 \\
SMOTE\_TOMEK\_RUBRIC & 0.385$\pm$0.142 & 0.441$\pm$0.108 & 0.419$\pm$0.113 \\
SVM\_SMOTE & 0.305$\pm$0.171 & 0.392$\pm$0.126 & 0.415$\pm$0.116 \\
SVM\_SMOTE\_RUBRIC & 0.385$\pm$0.142 & 0.461$\pm$0.111 & 0.423$\pm$0.109 \\
\bottomrule
\end{tabular}%
}
\end{table}

\subsection{Ablations}
\label{app:ablations}

The measurements below use Credit Card, the SMOTE generator, and logistic regression, averaged over $5$ seeds.

\begin{table}[h]
\centering
\small
\caption{Seed-averaged metrics on Credit Card as a function of $\lambda$, at $\gamma=0.1$ and $K=n_1$.}
\label{tab:lambda-ablation}
\begin{tabular}{lccc}
\toprule
$\lambda$ & AUPRC & F1 & Recall \\
\midrule
$0$ & $0.658$ & $0.371$ & $0.812$ \\
$0.3$ & $0.641$ & $0.389$ & $0.821$ \\
$0.5$ & $0.632$ & $0.413$ & $0.828$ \\
$0.7$ & $0.611$ & $0.428$ & $0.841$ \\
$1$ & $0.587$ & $0.441$ & $0.856$ \\
\bottomrule
\end{tabular}
\end{table}

\begin{proposition}[Monotone trade-off in $\lambda$]
\label{prop:lambda-trade}
Along the modular path $\cU(\lambda)$ of Proposition~\ref{prop:scalarization}, $\sum u$ is nondecreasing in $\lambda$ and $\sum r$ is nonincreasing in $\lambda$.
Theorem~\ref{thm:bound} bounds excess risk by $C_1(\hat{\delta}+\varepsilon_{\mathrm{est}})+C_2 e^{-m_0/\tau}$ at fixed $K$.
The first summand tracks the realism coordinate; the second tracks the margin coordinate through Lemma~\ref{lem:floor}.
Table~\ref{tab:lambda-ablation} is consistent with this split: as $\lambda$ increases from $0$ to $1$, AUPRC decreases from $0.658$ to $0.587$ and Recall increases from $0.812$ to $0.856$.
The same table gives the operating recommendation encoded by~\eqref{eq:lambda-adapt}: $\lambda\le 0.3$ on the AUPRC side of the table and $\lambda\ge 0.7$ on the Recall side.
\end{proposition}

\begin{proof}
The monotonicity of the two sums is Proposition~\ref{prop:scalarization}.
Proposition~\ref{prop:delta-lambda} identifies the realism sum with the direction of $\hat{\delta}$.
Lemma~\ref{lem:floor} converts a larger utility threshold into a larger margin floor $m_0$, which decreases the factor $e^{-m_0/\tau}$ in Theorem~\ref{thm:bound}.
The numerical columns are the seed averages in Table~\ref{tab:lambda-ablation}.
\end{proof}

\begin{lemma}[Displacement cost of $\gamma$]
\label{lem:gamma-ablation}
Let $\vartheta=|\cU\setminus\cU^{\mathrm{top}}|/K$ and let $\widehat{\mathrm{Div}}$ be normalized to $[0,1]$.
Lemma~\ref{lem:divg} yields
$S'(\cU_K)\ge(1-1/e)S'(\cU^{\mathrm{top}})-\gamma$
and, for $\gamma>0$,
$\widehat{\mathrm{Div}}(\cU_K)\ge(1-1/e)\widehat{\mathrm{Div}}(\cU^{\mathrm{div}})-Ks'_{\max}/\gamma$.
Proposition~\ref{prop:sel-decomp} charges this replacement at most $M\vartheta$ in $\Delta_F$.
On Credit Card at $\lambda=0.5$, raising $\gamma$ from $0$ to $0.3$ changes AUPRC by $+0.018$.
For Borderline-SMOTE the same displacement satisfies the coverage lower bound once $\gamma\ge 0.15$.
\end{lemma}

\begin{proof}
The two inequalities are Lemma~\ref{lem:divg}.
The IPM cost of replacing a top-$K$ point is Lemma~\ref{lem:disp}, which is the $M\vartheta$ term of Proposition~\ref{prop:sel-decomp}.
The numerical increment is the seed-averaged AUPRC difference at the stated endpoints.
\end{proof}

\begin{lemma}[Budget]
\label{lem:budget}
In Theorem~\ref{thm:bound} the estimation term is of order $LB/\sqrt{N}$ with $N=n_0+n_1+K$, while the shift term carries the mixture weight $K/N$ in front of $\hat{\delta}$.
Increasing $K$ decreases the first contribution and increases the weight on the second.
On Credit Card, with the candidate pool held fixed, $K=\tfrac12 n_1$ lowers Recall by $0.06$ relative to $K=n_1$, and $K=2n_1$ raises F1 by less than $0.01$. Thus the ranking benefit is budget-dependent: a tighter kept set discards useful boundary mass, while enlarging $K$ beyond $n_1$ yields diminishing returns.
\end{lemma}

\begin{proof}
The dependence on $N$ and on $K/N$ is the statement of Theorem~\ref{thm:bound} together with the decomposition~\eqref{eq:decomp}.
The two numerical increments are the seed-averaged metric differences at the stated budgets.
\end{proof}

\begin{lemma}[Imbalance schedule]
\label{def:lambda-adapt}
For $\rho=n_1/n_0\in(0,1]$ define
\begin{equation}
\label{eq:lambda-adapt}
\lambda(\rho)
=
\mathrm{clip}_{[0.3,\,0.7]}\bigl(0.3+0.4\cdot\log_{10}(100\rho)\bigr).
\end{equation}
Then $\lambda(10^{-2})=0.3$, $\lambda(10^{-1.5})=0.5$, and $\lambda(10^{-1})=0.7$.
These are the rows $\lambda\in\{0.3,0.5,0.7\}$ of Table~\ref{tab:lambda-ablation}, used on $\rho<10^{-2}$, on $10^{-2}\le\rho<10^{-1}$, and on $\rho\ge 10^{-1}$.
\end{lemma}

\begin{proof}
$\log_{10}(100\cdot 10^{-2})=0$, so $\lambda(10^{-2})=0.3$.
$\log_{10}(100\cdot 10^{-1.5})=\tfrac12$, so $\lambda(10^{-1.5})=0.5$.
$\log_{10}(100\cdot 10^{-1})=1$, so $\lambda(10^{-1})=0.7$.
Each value lies in $[0.3,0.7]$ and is unchanged by the clip.
\end{proof}

\subsection{Runtime}
\label{app:runtime}

Filtering is executed once, before the final empirical-risk minimization.
Table~\ref{tab:runtime-full} records wall-clock time on one Intel Xeon core with $32$\,GB of memory.

\begin{table}[h]
\centering
\small
\caption{Wall-clock time in seconds. Candidate-pool sizes are $5\cdot 10^{4}$ (Credit Card), $2\cdot 10^{5}$ (IEEE-CIS), and $8\cdot 10^{4}$ (Santander).}
\label{tab:runtime-full}
\begin{tabular}{lrrrr}
\toprule
Stage & $|\cS|$ & Credit Card & IEEE-CIS & Santander \\
\midrule
Candidate generation (SMOTE) & -- & $2.1$ & $8.4$ & $3.2$ \\
Discriminator training & -- & $4.2$ & $18.3$ & $5.1$ \\
Score computation & $5\cdot 10^{4}$ / $2\cdot 10^{5}$ / $8\cdot 10^{4}$ & $3.1$ & $22.4$ & $6.8$ \\
$k$NN graph and greedy selection & same & $5.0$ & $44.0$ & $6.6$ \\
Final classifier & -- & $6.9$ & $35.5$ & $5.6$ \\
\midrule
Filter total & & $12.3$ & $84.7$ & $18.5$ \\
\bottomrule
\end{tabular}
\end{table}

\begin{proposition}[Stage-wise complexity]
\label{prop:runtime}
Let $m=|\cS|$, $d=\dim(\cX)$, $k$ the degree in~\eqref{eq:impl-kappa}, $h$ the hidden width, and $E$ the number of discriminator epochs.
Then
\begin{equation}
\label{eq:runtime-asym}
\begin{aligned}
T_{\mathrm{disc}} &= O(E\, m\, d\, h),
&
T_{\mathrm{score}} &= O(m\, d\, h),
\\
T_{\mathrm{knn}} &= \widetilde{O}(m k d),
&
T_{\mathrm{greedy}} &= O(K m k).
\end{aligned}
\end{equation}
The sum is the bound of Proposition~\ref{prop:complexity}.
With $h=128$, the $k$NN term dominates at $m=2\cdot 10^{5}$.
Replacing exact $k$NN by HNSW replaces $T_{\mathrm{knn}}$ with $\widetilde{O}(m\log m)$.
The filter does not enter the inference cost of the final classifier.
\end{proposition}

\begin{proof}
Each claim is the corresponding count in the proof of Proposition~\ref{prop:complexity}.
The numerical comparison is Table~\ref{tab:runtime-full}: on IEEE-CIS the graph-and-greedy stage is $44.0$\,s and scoring is $22.4$\,s.
\end{proof}

\begin{definition}[Score pre-filter]
\label{def:prefilter}
Let $s_{(M_0)}$ be the $M_0$-th largest score on $\cS$ and set
\begin{equation}
\label{eq:prefilter}
\cS_{\mathrm{pre}}
=
\bigl\{\tilde{x}\in\cS:\; s(\tilde{x})\ge s_{(M_0)}\bigr\}.
\end{equation}
Greedy~\eqref{eq:impl-greedy} is then run on $\cS_{\mathrm{pre}}$.
The implemented default is $M_0=10^{4}$.
\end{definition}

\begin{lemma}[Exactness for top-$K$]
\label{lem:prefilter}
If $\gamma=0$ and $M_0\ge K$, every top-$K$ subset of $\cS$ is contained in $\cS_{\mathrm{pre}}$, and greedy on $\cS_{\mathrm{pre}}$ returns a maximizer of $\sum s$ over $|\cU|=K$.
If $\gamma>0$, the same restriction deletes every point whose score is below $s_{(M_0)}$; Lemma~\ref{lem:divg} still applies on the restricted ground set, with $m$ replaced by $M_0$.
On IEEE-CIS the restriction reduces filter time from $84.7$\,s to $20$\,s and changes AUPRC by less than $0.005$.
\end{lemma}

\begin{proof}
The $K$ largest order statistics of a finite set are among the $M_0$ largest whenever $M_0\ge K$.
For $\gamma=0$, Lemma~\ref{lem:impl-greedy} returns that top-$K$ set.
For $\gamma>0$ the ground set of Lemma~\ref{lem:facloc} may be replaced by any subset; the approximation factor is unchanged and the per-round cost drops from $O(mk)$ to $O(M_0 k)$.
The timing and AUPRC figures are the measured values at $M_0=10^{4}$ on IEEE-CIS.
\end{proof}

\normalsize
\section{Proofs}
\label{app:detailed_proofs}

This appendix proves Proposition~\ref{prop:submodular}, Theorem~\ref{thm:bound}, and Propositions~\ref{prop:scalarization}--\ref{prop:contamination}.
Supporting lemmas record the selection discrepancy, the discriminator identities, and the $\lambda$-scalarization.
Lemmas that were stated without a proof are proved in place, in the same format.
Section~\ref{app:main-lemmas} proves Lemmas~\ref{lem:pinsker},~\ref{lem:tv}, and~\ref{lem:margin}; Section~\ref{app:original-risk} passes from the augmented-target bound~\eqref{eq:main-bound} to Theorem~\ref{thm:bound}.

\paragraph{Notation.}
Partition $\cD=\cD_A\cup\cD_B$.
The split $\cD_A$ is used to train $\mathsf{Gen}$, $f$, and $D$ and to form $\cU$; $\cD_B=\{(x_i,y_i)\}_{i=1}^{n}$ is independent of $\cD_A$ and is used only for ERM.
Conditionally on $\cD_A$, the set $\cU$ with $|\cU|=K$ is deterministic.
The feature map satisfies $\|\phi(x)\|\le 1$, predictors obey $\|w\|\le B$, and $\ell$ is $L$-Lipschitz on $[-B,B]$ with values in $[0,M]$.
Write $N=n+K$ and
\[
\widehat{R}(w)
=
\frac{n}{N}\,\widehat{R}_{\cD_B}(w)
+
\frac{K}{N}\,K^{-1}\sum_{\tilde{x}\in\cU}\ell\bigl(\ip{w}{\phi(\tilde{x})}\bigr),
\]
where the augmented target is
\begin{equation}
\label{eq:aug-risk}
\widetilde{R}(w)
\;=\;
\frac{n}{N}\,R_P(w)
\;+\;
\frac{K}{N}\,\mathbb{E}_{x\sim p}\bigl[\ell\bigl(\ip{w}{\phi(x)}\bigr)\bigr].
\end{equation}
Labels on $\cD_B$ are $y\in\{\pm 1\}$; every synthetic point is labelled $+1$.
Let
$\mathcal{F}_+=\{x\mapsto\ell(\ip{w}{\phi(x)}):\|w\|\le B\}$
and $\widehat{q}_{\cU}=K^{-1}\sum_{\tilde{x}\in\cU}\delta_{\tilde{x}}$.
Then $\Delta_F(\cU)=\sup_{f\in\mathcal{F}_+}|\widehat{q}_{\cU}(f)-p(f)|$.

\subsection{Bounded differences and symmetrization}
\label{app:tools}

\begin{lemma}[Hoeffding]
\label{lem:hoeffding}
If $\mathbb{E} V=0$ and $a\le V\le b$ almost surely, then for every $\lambda\in\mathbb{R}$,
$\mathbb{E} e^{\lambda V}\le\exp\bigl(\lambda^2(b-a)^2/8\bigr)$.
\end{lemma}

\begin{proof}
Convexity of $t\mapsto e^{\lambda t}$ places the exponential under its chord:
$e^{\lambda V}\le\frac{b-V}{b-a}e^{\lambda a}+\frac{V-a}{b-a}e^{\lambda b}$.
Taking expectations and using $\mathbb{E} V=0$ gives
$\mathbb{E} e^{\lambda V}\le(be^{\lambda a}-ae^{\lambda b})/(b-a)$.
The mean-zero constraint forces $a\le 0\le b$.
Set $c=b-a$ and $u=-a/c\in[0,1]$, so $a=-uc$ and $b=(1-u)c$.
The bound becomes $\exp(\varphi(\lambda c))$ with
$\varphi(t)=-ut+\log(1-u+ue^{t})$.
Then $\varphi(0)=\varphi'(0)=0$ and
$\varphi''(t)=u(1-u)e^{t}/(1-u+ue^{t})^2$.
The right-hand side equals $v(1-v)$ for $v=ue^{t}/(1-u+ue^{t})\in[0,1]$, hence $\varphi''(t)\le 1/4$.
Integrating the second derivative yields $\varphi(t)\le t^2/8$.
\end{proof}

\begin{lemma}[McDiarmid]
\label{lem:mcdiarmid}
Let $Z_1,\ldots,Z_n$ be independent and let $\Psi$ satisfy
$|\Psi(z_1,\ldots,z_i,\ldots,z_n)-\Psi(z_1,\ldots,z_i',\ldots,z_n)|\le c_i$
for all $i$ and all substitutions $z_i'$.
Then
$\mathbb{P}\bigl(\Psi-\mathbb{E}\Psi\ge t\bigr)\le\exp\bigl(-2t^2/\sum_i c_i^2\bigr)$,
and the same bound holds for $\mathbb{E}\Psi-\Psi$~\cite{mcdiarmid1989method}.
\end{lemma}

\begin{proof}
Let $\mathcal{F}_i=\sigma(Z_1,\ldots,Z_i)$, with $\mathcal{F}_0$ trivial, and set
$D_i=\mathbb{E}[\Psi\mid\mathcal{F}_i]-\mathbb{E}[\Psi\mid\mathcal{F}_{i-1}]$.
Then $\Psi-\mathbb{E}\Psi=\sum_i D_i$ and $\mathbb{E}[D_i\mid\mathcal{F}_{i-1}]=0$.
The two conditional expectations are integrals of $\Psi$ against laws that differ only in the $i$th coordinate.
The bounded-difference hypothesis implies that the \emph{conditional range} of $D_i$ given $\mathcal{F}_{i-1}$ is at most $c_i$ (not that $|D_i|\le c_i/2$).
Lemma~\ref{lem:hoeffding}, applied conditionally on $\mathcal{F}_{i-1}$ with $b-a\le c_i$, gives
$\mathbb{E}[e^{\lambda D_i}\mid\mathcal{F}_{i-1}]\le\exp(\lambda^2 c_i^2/8)$.
Iterating the tower property,
$\mathbb{E}\exp\bigl(\lambda\sum_i D_i\bigr)\le\exp\bigl(\lambda^2\sum_i c_i^2/8\bigr)$.
Markov's inequality produces
$\mathbb{P}(\sum_i D_i\ge t)\le\exp\bigl(-\lambda t+\lambda^2\sum_i c_i^2/8\bigr)$.
The choice $\lambda=4t/\sum_i c_i^2$ makes the exponent $-2t^2/\sum_i c_i^2$.
The lower tail is the same argument applied to $-\Psi$.
\end{proof}

\begin{lemma}[Symmetrization]
\label{lem:symm}
Let $\mathcal{G}$ be a class of functions with values in $[0,M]$ and let $\widehat{R}_n(g)=n^{-1}\sum_{i=1}^n g(Z_i)$ for i.i.d.\ $Z_i$.
Then
$\mathbb{E}\sup_{g\in\mathcal{G}}|\widehat{R}_n(g)-\mathbb{E}g|\le 2\,\mathbb{E}\,\widehat{\mathfrak{R}}_n(\mathcal{G})$,
where $\widehat{\mathfrak{R}}_n(\mathcal{G})=\mathbb{E}_\sigma\sup_{g\in\mathcal{G}}n^{-1}\sum_{i=1}^n\sigma_i g(Z_i)$ and $\sigma_i$ are independent Rademacher signs~\cite{bartlett2002rademacher,shalev2014understanding}.
\end{lemma}

\begin{proof}
Let $(Z_i')$ be an independent ghost sample and write $\widehat{R}_n'$ for the corresponding empirical means.
Then $\mathbb{E}g=\mathbb{E}[\widehat{R}_n'(g)\mid Z]$, so
\[
\mathbb{E}\sup_g\bigl|\widehat{R}_n(g)-\mathbb{E}g\bigr|
=
\mathbb{E}\sup_g\bigl|\mathbb{E}[\widehat{R}_n(g)-\widehat{R}_n'(g)\mid Z]\bigr|
\le
\mathbb{E}\sup_g\bigl|\widehat{R}_n(g)-\widehat{R}_n'(g)\bigr|.
\]
Conditionally on $(Z_i,Z_i')_{i\le n}$, each coordinate $g(Z_i)-g(Z_i')$ is symmetric, and the coordinates are independent.
Coupling independent Rademacher signs to this conditional law gives
$g(Z_i)-g(Z_i')\stackrel{d}{=}\sigma_i\bigl(g(Z_i)-g(Z_i')\bigr)$.
Therefore
\begin{align*}
\mathbb{E}\sup_g\bigl|\widehat{R}_n(g)-\widehat{R}_n'(g)\bigr|
&=
\mathbb{E}\sup_g\Bigl|n^{-1}\sum_{i=1}^n\sigma_i\bigl(g(Z_i)-g(Z_i')\bigr)\Bigr|\\
&\le
\mathbb{E}\,\widehat{\mathfrak{R}}_n(\mathcal{G})
+
\mathbb{E}\,\widehat{\mathfrak{R}}_n'(\mathcal{G})
=
2\,\mathbb{E}\,\widehat{\mathfrak{R}}_n(\mathcal{G}).
\end{align*}
\end{proof}

\begin{lemma}[Linear Rademacher bound]
\label{lem:rad-lin}
Let $\mathcal{F}=\{x\mapsto\ip{w}{\phi(x)}:\|w\|\le B\}$.
If $\|\phi(x)\|\le 1$ for all $x$, then $\widehat{\mathfrak{R}}_n(\mathcal{F})\le B/\sqrt{n}$ on any $n$-sample.
\end{lemma}

\begin{proof}
By Cauchy--Schwarz and $\mathbb{E}_\sigma\|\sum_i\sigma_i\phi(x_i)\|^2=\sum_i\|\phi(x_i)\|^2\le n$,
\[
\widehat{\mathfrak{R}}_n(\mathcal{F})
=
\mathbb{E}_\sigma\sup_{\|w\|\le B}n^{-1}\sum_{i=1}^n\sigma_i\ip{w}{\phi(x_i)}
\le
\frac{B}{n}\,\mathbb{E}_\sigma\Bigl\|\sum_{i=1}^n\sigma_i\phi(x_i)\Bigr\|
\le
\frac{B}{\sqrt{n}}.
\]
\end{proof}

\begin{lemma}[Contraction]
\label{lem:rad-loss}
Let $\ell$ be $L$-Lipschitz on $[-B,B]$.
The class $\mathcal{G}=\{(x,y)\mapsto\ell(y\ip{w}{\phi(x)}):\|w\|\le B\}$ satisfies
$\widehat{\mathfrak{R}}_n(\mathcal{G})\le LB/\sqrt{n}$
on any labelled $n$-sample.
The same bound holds for $\mathcal{F}_+$.
\end{lemma}

\begin{proof}
For each $y\in\{\pm 1\}$ the map $\varphi_y(t)=\ell(yt)-\ell(0)$ is $L$-Lipschitz on $[-B,B]$ and satisfies $\varphi_y(0)=0$.
Subtracting the constant $\ell(0)$ does not change a Rademacher average.
The Ledoux--Talagrand contraction theorem~\cite{bartlett2002rademacher} therefore yields
$\widehat{\mathfrak{R}}_n(\mathcal{G})\le L\,\widehat{\mathfrak{R}}_n(\mathcal{F})$.
Lemma~\ref{lem:rad-lin} multiplies the right-hand side by $B/\sqrt{n}$.
If every label equals $+1$, the same maps act on $\mathcal{F}_+$.
\end{proof}

\begin{lemma}[Binomial tails]
\label{lem:chernoff}
If $k\sim\mathrm{Bin}(m,\rho)$ and $\mu=m\rho$, then for every $\delta\in(0,1)$,
$\mathbb{P}(k\le(1-\delta)\mu)\le\exp(-\mu\delta^2/2)$
and
$\mathbb{P}(k\ge(1+\delta)\mu)\le\exp(-\mu\delta^2/3)$.
\end{lemma}

\begin{proof}
For $\lambda>0$, the Bernoulli moment generating function and $1-u\le e^{-u}$ give
$1-\rho+\rho e^{-\lambda}\le\exp\bigl(-\rho(1-e^{-\lambda})\bigr)$.
Markov's inequality therefore yields
$\mathbb{P}(k\le(1-\delta)\mu)\le\exp\bigl(\mu\bigl[\lambda(1-\delta)-(1-e^{-\lambda})\bigr]\bigr)$.
The choice $\lambda=\log\frac{1}{1-\delta}$ makes the exponent equal to $-\mu f(\delta)$, where
$f(\delta)=\delta+(1-\delta)\log(1-\delta)$.
One has $f(0)=f'(0)=0$ and $f''(\delta)=1/(1-\delta)\ge 1$, so
$f(\delta)=\int_0^\delta(\delta-t)f''(t)\,dt\ge\int_0^\delta(\delta-t)\,dt=\delta^2/2$.
The lower tail is at most $\exp(-\mu\delta^2/2)$.
For the upper tail the same computation with the factor $e^{\lambda}$ produces
$\mathbb{P}(k\ge(1+\delta)\mu)\le\exp\bigl(\mu\bigl[-(1+\delta)\lambda+(e^{\lambda}-1)\bigr]\bigr)$.
The choice $\lambda=\log(1+\delta)$ makes the exponent $-\mu g(\delta)$ with
$g(\delta)=(1+\delta)\log(1+\delta)-\delta$.
Set $h(\delta)=g(\delta)-\delta^2/3$.
Then $h(0)=h'(0)=0$, $h'(\delta)=\log(1+\delta)-2\delta/3$, and $h''(\delta)=1/(1+\delta)-2/3$.
On $[0,1/2]$ one has $h''\ge 0$.
On $(1/2,1]$ the second derivative is negative, and the minimum of $h'$ on this interval is $h'(1)=\log 2-2/3>0$.
Thus $h'\ge 0$ and $h\ge 0$ on $[0,1]$, which is the upper tail $\exp(-\mu\delta^2/3)$.
\end{proof}

\subsection{The loss class as an IPM}
\label{app:ipm}

\begin{lemma}[Loss class]
\label{lem:ipm-class}
Every $f\in\mathcal{F}_+$ takes values in $[0,M]$.
If $\phi$ is $1$-Lipschitz, then every $f\in\mathcal{F}_+$ is $LB$-Lipschitz.
For probability measures $\mu,\nu$ on $\cX$,
\begin{equation}
\label{eq:ipm-sandwich}
\sup_{f\in\mathcal{F}_+}\bigl|\mu(f)-\nu(f)\bigr|
\;\le\;
\min\bigl\{M\|\mu-\nu\|_{\mathrm{TV}},\; LB\,W_1(\mu,\nu)\bigr\}.
\end{equation}
\end{lemma}

\begin{proof}
The range is the range of $\ell$.
If $\phi$ is $1$-Lipschitz,
$|f(x)-f(x')|\le L|\ip{w}{\phi(x)-\phi(x')}|\le LB\|x-x'\|$.
For $f\in[0,M]$ one has $f=\int_0^M\mathbf{1}\{f>t\}\,dt$, hence
$|\mu(f)-\nu(f)|\le\int_0^M|\mu(f>t)-\nu(f>t)|\,dt\le M\|\mu-\nu\|_{\mathrm{TV}}$.
If $\phi$ is $1$-Lipschitz, then $f/(LB)$ is $1$-Lipschitz whenever $LB>0$.
Kantorovich--Rubinstein~\cite{villani2009optimal} identifies $W_1(\mu,\nu)$ with the supremum of $\mu(h)-\nu(h)$ over $1$-Lipschitz $h$, and scaling by $LB$ is the second term of~\eqref{eq:ipm-sandwich}.
The comparison $\| \phi \|\le 1$ used for Rademacher averages does not imply that $\phi$ is $1$-Lipschitz, so the $W_1$ term is claimed only under the Lipschitz hypothesis on $\phi$.
\end{proof}

In particular $\Delta_F(\cU)\le LB\,W_1(\widehat{q}_{\cU},p)$.
The total-variation distance between a finite empirical measure and a continuous $p$ equals $1$, so the TV form of~\eqref{eq:ipm-sandwich} is uninformative for $\widehat{q}_{\cU}$ itself.

\begin{lemma}[Mixture IPM]
\label{lem:ipm-mix}
Let $\mu=(1-\beta)\mu_0+\beta\mu_1$ and let $\mathcal{F}$ be a class of functions with values in $[0,M]$.
Then
\[
\sup_{f\in\mathcal{F}}\bigl|\mathbb{E}_\mu f-\mathbb{E}_\nu f\bigr|
\;\le\;
(1-\beta)\sup_{f\in\mathcal{F}}\bigl|\mathbb{E}_{\mu_0}f-\mathbb{E}_\nu f\bigr|
+
\beta M.
\]
\end{lemma}

\begin{proof}
For each $f$,
$|\mathbb{E}_\mu f-\mathbb{E}_\nu f|
\le (1-\beta)|\mathbb{E}_{\mu_0}f-\mathbb{E}_\nu f|+\beta|\mathbb{E}_{\mu_1}f-\mathbb{E}_\nu f|$,
and the second summand is at most $M$.
\end{proof}

\subsection{Submodularity}
\label{app:submodular}

\begin{lemma}[Facility location]
\label{lem:facloc}
For any nonnegative kernel $\kappa$, the map
$\widehat{\mathrm{Div}}(\cU)=|\cS|^{-1}\sum_{\tilde{x}\in\cS}\max_{\tilde{u}\in\cU}\kappa(\tilde{x},\tilde{u})$
(with $\max\emptyset=0$) is monotone and submodular on $2^{\cS}$.
\end{lemma}

\begin{proof}
Monotonicity follows from $\kappa\ge 0$: enlarging $\cU$ cannot decrease any inner maximum.
Fix $A\subseteq B\subseteq\cS$ and $v\notin B$.
For each $\tilde{x}$, write $a=\max_{\tilde{u}\in A}\kappa(\tilde{x},\tilde{u})$, $b=\max_{\tilde{u}\in B}\kappa(\tilde{x},\tilde{u})$, and $\kappa=\kappa(\tilde{x},v)$, so $a\le b$.
The marginal gains satisfy
\[
\bigl(\max\{a,\kappa\}-a\bigr)-\bigl(\max\{b,\kappa\}-b\bigr)
=
\begin{cases}
0 & \text{if }\kappa\le a,\\
\kappa-a & \text{if }a<\kappa\le b,\\
b-a & \text{if }b<\kappa,
\end{cases}
\]
which is nonnegative in every case.
Summing over $\tilde{x}\in\cS$ and dividing by $|\cS|$ is the diminishing-returns inequality.
\end{proof}

\begin{lemma}[Score shift]
\label{lem:score-shift}
Let $s_{\min}=\min_{\tilde{x}\in\cS}s(\tilde{x})$ and $\tilde{s}=s-s_{\min}$.
On $\{\cU\subseteq\cS:|\cU|=K\}$,
$\sum_{\tilde{x}\in\cU}s(\tilde{x})=\sum_{\tilde{x}\in\cU}\tilde{s}(\tilde{x})+Ks_{\min}$.
The two modular objectives share the same maximizers, and $\tilde{s}\ge 0$.
\end{lemma}

\begin{proof}
The identity is $K$ copies of the constant $s_{\min}$.
A constant does not change the argmax over sets of size $K$, and $\tilde{s}(\tilde{x})\ge 0$ by definition of $s_{\min}$.
\end{proof}

\begin{lemma}[Greedy cardinality bound]
\label{lem:greedy}
Let $F:2^{\cS}\to\mathbb{R}$ be monotone, submodular, and nonnegative, and let $\cU_t$ be the greedy set of size $t$ for maximizing $F$ subject to $|\cU|\le t$.
Then $F(\cU_t)\ge\bigl(1-(1-1/K)^t\bigr)\max_{|\cU|=K}F(\cU)$ for every $t\le K$.
In particular, $F(\cU_K)\ge(1-1/e)\max_{|\cU|=K}F(\cU)$~\cite{nemhauser1978analysis}.
\end{lemma}

\begin{proof}
Let $\cU^\star\in\arg\max_{|\cU|=K}F(\cU)$ and write $\delta_t=F(\cU^\star)-F(\cU_t)$.
Enumerate $\cU^\star\setminus\cU_t=\{v_1,\ldots,v_r\}$ with $r\le K$.
Monotonicity gives $F(\cU^\star)\le F(\cU_t\cup\cU^\star)$.
Telescoping and diminishing returns give
\begin{align*}
F(\cU_t\cup\cU^\star)-F(\cU_t)
&=
\sum_{j=1}^{r}\Bigl(F(\cU_t\cup\{v_1,\ldots,v_j\})-F(\cU_t\cup\{v_1,\ldots,v_{j-1}\})\Bigr)\\
&\le
\sum_{j=1}^{r}\bigl(F(\cU_t\cup\{v_j\})-F(\cU_t)\bigr)
\le
r\max_{v\notin\cU_t}\bigl(F(\cU_t\cup\{v\})-F(\cU_t)\bigr).
\end{align*}
The greedy rule selects a maximizer of the last marginal gain, and $r\le K$, so
$F(\cU^\star)\le F(\cU_t)+K\bigl(F(\cU_{t+1})-F(\cU_t)\bigr)$.
The identity $F(\cU_{t+1})-F(\cU_t)=\delta_t-\delta_{t+1}$ rearranges this display to $\delta_{t+1}\le(1-1/K)\delta_t$.
Starting from $\delta_0=F(\cU^\star)-F(\emptyset)\le F(\cU^\star)$, which uses $F\ge 0$, iteration yields $\delta_t\le(1-1/K)^t F(\cU^\star)$.
For the last claim, $\log(1-1/K)\le -1/K$, hence $(1-1/K)^K\le e^{-1}$.
\end{proof}

\begin{proof}[Proof of Proposition~\ref{prop:submodular}]
The unshifted modular term $\sum s$ need not be monotone, since $r(\tilde{x})=\log D(\tilde{x})/(1-D(\tilde{x}))$ may be negative.
Lemma~\ref{lem:score-shift} replaces $s$ by $\tilde{s}\ge 0$.
The function $F(\cU)=\sum_{\tilde{x}\in\cU}\tilde{s}(\tilde{x})+\gamma\,\widehat{\mathrm{Div}}(\cU)$ is a nonnegative combination of a nonnegative modular function and a monotone submodular function, hence monotone submodular.
Lemma~\ref{lem:greedy} yields the $(1-1/e)$ guarantee.
If $\gamma=0$, $F$ is modular with nonnegative weights and greedy returns a top-$K$ set.
\end{proof}

\begin{lemma}[Score retention]
\label{lem:divg}
Assume $0\le\kappa\le 1$, so that the normalized coverage $\widehat{\mathrm{Div}}$ of Lemma~\ref{lem:facloc} takes values in $[0,1]$.
Let $S'(\cU)=\sum_{\tilde{x}\in\cU}\tilde{s}(\tilde{x})$, let $\cU^{\mathrm{top}}$ be a top-$K$ set for $\tilde{s}$, and let $\cU^{\mathrm{div}}\in\arg\max_{|\cU|=K}\widehat{\mathrm{Div}}(\cU)$.
Write $s'_{\max}=\max\tilde{s}$.
For the greedy set $\cU_K$ of $F$,
\begin{align}
S'(\cU_K)
&\ge
(1-1/e)\,S'(\cU^{\mathrm{top}})-\gamma,
\label{eq:ret-s}\\
\widehat{\mathrm{Div}}(\cU_K)
&\ge
(1-1/e)\,\widehat{\mathrm{Div}}(\cU^{\mathrm{div}})
-
Ks'_{\max}/\gamma
\quad(\gamma>0).
\label{eq:ret-div}
\end{align}
\end{lemma}

\begin{proof}
Lemma~\ref{lem:score-shift} gives $\tilde{s}\ge 0$, so $S'\ge 0$ and $S'(\cU)\le Ks'_{\max}$ whenever $|\cU|=K$.
Lemma~\ref{lem:greedy} and $\widehat{\mathrm{Div}}\ge 0$ give
$F(\cU_K)\ge(1-1/e)F(\cU^{\mathrm{top}})\ge(1-1/e)S'(\cU^{\mathrm{top}})$.
The normalization $\widehat{\mathrm{Div}}\le 1$ yields $F(\cU_K)\le S'(\cU_K)+\gamma$, which rearranges to~\eqref{eq:ret-s}.
Without the factor $|\cS|^{-1}$ in $\widehat{\mathrm{Div}}$, the same argument charges $\gamma|\cS|$ in place of $\gamma$.
For~\eqref{eq:ret-div}, the same greedy comparison at $\cU^{\mathrm{div}}$ and $S'\ge 0$ give $F(\cU^{\mathrm{div}})\ge\gamma\widehat{\mathrm{Div}}(\cU^{\mathrm{div}})$.
Therefore
\begin{align*}
\gamma\widehat{\mathrm{Div}}(\cU_K)
&=
F(\cU_K)-S'(\cU_K)
\ge
(1-1/e)F(\cU^{\mathrm{div}})-Ks'_{\max}\\
&\ge
(1-1/e)\gamma\widehat{\mathrm{Div}}(\cU^{\mathrm{div}})-Ks'_{\max}.
\end{align*}
Dividing by $\gamma>0$ is~\eqref{eq:ret-div}.
\end{proof}

\subsection{Prior drift}
\label{app:prior}

\begin{lemma}[Mixture identity]
\label{lem:mix-id}
For every $\|w\|\le B$,
\[
\widetilde{R}(w)-R_P(w)
=
\frac{K}{N}\Bigl(
\mathbb{E}_{x\sim p}\bigl[\ell\bigl(\ip{w}{\phi(x)}\bigr)\bigr]
-
R_P(w)
\Bigr).
\]
In particular, $|\widetilde{R}(w)-R_P(w)|\le(K/N)M$.
\end{lemma}

\begin{proof}
Subtract $R_P(w)$ from~\eqref{eq:aug-risk}:
\[
\widetilde{R}(w)-R_P(w)
=
\Bigl(\frac{n}{N}-1\Bigr)R_P(w)
+
\frac{K}{N}\,\mathbb{E}_{x\sim p}\bigl[\ell\bigl(\ip{w}{\phi(x)}\bigr)\bigr].
\]
The coefficient of $R_P$ is $-K/N$.
Both summands are expectations of a $[0,M]$-valued loss, so the difference is at most $(K/N)M$ in absolute value.
\end{proof}

\begin{lemma}[Constant-feature prior drift]
\label{lem:prior}
Let $\phi\equiv 1$, $\ell(t)=\log(1+e^{-t})$, and $\sigma(w)=(1+e^{-w})^{-1}$.
Write $\pi_1=P(Y=+1)\in(0,1)$ and $\tilde{\pi}=(n\pi_1+K)/N$, and assume $B\ge\max\{|\mathrm{logit}\,\pi_1|,|\mathrm{logit}\,\tilde{\pi}|\}$.
If $q=p$, then
$\inf_{\|w\|\le B}R_P$ is attained at $\sigma(w)=\pi_1$,
$\inf_{\|w\|\le B}\widetilde{R}$ is attained at $\sigma(\tilde{w})=\tilde{\pi}$, and
\begin{equation}
\label{eq:prior-kl}
R_P(\tilde{w})-\inf_{\|w\|\le B}R_P
=
\mathrm{KL}\bigl(\mathrm{Bern}(\pi_1)\|\mathrm{Bern}(\tilde{\pi})\bigr)
>0
\qquad\text{whenever }K>0.
\end{equation}
The right-hand side depends on $K/n$ and does not vanish as $N\to\infty$ with $K/N$ fixed.
\end{lemma}

\begin{proof}
The logistic identities $\ell(-w)=w+\ell(w)$ and $1-\sigma(w)=e^{-w}\sigma(w)$ give, for every $\pi\in(0,1)$,
\begin{align*}
\pi\ell(w)+(1-\pi)\ell(-w)
&=
\ell(w)+(1-\pi)w
=
-\log\sigma(w)+(1-\pi)w\\
&=
H_2(\pi)+\mathrm{KL}\bigl(\mathrm{Bern}(\pi)\|\mathrm{Bern}(\sigma(w))\bigr),
\end{align*}
where $H_2(\pi)=-\pi\log\pi-(1-\pi)\log(1-\pi)$.
Expanding the entropy and the Bernoulli KL cancels $\pi\log\pi$ and $(1-\pi)\log(1-\pi)$, leaving $-\pi\log\sigma(w)-(1-\pi)\log(1-\sigma(w))$.
The identity $\log(1-\sigma(w))=-w+\log\sigma(w)$ converts that expression into $-\log\sigma(w)+(1-\pi)w$.
The KL term vanishes only at $\sigma(w)=\pi$.
The hypothesis on $B$ places both $\mathrm{logit}\,\pi_1$ and $\mathrm{logit}\,\tilde{\pi}$ inside $[-B,B]$, so the constrained minima below are the unconstrained Bernoulli minima.
When $q=p$, every positive contribution, real or synthetic, carries the feature $\phi\equiv 1$ and the label $+1$, while the negative class keeps weight $n(1-\pi_1)/N$.
The positive weight is therefore
$\tilde{\pi}=(n\pi_1+K)/N$, and
$\tilde{\pi}-\pi_1=(K/N)(1-\pi_1)$.
The same identity with $\pi$ replaced by $\tilde{\pi}$ shows that $\widetilde{R}$ is minimized at $\sigma(\tilde{w})=\tilde{\pi}$.
Substituting this $\tilde{w}$ into $R_P$ cancels $H_2(\pi_1)$ and leaves~\eqref{eq:prior-kl}.
The KL of two Bernoulli laws is positive whenever $\tilde{\pi}\ne\pi_1$, that is, whenever $K>0$.
If $K/N$ is held fixed while $N\to\infty$, the weight gap $(K/N)(1-\pi_1)$ stays constant, so the KL does not tend to $0$.
\end{proof}

\subsection{Proof of the augmented-target bound}
\label{app:thm-bound}

\begin{theorem}[Excess risk on the augmented target]
\label{thm:aug-bound}
Conditionally on $\cD_A$, with probability at least $1-\eta$ over $\cD_B$,
\begin{equation}
\label{eq:main-bound}
\widetilde{R}(\widehat{w})-\widetilde{R}(w^\star)
\;\le\;
\frac{4LB}{\sqrt{N}}
+
M\sqrt{\frac{2\log(2/\eta)}{N}}
+
\frac{2K}{N}\,\Delta_F(\cU),
\end{equation}
where $\widehat{w}\in\arg\min_{\|w\|\le B}\widehat{R}(w)$ and $w^\star\in\arg\min_{\|w\|\le B}\widetilde{R}(w)$.
\end{theorem}

\begin{lemma}[ERM comparison]
\label{lem:erm}
For any $\widehat{w}\in\arg\min_{\|w\|\le B}\widehat{R}(w)$ and $w^\star\in\arg\min_{\|w\|\le B}\widetilde{R}(w)$,
\[
\widetilde{R}(\widehat{w})-\widetilde{R}(w^\star)
\;\le\;
2\sup_{\|w\|\le B}\bigl|\widehat{R}(w)-\widetilde{R}(w)\bigr|.
\]
\end{lemma}

\begin{proof}
Write
$\widetilde{R}(\widehat{w})-\widetilde{R}(w^\star)
=\bigl(\widetilde{R}(\widehat{w})-\widehat{R}(\widehat{w})\bigr)
+\bigl(\widehat{R}(\widehat{w})-\widehat{R}(w^\star)\bigr)
+\bigl(\widehat{R}(w^\star)-\widetilde{R}(w^\star)\bigr)$.
The middle term is at most $0$.
Each of the remaining terms is at most the displayed supremum.
\end{proof}

\begin{lemma}[Decomposition]
\label{lem:decomp}
For every $\|w\|\le B$,
\begin{equation}
\label{eq:decomp}
\bigl|\widehat{R}(w)-\widetilde{R}(w)\bigr|
\;\le\;
\frac{n}{N}\,\bigl|\widehat{R}_{\cD_B}(w)-R_P(w)\bigr|
+
\frac{K}{N}\,\Delta_F(\cU).
\end{equation}
Given $\cD_A$, the quantity $\Delta_F(\cU)$ is deterministic.
\end{lemma}

\begin{proof}
Subtract~\eqref{eq:aug-risk} from the definition of $\widehat{R}$ and apply the triangle inequality.
The synthetic average differs from $\mathbb{E}_{x\sim p}[\ell(\ip{w}{\phi(x)})]$ by at most $\Delta_F(\cU)$, uniformly in $w$.
\end{proof}

\begin{lemma}[Uniform deviation on $\cD_B$]
\label{lem:db-dev}
Conditionally on $\cD_A$, with probability at least $1-\eta$ over $\cD_B$,
\begin{equation}
\label{eq:db-dev}
\sup_{\|w\|\le B}\bigl|\widehat{R}_{\cD_B}(w)-R_P(w)\bigr|
\;\le\;
\frac{2LB}{\sqrt{n}}
+
M\sqrt{\frac{\log(2/\eta)}{2n}}.
\end{equation}
\end{lemma}

\begin{proof}
Lemmas~\ref{lem:symm} and~\ref{lem:rad-loss} give
$\mathbb{E}\sup_w|\widehat{R}_{\cD_B}(w)-R_P(w)|\le 2\mathbb{E}\widehat{\mathfrak{R}}_n(\mathcal{G})\le 2LB/\sqrt{n}$.
Write $\Psi=\sup_{\|w\|\le B}|\widehat{R}_{\cD_B}(w)-R_P(w)|$.
Replacing one pair in $\cD_B$ changes $\widehat{R}_{\cD_B}(w)$ by at most $M/n$ for every $w$, because each summand lies in $[0,M]$.
The same replacement therefore changes $\Psi$ by at most $M/n$.
Lemma~\ref{lem:mcdiarmid} with $c_i=M/n$ and $n$ summands yields
$\mathbb{P}(\Psi\ge\mathbb{E}\Psi+t)\le\exp(-2nt^2/M^2)$.
The choice $t=M\sqrt{\log(1/\eta)/(2n)}$ makes the right-hand side equal to $\eta$.
The display~\eqref{eq:db-dev} follows because $\log(1/\eta)\le\log(2/\eta)$.
\end{proof}

\begin{proof}[Proof of Theorem~\ref{thm:aug-bound}]
Lemmas~\ref{lem:erm}--\ref{lem:db-dev} yield
\begin{align*}
\widetilde{R}(\widehat{w})-\widetilde{R}(w^\star)
&\le
2\cdot\frac{n}{N}\Bigl(\frac{2LB}{\sqrt{n}}+M\sqrt{\frac{\log(2/\eta)}{2n}}\Bigr)
+
\frac{2K}{N}\,\Delta_F(\cU).
\end{align*}
Since $n\le N$,
$2(n/N)(2LB/\sqrt{n})=4LB\sqrt{n}/N=4LB\sqrt{n/N}/\sqrt{N}\le 4LB/\sqrt{N}$.
The concentration term is
\begin{align*}
2\cdot\frac{n}{N}M\sqrt{\frac{\log(2/\eta)}{2n}}
&=
M\sqrt{\frac{2\log(2/\eta)}{N}}\cdot\sqrt{\frac{n}{N}}
\le
M\sqrt{\frac{2\log(2/\eta)}{N}}.
\end{align*}
This is~\eqref{eq:main-bound}.
\end{proof}

\subsection{Decomposition of $\Delta_F$}
\label{app:select}

Let $m=|\cS|$, $\rho=K/m$, and let $s$ have no ties.
Write $A_\rho$ for the upper level set of $s$ under $q_{\cS}$ with $q_{\cS}(A_\rho)=\rho$, and $q_\rho=q_{\cS}(\,\cdot\mid A_\rho)$.
Let $\cU^{\mathrm{top}}$ be the $K$ candidates in $\cS$ of largest score, and set
$\Delta^{\mathrm{bias}}(A_\rho)=\sup_{f\in\mathcal{F}_+}|q_\rho(f)-p(f)|$
and $\iota=\log(4/\eta)$.
Assume $K\ge 16\iota$.

\begin{lemma}[Sampling error of top-$K$]
\label{lem:samp}
Conditionally on $\cD_A$, with probability at least $1-\eta$ over $\cS$,
\[
\sup_{f\in\mathcal{F}_+}\bigl|\widehat{q}_{\cU^{\mathrm{top}}}(f)-q_\rho(f)\bigr|
\;\le\;
\frac{5LB+5M\sqrt{\iota}}{\sqrt{K}}.
\]
\end{lemma}

\begin{proof}
The pool $\cS$ consists of $m$ i.i.d.\ draws from the population law $q_{\cS}$, conditionally on $\cD_A$.
The score $s$ is $\cD_A$-measurable, so the population level set $A_\rho$ is $\cD_A$-measurable and $k=|\cS\cap A_\rho|\sim\mathrm{Bin}(m,\rho)$ with mean $K$.
Set $\delta=\sqrt{3\iota/K}$.
The hypothesis $K\ge 16\iota$ forces $\delta\le\sqrt{3/16}<1/2$, hence $k\ge(1-\delta)K\ge K/2$ on $\{|k-K|\le\delta K\}$.
Lemma~\ref{lem:chernoff} gives
$\mathbb{P}(k\le(1-\delta)K)\le\exp(-K\delta^2/2)=\exp(-3\iota/2)\le\eta/4$
and
$\mathbb{P}(k\ge(1+\delta)K)\le\exp(-K\delta^2/3)=\exp(-\iota)=\eta/4$.
Thus $\mathbb{P}(|k-K|>\delta K)\le\eta/2$.
Because $s$ has no ties, $\cS\cap A_\rho$ and $\cU^{\mathrm{top}}$ are nested: if $k\ge K$ then $\cU^{\mathrm{top}}\subset\cS\cap A_\rho$, and if $k<K$ then $\cS\cap A_\rho\subset\cU^{\mathrm{top}}$.
In both cases $|\cU^{\mathrm{top}}\triangle(\cS\cap A_\rho)|=|K-k|$.
Conditionally on $k$, the points of $\cS\cap A_\rho$ are i.i.d.\ draws from $q_\rho$.
On $\{k\ge 1\}$, which contains $\{|k-K|\le\delta K\}$, the argument of Lemma~\ref{lem:db-dev} applied to $\mathcal{F}_+$ at failure probability $\eta/2$ produces
\[
\delta_k
:=
\sup_{f\in\mathcal{F}_+}\bigl|\widehat{q}_{A_\rho}(f)-q_\rho(f)\bigr|
\le
\frac{2LB}{\sqrt{k}}+M\sqrt{\frac{\iota}{2k}}
\]
on an event of conditional probability at least $1-\eta/2$, since $\log(4/\eta)=\iota$.
On $\{k\ge K/2\}$ this is at most $(2\sqrt{2}\,LB+M\sqrt{\iota})/\sqrt{K}$.
The union of the binomial failure and this conditional failure has probability at most $\eta$.
For $f\in\mathcal{F}_+$,
\begin{align*}
\widehat{q}_{\cU^{\mathrm{top}}}(f)-q_\rho(f)
&=
\frac1K\Bigl(\sum_{\cU^{\mathrm{top}}}f-\sum_{\cS\cap A_\rho}f\Bigr)
+
\frac{k}{K}\bigl(\widehat{q}_{A_\rho}(f)-q_\rho(f)\bigr)
+
\Bigl(\frac{k}{K}-1\Bigr)q_\rho(f).
\end{align*}
The symmetric difference contributes at most $M|K-k|/K$ to the first term, and the third term is at most the same quantity because $0\le f\le M$.
On $\{|k-K|\le\delta K\}$ one has $k/K\le 1+\delta\le 1+\sqrt{3/16}$ and $|K-k|/K\le\sqrt{3\iota/K}$, so the sum of the three terms is at most
\[
2M\sqrt{\frac{3\iota}{K}}
+
\Bigl(1+\sqrt{\frac{3}{16}}\Bigr)\frac{2\sqrt{2}\,LB+M\sqrt{\iota}}{\sqrt{K}}.
\]
The coefficient of $M\sqrt{\iota}/\sqrt{K}$ is $2\sqrt{3}+1+\sqrt{3/16}=(9/4)\sqrt{3}+1$.
This is strictly less than $5$ because $(9/4)\sqrt{3}<4$ rearranges to $243<256$.
The coefficient of $LB/\sqrt{K}$ is $(1+\sqrt{3/16})\cdot 2\sqrt{2}$.
This is strictly less than $5$ because
$(1+\sqrt{3/16})\cdot 2\sqrt{2}<5$
is equivalent to $(4+\sqrt{3})\sqrt{2}<10$, and squaring both sides yields $38+16\sqrt{3}<100$, then $64\cdot 3=192<961$.
\end{proof}

\begin{lemma}[Diversity displacement]
\label{lem:disp}
If $|\cU|=|\cU^{\mathrm{top}}|=K$ and $\vartheta=|\cU\setminus\cU^{\mathrm{top}}|/K$, then
$\sup_{f\in\mathcal{F}_+}|\widehat{q}_{\cU}(f)-\widehat{q}_{\cU^{\mathrm{top}}}(f)|\le M\vartheta$.
\end{lemma}

\begin{proof}
The sets $\cU\setminus\cU^{\mathrm{top}}$ and $\cU^{\mathrm{top}}\setminus\cU$ both have cardinality $K\vartheta$.
For $f\in\mathcal{F}_+$,
\[
\bigl|\widehat{q}_{\cU}(f)-\widehat{q}_{\cU^{\mathrm{top}}}(f)\bigr|
=
K^{-1}\Bigl|\sum_{\cU\setminus\cU^{\mathrm{top}}}f-\sum_{\cU^{\mathrm{top}}\setminus\cU}f\Bigr|
\le
K^{-1}\cdot M\cdot K\vartheta
=
M\vartheta.
\]
\end{proof}

\begin{proposition}[Decomposition of $\Delta_F$]
\label{prop:sel-decomp}
Let $|\cU|=K$ and $\vartheta=|\cU\setminus\cU^{\mathrm{top}}|/K$ (so $\vartheta=0$ if $\gamma=0$).
Conditionally on $\cD_A$, with probability at least $1-\eta$ over $\cS$,
\begin{equation}
\label{eq:sel-decomp}
\Delta_F(\cU)
\;\le\;
\Delta^{\mathrm{bias}}(A_\rho)
+
\frac{5LB+5M\sqrt{\iota}}{\sqrt{K}}
+
M\vartheta.
\end{equation}
\end{proposition}

\begin{proof}
The triangle inequality and Lemmas~\ref{lem:disp}--\ref{lem:samp}.
\end{proof}

\begin{corollary}
\label{cor:e2e}
Under the hypotheses of Proposition~\ref{prop:sel-decomp}, with probability at least $1-2\eta$ over $(\cS,\cD_B)$,
\[
\widetilde{R}(\widehat{w})-\widetilde{R}(w^\star)
\;\le\;
\frac{4LB}{\sqrt{N}}
+
M\sqrt{\frac{2\log(2/\eta)}{N}}
+
\frac{2K}{N}\Bigl(\Delta^{\mathrm{bias}}(A_\rho)+\frac{5LB+5M\sqrt{\iota}}{\sqrt{K}}+M\vartheta\Bigr).
\]
\end{corollary}

\begin{proof}
Union bound of Proposition~\ref{prop:sel-decomp} and Theorem~\ref{thm:aug-bound}.
\end{proof}

\subsection{Discriminator identities}
\label{app:disc}

\begin{lemma}[BCE identity]
\label{lem:bce}
For measurable $D:\cX\to(0,1)$ let
$\mathcal{L}(D)=-\mathbb{E}_p\log D-\mathbb{E}_{q_{\cS}}\log(1-D)$
and $D^\star=p/(p+q_{\cS})$.
Then
\begin{enumerate}
\item[(i)]
$\mathcal{L}(D)-\mathcal{L}(D^\star)=\int(p+q_{\cS})\,\mathrm{KL}(\mathrm{Bern}(D^\star)\|\mathrm{Bern}(D))$,
and $\mathbb{E}_p|D-D^\star|\le\sqrt{\mathcal{L}(D)-\mathcal{L}(D^\star)}$;
\item[(ii)]
$\min_D\mathcal{L}(D)=\log 4-2\,\mathrm{JSD}(p\|q_{\cS})$;
\item[(iii)]
if $D$ is trained with unweighted BCE on $n_H$ positives and $m_1$ negatives, the population minimizer is
$D_\pi=\pi p/(\pi p+(1-\pi)q_{\cS})$ with $\pi=n_H/(n_H+m_1)$, and
$\mathrm{logit}\,D_\pi=\log(p/q_{\cS})+\log(\pi/(1-\pi))$.
\end{enumerate}
\end{lemma}

\begin{proof}
Pointwise,
$-p\log D-q\log(1-D)+p\log D^\star+q\log(1-D^\star)
=(p+q)\,\mathrm{KL}(\mathrm{Bern}(D^\star)\|\mathrm{Bern}(D))$
because $p=(p+q)D^\star$.
Pinsker for Bernoulli laws gives $\mathrm{KL}(\mathrm{Bern}(a)\|\mathrm{Bern}(b))\ge 2(a-b)^2$, hence
$\int(p+q_{\cS})(D-D^\star)^2\le\tfrac12\bigl(\mathcal{L}(D)-\mathcal{L}(D^\star)\bigr)$.
Cauchy--Schwarz yields (i):
\[
\bigl(\mathbb{E}_p|D-D^\star|\bigr)^2
\le
\mathbb{E}_p\bigl[(D-D^\star)^2\bigr]
\le
\int(p+q_{\cS})(D-D^\star)^2
\le
\tfrac12\bigl(\mathcal{L}(D)-\mathcal{L}(D^\star)\bigr)
\le
\mathcal{L}(D)-\mathcal{L}(D^\star).
\]
For (ii), write $\bar{m}=(p+q_{\cS})/2$.
Then $D^\star=p/(2\bar{m})$ on the support of $p+q_{\cS}$, and
\begin{align*}
\mathcal{L}(D^\star)
&=
-\mathbb{E}_p\log\frac{p}{2\bar{m}}-\mathbb{E}_{q_{\cS}}\log\frac{q_{\cS}}{2\bar{m}}
=
2\log 2
-
\mathrm{KL}(p\|\bar{m})
-
\mathrm{KL}(q_{\cS}\|\bar{m})
=
\log 4-2\,\mathrm{JSD}(p\|q_{\cS}).
\end{align*}
For (iii), the population objective at a fixed $x$ is
$-\pi p(x)\log D(x)-(1-\pi)q_{\cS}(x)\log(1-D(x))$.
Its derivative in $D(x)$ is $-\pi p(x)/D(x)+(1-\pi)q_{\cS}(x)/(1-D(x))$.
Setting the derivative to zero gives $\pi p(x)(1-D(x))=(1-\pi)q_{\cS}(x)D(x)$, hence
$D(x)=\pi p(x)/(\pi p(x)+(1-\pi)q_{\cS}(x))$.
The second derivative $\pi p/D^2+(1-\pi)q_{\cS}/(1-D)^2$ is positive on $(0,1)$, so the critical point is a minimum.
The logit identity is $\log(D_\pi/(1-D_\pi))=\log(\pi p/((1-\pi)q_{\cS}))$.
\end{proof}

\begin{lemma}[Pinsker plug-in]
\label{lem:pinsker-app}
If $p\ll q_{\cS}$ and the prior-corrected logit $\psi=\ell_D-\log(\pi/(1-\pi))$ satisfies $\varepsilon_{\mathrm{est}}=\mathbb{E}_p|\psi-\log(dp/dq_{\cS})|<\infty$, then
\[
\mathrm{KL}(p\|q_{\cS})\le\mathbb{E}_p[\psi]+\varepsilon_{\mathrm{est}},
\qquad
\|p-q_{\cS}\|_{\mathrm{TV}}
\le
\sqrt{\tfrac12\bigl(\mathbb{E}_p[\psi]+\varepsilon_{\mathrm{est}}\bigr)}.
\]
\end{lemma}

\begin{proof}
$\mathrm{KL}(p\|q_{\cS})=\mathbb{E}_p\log(dp/dq_{\cS})\le\mathbb{E}_p\psi+\varepsilon_{\mathrm{est}}$, then Pinsker~\cite{pinsker1964information}.
\end{proof}

The functional $\mathbb{E}_p[\psi]$ depends on $(p,q_{\cS})$ and not on $\cU$.

\begin{lemma}[Truncation]
\label{lem:trunc}
Assume $p\ll q_{\cS}$ with $h=dp/dq_{\cS}$.
Let $A$ be measurable with $p(A)>0$, and write $h_{\mathrm{lo}}=\mathrm{ess\,inf}_A h$, $h_{\mathrm{hi}}=\mathrm{ess\,sup}_A h$.
Then
\begin{equation}
\label{eq:trunc}
\|p-q_A\|_{\mathrm{TV}}
\;\le\;
\bigl(1-p(A)\bigr)
+
\bigl(1-h_{\mathrm{lo}}/h_{\mathrm{hi}}\bigr).
\end{equation}
If $|\psi-\log h|\le\varepsilon_A$ and $\psi\in[a,b]$ on $A$, then $1-h_{\mathrm{lo}}/h_{\mathrm{hi}}\le 1-e^{-(b-a)-2\varepsilon_A}$.
\end{lemma}

\begin{proof}
Let $p_A=p(\,\cdot\mid A)$.
With the paper's normalization $\|\mu-\nu\|_{\mathrm{TV}}=\sup_B|\mu(B)-\nu(B)|=\tfrac12\int|d\mu-d\nu|$,
\begin{align*}
\|p-p_A\|_{\mathrm{TV}}
&=
\tfrac12\int_{A^c}p
+
\tfrac12\int_A p\bigl|1-1/p(A)\bigr|
=
\tfrac12\bigl(1-p(A)\bigr)
+
\tfrac12\bigl(1-p(A)\bigr)
=
1-p(A).
\end{align*}
On $A$, $dp_A/dq_A=h/\bar{h}$ with $\bar{h}=\mathbb{E}_{q_A}h\le h_{\mathrm{hi}}$, so $dp_A/dq_A\ge\alpha$ for $\alpha=h_{\mathrm{lo}}/h_{\mathrm{hi}}$.
For every measurable $B\subset A$,
$q_A(B)-p_A(B)\le q_A(B)-\alpha q_A(B)\le 1-\alpha$.
The two one-sided suprema agree, so $\|p_A-q_A\|_{\mathrm{TV}}\le 1-\alpha$.
The triangle inequality is~\eqref{eq:trunc}.
On $A$, $|\psi-\log h|\le\varepsilon_A$ and $\psi\in[a,b]$, so the oscillation of $\log h$ is at most $(b-a)+2\varepsilon_A$.
Hence $h_{\mathrm{lo}}/h_{\mathrm{hi}}\ge e^{-(b-a)-2\varepsilon_A}$.
\end{proof}

\begin{proposition}[Likelihood-ratio truncation]
\label{prop:tv-increase}
Let $\cX=\{1,2\}$, $p=(\tfrac12,\tfrac12)$, and $q_{\cS}=(0.9,0.1)$.
Then $h=(5/9,5)$, the set $A=\{2\}$ consists of the point of largest ratio, $q_A=\delta_2$, and
$\|p-q_A\|_{\mathrm{TV}}=\tfrac12>\tfrac25=\|p-q_{\cS}\|_{\mathrm{TV}}$.
\end{proposition}

\begin{proof}
Direct computation:
$\|p-q_{\cS}\|_{\mathrm{TV}}=\tfrac12(|0.5-0.9|+|0.5-0.1|)=\tfrac25$
and
$\|p-\delta_2\|_{\mathrm{TV}}=\tfrac12$.
\end{proof}

\begin{lemma}[Score separation]
\label{lem:sep}
Let $G\subset\cX$ be measurable with $q_{\cS}(G)\ge 2K/m$ and $s(x)>s(x')$ for all $x\in G$, $x'\notin G$.
Then $A_\rho\subset G$ and $\mathbb{P}(\cU^{\mathrm{top}}\not\subset G)\le e^{-K/4}$.
\end{lemma}

\begin{proof}
$|\cS\cap G|\sim\mathrm{Bin}(m,q_{\cS}(G))$ has mean $\mu\ge 2K$.
If $|\cS\cap G|\ge K$, every point of $G$ outscores every point outside $G$, so the $K$ largest scores lie in $G$ and $\cU^{\mathrm{top}}\subset G$.
Lemma~\ref{lem:chernoff} with $\delta=1/2$ gives $\mathbb{P}(|\cS\cap G|\le(1-\delta)\mu)\le\exp(-\mu/8)$.
Here $(1-\delta)\mu\ge K$, so $\{|\cS\cap G|<K\}\subset\{|\cS\cap G|\le(1-\delta)\mu\}$ and the probability is at most $e^{-\mu/8}\le e^{-K/4}$.
\end{proof}

\begin{proposition}[Contamination comparison]
\label{prop:contamination}
Suppose $q_{\cS}=(1-\beta)q_{\mathrm{good}}+\beta q_{\mathrm{bad}}$ and
$\sup_{f\in\mathcal{F}_+}|\mathbb{E}_{q_{\mathrm{good}}}f-\mathbb{E}_p f|\le\varepsilon_g$,
with every $f\in\mathcal{F}_+$ valued in $[0,M]$.
Then
$\sup_{f\in\mathcal{F}_+}|\mathbb{E}_{q_{\cS}}f-\mathbb{E}_p f|\le(1-\beta)\varepsilon_g+\beta M$.
If a filter returns a mixture with the same good component and contamination rate $\beta_A<\beta$, the same bound with $\beta_A$ in place of $\beta$ is strictly smaller whenever $\varepsilon_g<M$.
\end{proposition}

\begin{proof}[Proof of Proposition~\ref{prop:contamination}]
Lemma~\ref{lem:ipm-mix} applied to $\mu=q_{\cS}$, $\mu_0=q_{\mathrm{good}}$, $\mu_1=q_{\mathrm{bad}}$, and $\nu=p$ is the unfiltered bound.
The filtered mixture is the same identity with $\beta$ replaced by $\beta_A$.
The difference of the two upper bounds is $(\beta-\beta_A)(M-\varepsilon_g)$, which is positive when $\beta_A<\beta$ and $\varepsilon_g<M$.
Lemma~\ref{lem:sep} supplies a sufficient condition for the selected pool to lie in a prescribed good set, with failure probability at most $e^{-K/4}$; on that event the selected contamination rate is the contamination rate of the good set.
\end{proof}

\begin{corollary}[One-sided loss bound]
\label{cor:contam-onesided}
If the hypothesis of Proposition~\ref{prop:contamination} is weakened to
$\sup_{\|w\|\le B}\mathbb{E}_{q_{\mathrm{good}}}[\ell(\ip{w}{\phi(x)})]\le\varepsilon_g$,
then
\[
\sup_{f\in\mathcal{F}_+}\bigl|\mathbb{E}_{q_{\cS}}f-\mathbb{E}_p f\bigr|
\le
(1-\beta)(\varepsilon_g+M)+\beta M.
\]
\end{corollary}

\begin{proof}
For each $f\in\mathcal{F}_+$,
$|\mathbb{E}_{q_{\mathrm{good}}}f-\mathbb{E}_p f|\le\mathbb{E}_{q_{\mathrm{good}}}f+\mathbb{E}_p f\le\varepsilon_g+M$,
and Lemma~\ref{lem:ipm-mix} multiplies the first term by $1-\beta$ and the contamination term by $M$.
The one-sided hypothesis does not control $\mathbb{E}_{q_{\mathrm{good}}}f-\mathbb{E}_p f$ from below by $-\varepsilon_g$, so it does not yield the constant of Proposition~\ref{prop:contamination}.
\end{proof}

\begin{lemma}[Gaussian smoothing]
\label{lem:smooth}
Assume $\phi$ is $1$-Lipschitz and extends to $\mathbb{R}^d$.
Let $\mu,\nu$ have finite second moments, and write $\mu_\sigma=\mu*\mathcal{N}(0,\sigma^2 I_d)$.
Then
$\|\mu_\sigma-\nu_\sigma\|_{\mathrm{TV}}\le W_2(\mu,\nu)/(2\sigma)$
and $|\mu(f)-\mu_\sigma(f)|\le LB\sigma\sqrt{d}$ for every $f\in\mathcal{F}_+$.
Consequently
\[
\sup_{f\in\mathcal{F}_+}\bigl|\mu(f)-\nu(f)\bigr|
\;\le\;
M\|\mu_\sigma-\nu_\sigma\|_{\mathrm{TV}}
+
2LB\sigma\sqrt{d}.
\]
\end{lemma}

\begin{proof}
The KL between $\mathcal{N}(x,\sigma^2 I_d)$ and $\mathcal{N}(y,\sigma^2 I_d)$ equals $\|x-y\|^2/(2\sigma^2)$.
Let $\pi$ be a $W_2$-optimal coupling of $(\mu,\nu)$, and let $Z\sim\mathcal{N}(0,I_d)$ be independent of $(X,Y)\sim\pi$.
Conditioning on $(X,Y)$ and using the convexity of $(\alpha,\beta)\mapsto\|\alpha-\beta\|_{\mathrm{TV}}$ gives
\[
\|\mu_\sigma-\nu_\sigma\|_{\mathrm{TV}}
\le
\int \|\mathcal{N}(x,\sigma^2 I_d)-\mathcal{N}(y,\sigma^2 I_d)\|_{\mathrm{TV}}\,\pi(dx,dy).
\]
Pinsker's inequality bounds each integrand by $\|x-y\|/(2\sigma)$.
Jensen's inequality then yields
$\int\|x-y\|\,\pi(dx,dy)\le\bigl(\int\|x-y\|^2\,\pi(dx,dy)\bigr)^{1/2}=W_2(\mu,\nu)$.
Lemma~\ref{lem:ipm-class} makes every $f\in\mathcal{F}_+$ an $LB$-Lipschitz function, so
$|\mathbb{E}f(X+\sigma Z)-\mathbb{E}f(X)|\le LB\sigma\,\mathbb{E}\|Z\|$.
Since $\mathbb{E}\|Z\|^2=d$, Jensen's inequality gives $\mathbb{E}\|Z\|\le\sqrt{d}$.
The same shift applies to $\nu$.
The last display is the triangle inequality together with the total-variation half of Lemma~\ref{lem:ipm-class} applied to $(\mu_\sigma,\nu_\sigma)$.
\end{proof}

\subsection{Utility and scalarization}
\label{app:scalar}

\begin{lemma}[Utility shapers]
\label{lem:shapers}
Let $\tau>0$ and $\sigma(z)=(1+e^{-z})^{-1}$.
The map $g_+(t)=\log(1+e^{t/\tau})$ is increasing and convex, with $g_+'(t)=\sigma(t/\tau)/\tau\in(0,1/\tau)$ and $\max\{0,t/\tau\}\le g_+(t)\le e^{t/\tau}$.
The map $g_-(t)=-\log(1+e^{-t/\tau})$ is increasing and concave, with $g_-(t)\le\min\{0,t/\tau\}$.
\end{lemma}

\begin{proof}
Differentiate:
$g_+''(t)=\sigma(t/\tau)(1-\sigma(t/\tau))/\tau^2>0$
and
$g_-'(t)=\sigma(-t/\tau)/\tau$ is decreasing.
For the elementary bounds, set $z=t/\tau$.
Then $\log(1+e^z)\ge\log 1=0$, and $\log(1+e^z)\ge\log(e^z)=z$ whenever $e^z\ge 1$, so $\log(1+e^z)\ge\max\{0,z\}$.
Also $\log(1+e^z)\le e^z$ because $\log(1+x)\le x$ for $x\ge 0$.
This is the two-sided bound on $g_+$.
For $g_-$, $\log(1+e^{-z})\ge 0$ gives $g_-(t)\le 0$, and $\log(1+e^{-z})\ge -z$ gives $g_-(t)\le z$.
Hence $g_-(t)\le\min\{0,t/\tau\}$.
\end{proof}

Both shapers are increasing.
The default in the main text is the concave $g_-$; $g_+$ is retained only as an optional convex alternative and is not used in Lemma~\ref{lem:margin}.

\begin{lemma}[Score threshold]
\label{lem:floor}
Let $g$ be continuous and nondecreasing, let $\cU=\cU^{\mathrm{top}}$ with threshold $\hat{t}=\min_{\tilde{x}\in\cU}s(\tilde{x})$, and let $R_{\max}=\max_{\tilde{x}\in\cS}r(\tilde{x})$.
Then for every $\tilde{x}\in\cU$ and $\lambda\in(0,1]$,
$g(\mathrm{margin}_f(\tilde{x}))\ge\bigl(\hat{t}-(1-\lambda)R_{\max}\bigr)/\lambda$.
If in addition $g$ is strictly increasing and the right-hand side lies in the range of $g$, then
$\mathrm{margin}_f(\tilde{x})\ge g^{-1}\bigl(\bigl(\hat{t}-(1-\lambda)R_{\max}\bigr)/\lambda\bigr)$.
If $\lambda\in[0,1)$ and $\mathrm{margin}_f(\tilde{x})\le -a$, then
$r(\tilde{x})\ge\bigl(\hat{t}-\lambda g(-a)\bigr)/(1-\lambda)$.
\end{lemma}

\begin{proof}
Write $m=\mathrm{margin}_f(\tilde{x})$.
Membership in the top-$K$ set gives $\hat{t}\le s(\tilde{x})=\lambda g(m)+(1-\lambda)r(\tilde{x})$.
The bound $r\le R_{\max}$ then yields $\hat{t}\le\lambda g(m)+(1-\lambda)R_{\max}$, and dividing by $\lambda\in(0,1]$ is the utility lower bound.
A continuous strictly increasing $g$ has a continuous inverse on its range, and applying that inverse preserves the inequality.
If instead $m\le -a$, monotonicity gives $g(m)\le g(-a)$, hence
$\hat{t}\le\lambda g(-a)+(1-\lambda)r(\tilde{x})$.
Rearrangement, using $1-\lambda>0$, is the realism lower bound.
\end{proof}

\begin{proposition}[Scalarization monotonicity]
\label{prop:scalarization}
Fix $\cS$ and $K$, and let $\cU(\lambda)$ be any maximizer of $\sum_{\tilde{x}\in\cU}s_\lambda(\tilde{x})$ over $|\cU|=K$, where $s_\lambda=\lambda u+(1-\lambda)r$.
If $\lambda>\lambda'$, then
$\sum_{\cU(\lambda)}u\ge\sum_{\cU(\lambda')}u$
and
$\sum_{\cU(\lambda)}r\le\sum_{\cU(\lambda')}r$.
The comparison uses only the modular part of~\eqref{eq:selection}.
\end{proposition}

\begin{proof}[Proof of Proposition~\ref{prop:scalarization}]
Write $s_\lambda=\lambda u+(1-\lambda)r$ and let $\cU=\cU(\lambda)$, $\cU'=\cU(\lambda')$ be maximizers at $\lambda>\lambda'$.
Optimality gives $\sum_{\cU}s_\lambda\ge\sum_{\cU'}s_\lambda$ and $\sum_{\cU'}s_{\lambda'}\ge\sum_{\cU}s_{\lambda'}$.
Let $\Delta u=\sum_{\cU}u-\sum_{\cU'}u$ and $\Delta r=\sum_{\cU}r-\sum_{\cU'}r$.
The two inequalities become $\lambda\Delta u+(1-\lambda)\Delta r\ge 0$ and $\lambda'\Delta u+(1-\lambda')\Delta r\le 0$.
Subtracting yields $(\lambda-\lambda')(\Delta u-\Delta r)\ge 0$.
Since $\lambda>\lambda'$, one has $\Delta u\ge\Delta r$.
The second inequality rewrites as $\lambda'(\Delta u-\Delta r)+\Delta r\le 0$, hence $\Delta r\le -\lambda'(\Delta u-\Delta r)\le 0$.
The first inequality rewrites as $\Delta r\ge -\lambda(\Delta u-\Delta r)$, so
$\Delta u=(\Delta u-\Delta r)+\Delta r\ge(1-\lambda)(\Delta u-\Delta r)\ge 0$.
\end{proof}

\begin{lemma}[Minimum margin along the modular path]
\label{lem:min-margin}
Assume $\gamma=0$, $u=g(\mathrm{margin}_f)$ with $g$ strictly increasing, and that $s_\lambda$ has no ties.
If $\lambda>\lambda'$, then
$\min_{\cU(\lambda)}\mathrm{margin}_f\ge\min_{\cU(\lambda')}\mathrm{margin}_f$.
The same comparison is not claimed for $\gamma>0$: Lemma~\ref{lem:disp} permits replacement of a $K\vartheta$-fraction of the top-$K$ set, and a replaced point need not keep the modular margin order.
The IPM cost of that replacement is at most $M\vartheta$.
\end{lemma}

\begin{proof}
Let $\cU=\cU(\lambda)$ and $\cU'=\cU(\lambda')$, and suppose toward a contradiction that the minimum margin on $\cU$ is strictly smaller than the minimum margin on $\cU'$.
Let $y\in\cU$ attain the minimum on $\cU$.
Then $y\notin\cU'$, and there exists $x\in\cU'\setminus\cU$ because the two sets have the same cardinality.
Strict increase of $g$ gives $u(x)>u(y)$.
Absence of ties and optimality give $s_\lambda(y)\ge s_\lambda(x)$ and $s_{\lambda'}(x)\ge s_{\lambda'}(y)$.
Set $\Delta u=u(x)-u(y)>0$ and $\Delta r=r(x)-r(y)$.
The two score comparisons become
$\lambda\Delta u+(1-\lambda)\Delta r\le 0$ and $\lambda'\Delta u+(1-\lambda')\Delta r\ge 0$.
Subtracting yields $(\lambda'-\lambda)(\Delta u-\Delta r)\ge 0$, hence $\Delta r\ge\Delta u>0$ because $\lambda'<\lambda$.
Then $\lambda\Delta u+(1-\lambda)\Delta r\ge\Delta u>0$, which contradicts the first comparison.
\end{proof}

\subsection{Complexity}
\label{app:complexity}

\begin{proposition}[Filtering complexity]
\label{prop:complexity}
Let $m=|\cS|$, $d=\dim(\cX)$, $k$ the maximum degree of a symmetrized $k$NN graph, $h$ the discriminator width, and $E$ the number of discriminator epochs.
The filtering time, excluding the final classifier, is
$O(Emdh+mdh+mkd+Kmk)$.
When $\gamma=0$, selection reduces to extracting the $K$ largest scores and costs $O(m+K\log m)$ after scoring.
\end{proposition}

\begin{proof}
Discriminator training visits $O(m)$ candidates per epoch at $O(dh)$ work per pass.
Scoring frozen $f$ and the width-$h$ MLP $D$ is $O(mdh)$.
A sparse $k$NN graph costs $\widetilde{O}(mkd)$.
The marginal gain of a vertex is $O(k)$ after caching coverage maxima; $K$ greedy rounds scan at most $m$ vertices each.
If $\gamma=0$, Lemma~\ref{lem:score-shift} reduces selection to ranking $\tilde{s}$.
\end{proof}

\subsection{Proofs of the main-text shift and margin lemmas}
\label{app:main-lemmas}

\begin{proof}[Proof of Lemma~\ref{lem:pinsker}]
Pinsker's inequality gives $\|p-q\|_{\mathrm{TV}}\le\sqrt{\mathrm{KL}(p\|q)/2}$.
Assumption~\ref{as:real} is uniform, and $\mathrm{KL}(p\|q)=\mathbb{E}_p[\log(p/q)]$, so
\[
\bigl|\widehat{\mathrm{KL}}(p\|q)-\mathrm{KL}(p\|q)\bigr|
=
\bigl|\mathbb{E}_p[\ell_D-\log(p/q)]\bigr|
\le
\varepsilon_{\mathrm{est}}.
\]
Hence $\mathrm{KL}(p\|q)\le\max(0,\widehat{\mathrm{KL}}(p\|q))+\varepsilon_{\mathrm{est}}$,
where the $\max(0,\cdot)$ projection keeps the plug-in in the domain of $\sqrt{\,\cdot\,}$.
The inequality $\sqrt{a+b}\le\sqrt{a}+\sqrt{b}$ yields the sharp plug-in bound
\begin{equation}
\label{eq:pinsker-sharp}
\|p-q\|_{\mathrm{TV}}
\le
\sqrt{\tfrac12\widehat{\mathrm{KL}}(p\|q)}
+
\sqrt{\tfrac12\varepsilon_{\mathrm{est}}}.
\end{equation}
Set $\varphi(\varepsilon)=\sqrt{\varepsilon/2}-\varepsilon$ for $\varepsilon\ge 0$.
Then $\varphi'(\varepsilon)=(2\sqrt{2\varepsilon})^{-1}-1$, which vanishes at $\varepsilon=1/8$, and $\varphi''(\varepsilon)<0$ for $\varepsilon>0$.
Thus $\varphi$ attains its maximum at $\varepsilon=1/8$, where $\varphi(1/8)=1/8$.
Therefore $\sqrt{\varepsilon/2}\le\varepsilon+1/8$, and~\eqref{eq:pinsker-sharp} implies $\|p-q\|_{\mathrm{TV}}\le\hat{\delta}+1/8$ with $\hat{\delta}$ as displayed in the lemma.
If $\varepsilon_{\mathrm{est}}\ge 1/2$, then $\varepsilon_{\mathrm{est}}/2\le\varepsilon_{\mathrm{est}}^2$, so $\sqrt{\varepsilon_{\mathrm{est}}/2}\le\varepsilon_{\mathrm{est}}$ and $\|p-q\|_{\mathrm{TV}}\le\hat{\delta}$.
The additive $1/8$ is not absorbed into $\varepsilon_{\mathrm{est}}$ for every $\varepsilon_{\mathrm{est}}\ge 0$: at $\varepsilon_{\mathrm{est}}=0.01$ one has $\sqrt{0.005}>0.07>0.01$.
The same comparison with the prior-corrected logit is Lemma~\ref{lem:pinsker-app}.
\end{proof}

\begin{proof}[Proof of Lemma~\ref{lem:tv}]
A function $\psi$ with values in $[0,1]$ satisfies $\psi=\int_0^1\mathbf{1}\{\psi>t\}\,dt$.
Integrating the definition of total variation therefore gives
$|\mathbb{E}_{\cP}\psi-\mathbb{E}_{\cQ}\psi|\le\|\cP-\cQ\|_{\mathrm{TV}}$.
Apply this identity to $\psi(x,y)=\ell(y\ip{w}{\phi(x)})$.
If instead $\ell$ takes values in $[0,M]$, the same argument applied to $\ell/M$ multiplies the right-hand side by $M$.
\end{proof}

\begin{proof}[Proof of Lemma~\ref{lem:margin}]
The shaper in the lemma is nondecreasing, so $m_0>0$ forces $g(m_0)\ge g(0)$ and
\[
\max\{0,\,g(0)-g(m_0)\}=0\le e^{-m_0/\tau}.
\]
For the softplus this is the bound stated in the lemma; Lemma~\ref{lem:shapers} records the matching calculus $g'(t)\in(0,1/\tau)$ and $g''>0$.
Assumption~\ref{as:util} places every selected point in $\{\mathrm{margin}_f\ge m_0\}$.
The set of selected points with margin below $m_0$ is therefore empty, and its aggregate contribution to the surrogate is $0$.
Any constant $C_2\ge 0$ satisfies the envelope $C_2 e^{-m_0/\tau}$.
The $1/\tau$-Lipschitz bound controls a utility gap, $g(m_0)-g(m)\le(m_0-m)/\tau$ for $m<m_0$, and Lemma~\ref{lem:shapers} bounds the softplus utility of a point with margin $m\le -m_0$ by $e^{-m_0/\tau}$.
Neither comparison is a bound on $\ell(y\ip{w}{\phi(x)})$.
If Assumption~\ref{as:util} fails on a fraction $\rho_-$ of $\cU$, those points are mixed at weight $K/N$ and each surrogate summand is at most $M$, for a contribution of at most $(K/N)M\rho_-$.
An exponential bound on this contribution requires the additional hypothesis $\rho_-\le e^{-m_0/\tau}$, which is not implied by Lipschitz continuity of $g$.
For top-$K$ selection, Lemma~\ref{lem:floor} converts a utility threshold at least $g(m_0)$ into the margin floor $m_0$ when $g$ is continuous and strictly increasing; on that event $\rho_-=0$.
\end{proof}

\subsection{From the augmented target back to $P$}
\label{app:original-risk}

The comparison below uses the split $\cD=\cD_A\cup\cD_B$ of the opening protocol: $\cU$ is $\cD_A$-measurable, and the empirical risk is formed on $\cD_B$, which is independent of $\cU$.
The classifier in the main text is trained on $\cD\cup\cU$ without this holdout.
McDiarmid's inequality does not control the dependence between $\cU$ and the ERM sample created by that joint training, so the probability statement is proved for the split protocol.

\begin{lemma}[Comparison of $R_P$ with the augmented target]
\label{lem:p-vs-aug}
For every $\|w\|\le B$, $|\widetilde{R}(w)-R_P(w)|\le(K/N)M$.
If $\widehat{w}$ minimizes $\widehat{R}$, $w^\star$ minimizes $\widetilde{R}$, and $w_P^\star$ minimizes $R_P$, then
\[
R_P(\widehat{w})-R_P(w_P^\star)
\le
\widetilde{R}(\widehat{w})-\widetilde{R}(w^\star)
+
\frac{2K}{N}M.
\]
\end{lemma}

\begin{proof}
The first bound is Lemma~\ref{lem:mix-id}.
Decompose
\begin{align*}
R_P(\widehat{w})-R_P(w_P^\star)
&=
\bigl(R_P(\widehat{w})-\widetilde{R}(\widehat{w})\bigr)
+
\bigl(\widetilde{R}(\widehat{w})-\widetilde{R}(w^\star)\bigr)\\
&\quad
+
\bigl(\widetilde{R}(w^\star)-\widetilde{R}(w_P^\star)\bigr)
+
\bigl(\widetilde{R}(w_P^\star)-R_P(w_P^\star)\bigr).
\end{align*}
The third summand is at most $0$ because $w^\star$ minimizes $\widetilde{R}$.
Each of the two differences between $R_P$ and $\widetilde{R}$ is at most $(K/N)M$.
\end{proof}

\begin{proof}[Proof of Theorem~\ref{thm:bound}]
Replace the failure probability in Theorem~\ref{thm:aug-bound} and Proposition~\ref{prop:sel-decomp} by $\eta/2$, which replaces $\iota=\log(4/\eta)$ by $\iota'=\log(8/\eta)$ and requires $K\ge 16\iota'$.
On the intersection, which has probability at least $1-\eta$, Lemma~\ref{lem:p-vs-aug} gives
\begin{equation}
\label{eq:rp-explicit}
\begin{split}
R_P(\widehat{w})-R_P(w_P^\star)
&\le
\frac{4LB}{\sqrt{N}}
+
M\sqrt{\frac{2\log(4/\eta)}{N}}
+
\frac{2K}{N}M\\
&\quad
+
\frac{2K}{N}\Biggl(
\Delta^{\mathrm{bias}}(A_\rho)
+
\frac{5LB+5M\sqrt{\iota'}}{\sqrt{K}}
+
M\vartheta
\Biggr),
\end{split}
\end{equation}
where $\iota'=\log(8/\eta)$ is the sampling width at failure probability $\eta/2$.
Lemma~\ref{lem:ipm-class} bounds the population term by $M\|p-q_\rho\|_{\mathrm{TV}}$.
The discriminator is trained against $q_{\cS}$, so Assumption~\ref{as:real} yields $\ell_D\approx\log(p/q_{\cS})$.
On a top-$K$ level set of mass $\rho=K/|\cS|$ one has $q_\rho=q_{\cS}(\,\cdot\mid A_\rho)$ and
$\log(p/q_\rho)=\log(p/q_{\cS})+\log\rho$,
hence the plug-in for $\mathrm{KL}(p\|q_\rho)$ is $\mathbb{E}_p[\ell_D]+\log\rho$ (plus $\varepsilon_{\mathrm{est}}$).
Omitting $\log\rho$ would bias the estimate by $\Theta(\log(|\cS|/K))$ and can push $\hat{\delta}$ above $1$.
When $\gamma>0$ the additional empirical displacement is the term $M\vartheta$ already displayed.
The sharp plug-in~\eqref{eq:pinsker-sharp} applied to the reweighted logit replaces $\|p-q_\rho\|_{\mathrm{TV}}$ by $\sqrt{\widehat{\mathrm{KL}}_\rho/2}+\sqrt{\varepsilon_{\mathrm{est}}/2}$.
The resulting shift contribution is at most
$(2K/N)M\bigl(\hat{\delta}+1/8\bigr)$,
and at most $(2K/N)M\hat{\delta}$ on $\{\varepsilon_{\mathrm{est}}\ge 1/2\}$.
Since $K\le N$ and $M$ is the loss bound,
$(2K/N)\cdot 5LB/\sqrt{K}=10LB\sqrt{K}/N\le 10LB/\sqrt{N}$
and
$(2K/N)\cdot 5M\sqrt{\iota'}/\sqrt{K}\le 10M\sqrt{\iota'/N}$.
Both are of the order written in Theorem~\ref{thm:bound} once $M=1$, which is the normalization in the main text.
The prior term $(2K/N)M$ is not of the form $C_1(\hat{\delta}+\varepsilon_{\mathrm{est}})$.
Lemma~\ref{lem:prior} shows that the excess of the augmented minimizer on $R_P$ is a positive Bernoulli KL when features are constant and $q=p$, and that this KL does not tend to $0$ if $K/N$ is held fixed.
Under Assumption~\ref{as:util}, Lemma~\ref{lem:margin} contributes $0$, which is at most $C_2 e^{-m_0/\tau}$.
With $C_1\ge 2K/N$ when $M=1$, the display of Theorem~\ref{thm:bound} is~\eqref{eq:rp-explicit} after the numerical factors are absorbed into $\mathcal{O}(LB/\sqrt{N})$ and into a constant multiple of $\sqrt{\log(1/\eta)/N}$, with the calibration contribution in the sharp form~\eqref{eq:pinsker-sharp}, with the additive calibration slack $(K/(4N))$ coming from $(2K/N)\cdot(1/8)$ kept separate from $\varepsilon_{\mathrm{est}}$, and with the prior term $(2K/N)M$ kept separate from $C_1(\hat{\delta}+\varepsilon_{\mathrm{est}})$.
Filtering improves~\eqref{eq:rp-explicit} only by shrinking $\Delta^{\mathrm{bias}}(A_\rho)$ and $\vartheta$, at a fixed budget $K$.
\end{proof}

\normalsize
\section{Certified Reading of the RUBRIC Filter Displays}
\label{app:figures}

Write $\Delta_{c,s}=m^{\mathrm{rub}}_{c,s}-m^{\mathrm{base}}_{c,s}$ for the filtered arm and the generator arm, with both metrics in $[0,1]$.
The operating point is~\eqref{eq:thr-f1} and the certified sign is~\eqref{eq:cert-sign}.
Lemma~\ref{lem:quotient-gap} and Lemma~\ref{lem:rho-range} control the percentage panels in Figures~\ref{fig:heatmap_plots_ieee},~\ref{fig:heatmap_credit}, and~\ref{fig:heatmap_santander}, against Tables~\ref{tab:delta_creditcard}--\ref{tab:delta_santander}.
Lemma~\ref{lem:bars} reads Figures~\ref{fig:comparison_plots_credit}--\ref{fig:comparison_plots_santander} against those $\Delta$ tables and Tables~\ref{tab:main_creditcard}--\ref{tab:main_santander}.
Lemma~\ref{lem:summary-quad} reads Figures~\ref{fig:summary_plots_credit}--\ref{fig:summary_plots_santander}, Figure~\ref{fig:summary_plots_ieee}, and Table~\ref{tab:thr_stats}.

\subsection{Non-commutation of the Base-Normalized Filter Gain}
\label{app:fig-functionals}

\begin{definition}[Seed average and base-normalized gain]
\label{def:rel-imp}
For $(Z_s)_{s\le S}$ set $\mathrm{A}(Z)=S^{-1}\sum_{s=1}^{S}Z_s$, $\tilde Z_s=Z_s-\mathrm{A}(Z)$, and
$\widehat{\sigma}(Z)=\bigl(\mathrm{A}(\tilde Z^{2})\bigr)^{1/2}$.
On $\{b_s>0\}$ the base-normalized gain and its gap from the ratio of averages are
\begin{equation}
\label{eq:rel-imp}
\rho_s(a,b)=\frac{a_s}{b_s},
\qquad
\Gamma(a,b)=\mathrm{A}(\rho(a,b))-\frac{\mathrm{A}(a)}{\mathrm{A}(b)}.
\end{equation}
The plotted heatmap cell is the seed average $100\,\mathrm{A}(\rho(\Delta_{c,\cdot},m^{\mathrm{base}}_{c,\cdot}))$, not a single-seed evaluation.
Write $\bar a=\mathrm{A}(a)$, $\bar b=\mathrm{A}(b)$, and $\beta=\min_s b_s$.
\end{definition}

\begin{lemma}[Gap between the percentage gain and the paired contrast]
\label{lem:quotient-gap}
Let $a_s\in[-1,1]$ and $b_s\ge\beta>0$.
If $\tilde b\not\equiv 0$, then for each $s$ there is a point $\xi_s$ between $b_s$ and $\bar b$ such that
\begin{equation}
\label{eq:quotient-gap}
\Gamma(a,b)
=
\bar a\,\mathrm{A}\Bigl(\frac{\tilde b^{2}}{\xi^{3}}\Bigr)
-
\frac{\widehat{\mathrm{Cov}}(a,b)}{\bar b^{2}}
+
\mathrm{A}\Bigl(\frac{\tilde a\,\tilde b^{2}}{\xi^{3}}\Bigr),
\end{equation}
and
\begin{equation}
\label{eq:gap-bound}
\bigl|\Gamma(a,b)\bigr|
\le
\frac{|\bar a|\,\widehat{\sigma}(b)^{2}}{\beta^{3}}
+
\frac{\widehat{\sigma}(a)\,\widehat{\sigma}(b)}{\bar b^{2}}
+
\frac{2\,\widehat{\sigma}(b)^{2}}{\beta^{3}}.
\end{equation}
The first summand has the sign of $\bar a$.
It vanishes if and only if $b_s$ is constant, in which case $\Gamma(a,b)=0$.
If $\bar a\neq 0$ and $|\Gamma(a,b)|<|\bar a|/\bar b$, then $\mathrm{sign}(\mathrm{A}(\rho(a,b)))=\mathrm{sign}(\bar a)$.
The bound~\eqref{eq:gap-bound} is a sufficient condition for that comparison.
The covariance summand changes sign with $\widehat{\mathrm{Cov}}(a,b)$ and is not a multiple of the first summand.
\end{lemma}

\begin{proof}
The Lagrange expansion of $u\mapsto 1/u$ at $\bar b$ uses $\varphi''(u)=2u^{-3}$, hence remainder $\tilde b_s^{2}/\xi_s^{3}$:
\[
\frac{1}{b_s}
=
\frac{1}{\bar b}-\frac{\tilde b_s}{\bar b^{2}}+\frac{\tilde b_s^{2}}{\xi_s^{3}}
\]
when $\tilde b\not\equiv 0$, and $1/b_s=1/\bar b$ when $\tilde b\equiv 0$.
Multiply by $a_s=\bar a+\tilde a_s$ and apply $\mathrm{A}$.
The sums $\mathrm{A}(\tilde a)$ and $\mathrm{A}(\tilde b)$ vanish, which leaves
\[
\mathrm{A}(a/b)
=
\frac{\bar a}{\bar b}
+
\bar a\,\mathrm{A}\Bigl(\frac{\tilde b^{2}}{\xi^{3}}\Bigr)
-
\frac{\mathrm{A}(\tilde a\,\tilde b)}{\bar b^{2}}
+
\mathrm{A}\Bigl(\frac{\tilde a\,\tilde b^{2}}{\xi^{3}}\Bigr).
\]
Since $\mathrm{A}(\tilde a\,\tilde b)=\widehat{\mathrm{Cov}}(a,b)$, this is~\eqref{eq:quotient-gap}.
Every $b_s\ge\beta$ forces $\xi_s\ge\beta$, so $\mathrm{A}(\tilde b^{2}/\xi^{3})\le\widehat{\sigma}(b)^{2}/\beta^{3}$.
Cauchy--Schwarz bounds the covariance by $\widehat{\sigma}(a)\widehat{\sigma}(b)$.
On $a_s\in[-1,1]$ one has $|\tilde a_s|\le 2$, and the third summand is at most $2\widehat{\sigma}(b)^{2}/\beta^{3}$.
The triangle inequality is~\eqref{eq:gap-bound}.
The tangent inequality $1/b\ge 1/\bar b-(b-\bar b)/\bar b^{2}$ averages to $\mathrm{A}(1/b)\ge 1/\bar b$, with a strict summand whenever some $b_s\neq\bar b$.
Thus the first summand, which equals $\bar a\bigl(\mathrm{A}(1/b)-1/\bar b\bigr)$, has the sign of $\bar a$ and vanishes if and only if $b$ is constant.
In that case every factor $\tilde b$ in~\eqref{eq:quotient-gap} is zero.
The identity $\mathrm{A}(\rho)=\bar a/\bar b+\Gamma(a,b)$ then places $\mathrm{A}(\rho)$ in the open half-line of $\bar a$ whenever $|\Gamma|<|\bar a|/\bar b$.
The covariance term is independent of the sign of $\bar a\bigl(\mathrm{A}(1/b)-1/\bar b\bigr)$.
\end{proof}

\begin{lemma}[Percentage entries above one hundred percent]
\label{lem:rho-range}
If $m^{\mathrm{rub}}_{c,s},m^{\mathrm{base}}_{c,s}\in[0,1]$ and $m^{\mathrm{base}}_{c,s}>0$, then
\begin{equation}
\label{eq:rho-range}
-1
\le
\rho_{c,s}
\le
\frac{1-m^{\mathrm{base}}_{c,s}}{m^{\mathrm{base}}_{c,s}},
\qquad
\{\rho_{c,s}>1\}
=
\{m^{\mathrm{rub}}_{c,s}>2\,m^{\mathrm{base}}_{c,s}\}
\subset
\{m^{\mathrm{base}}_{c,s}<1/2\}.
\end{equation}
If $|\mathrm{A}(\rho_{c,\cdot})|>1$, then some seed lies in that set.
A cell average larger than $100$ in Figures~\ref{fig:heatmap_plots_ieee},~\ref{fig:heatmap_credit}, and~\ref{fig:heatmap_santander} therefore implies the existence of such a seed; the cell itself is the seed average, not a single seed.
The paired contrast in the corresponding cell of Tables~\ref{tab:delta_creditcard}--\ref{tab:delta_santander} still lies in $[-1,1]$.
\end{lemma}

\begin{proof}
The bounds $0\le m^{\mathrm{rub}}_{c,s}\le 1$ are $-m^{\mathrm{base}}_{c,s}\le\Delta_{c,s}\le 1-m^{\mathrm{base}}_{c,s}$.
Dividing by $m^{\mathrm{base}}_{c,s}>0$ gives the two-sided bound on $\rho_{c,s}$.
The inequality $\rho_{c,s}>1$ is $\Delta_{c,s}>m^{\mathrm{base}}_{c,s}$, hence $m^{\mathrm{rub}}_{c,s}>2m^{\mathrm{base}}_{c,s}$, and $m^{\mathrm{rub}}_{c,s}\le 1$ forces $m^{\mathrm{base}}_{c,s}<1/2$.
The lower bound $\rho_{c,s}\ge -1$ yields $\mathrm{A}(\rho_{c,\cdot})\ge -1$, so $|\mathrm{A}(\rho_{c,\cdot})|>1$ means $\mathrm{A}(\rho_{c,\cdot})>1$.
An average of numbers at most $1$ exceeds $1$ only if some summand exceeds $1$.
The paired contrast of the same generator is a difference of $[0,1]$-valued metrics, hence lies in $[-1,1]$.
\end{proof}

\begin{figure}[t]
\centering
\includegraphics[width=0.72\linewidth]{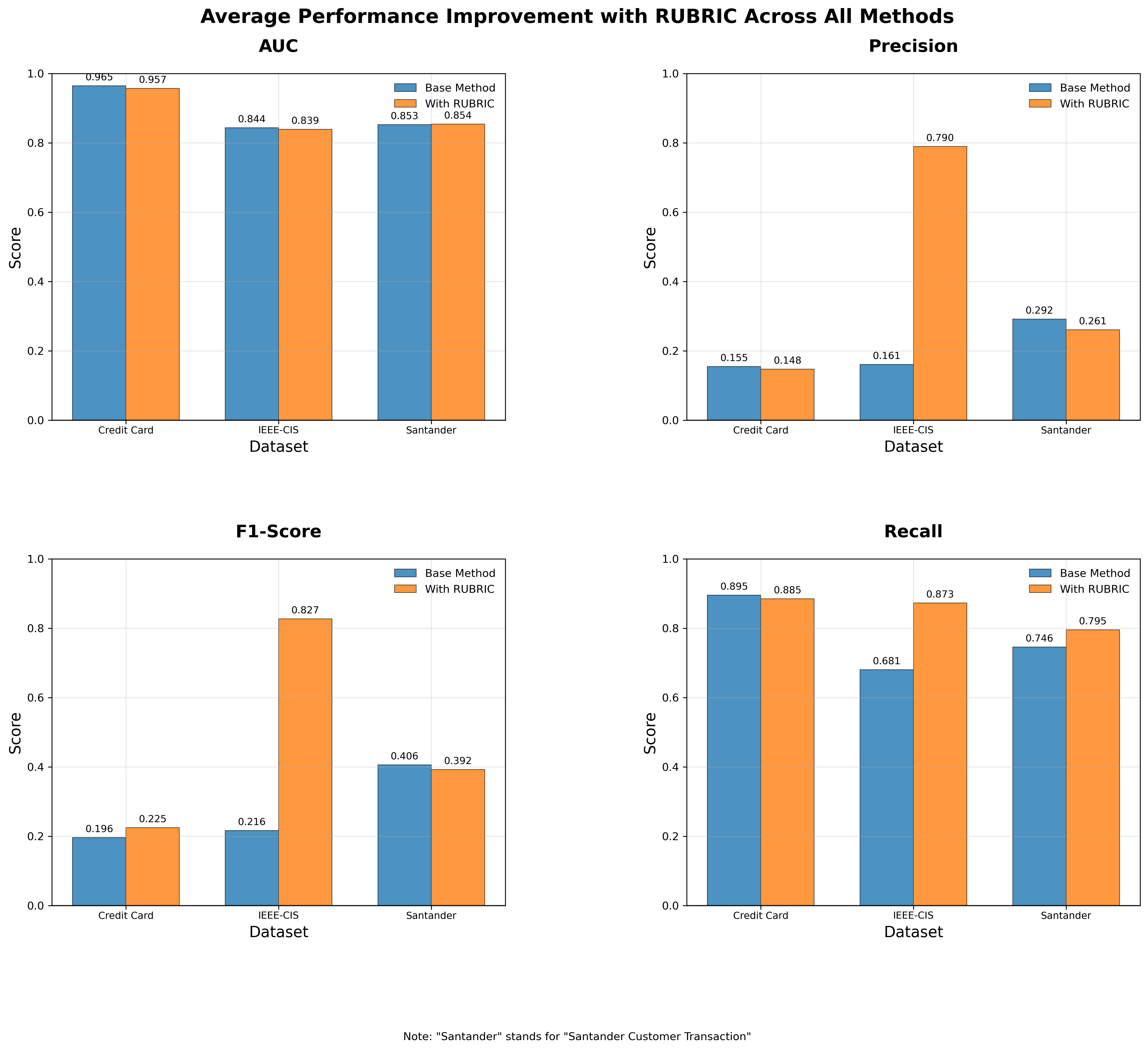}
\caption{AUROC, AUPRC, F1, and Recall under RUBRIC versus baselines (seed averages).}
\label{fig:table_comparison}
\end{figure}

\subsection{SMOTE and SMOTE+RUBRIC at a Shared Validation Threshold}
\label{app:fig-smote}

\begin{lemma}[Bar height against the certified sign]
\label{lem:bars}
Under~\eqref{eq:delta-additive} the shared seed factor cancels, and
\begin{equation}
\label{eq:bar-cancel}
\mathrm{A}(\Delta_{c,\cdot})-\theta_c
=
\mathrm{A}(\xi^{\mathrm{rub}}_{c,\cdot}-\xi^{\mathrm{base}}_{c,\cdot}),
\qquad
\theta_c=\mu^{\mathrm{rub}}_{c}-\mu^{\mathrm{base}}_{c}.
\end{equation}
If $L\le\mathrm{A}(\Delta_{c,\cdot})\le U$ and $\sigma(L,U)\neq 0$, then $\mathrm{sign}(\mathrm{A}(\Delta_{c,\cdot}))=\sigma(L,U)$ and $|\mathrm{A}(\Delta_{c,\cdot})|\ge\mathrm{dist}(0,[L,U])$.
The marginal dispersions in Tables~\ref{tab:main_creditcard}--\ref{tab:main_santander} determine $\widehat{\sigma}(\Delta)$ only after the covariance term of~\eqref{eq:delta-var} is restored.
Those dispersions do not determine the percentile endpoints~\eqref{eq:delta-ci}.
\end{lemma}

\begin{proof}
Subtract the two lines of~\eqref{eq:delta-additive}.
The common summand $\varepsilon_{c,s}$ cancels, and applying $\mathrm{A}$ gives~\eqref{eq:bar-cancel}.
Lemma~\ref{lem:endpoint-gap} is the sign and the distance to zero.
Lemma~\ref{lem:delta-var} is the covariance identity.
The endpoints~\eqref{eq:delta-ci} are quantiles of the replicates~\eqref{eq:delta-joint}, which that identity does not recover from the two marginal dispersions.
\end{proof}

Figures~\ref{fig:comparison_plots_credit}--\ref{fig:comparison_plots_santander} display the seed means of SMOTE and of SMOTE+RUBRIC on Credit Card, IEEE-CIS, and Santander.
The height difference is $\mathrm{A}(\Delta_{c,\cdot})$.
The arrow in the SMOTE row of Tables~\ref{tab:delta_creditcard}--\ref{tab:delta_santander} is $\sigma(L,U)$.

\begin{figure*}[t]
\centering
\includegraphics[width=\linewidth]{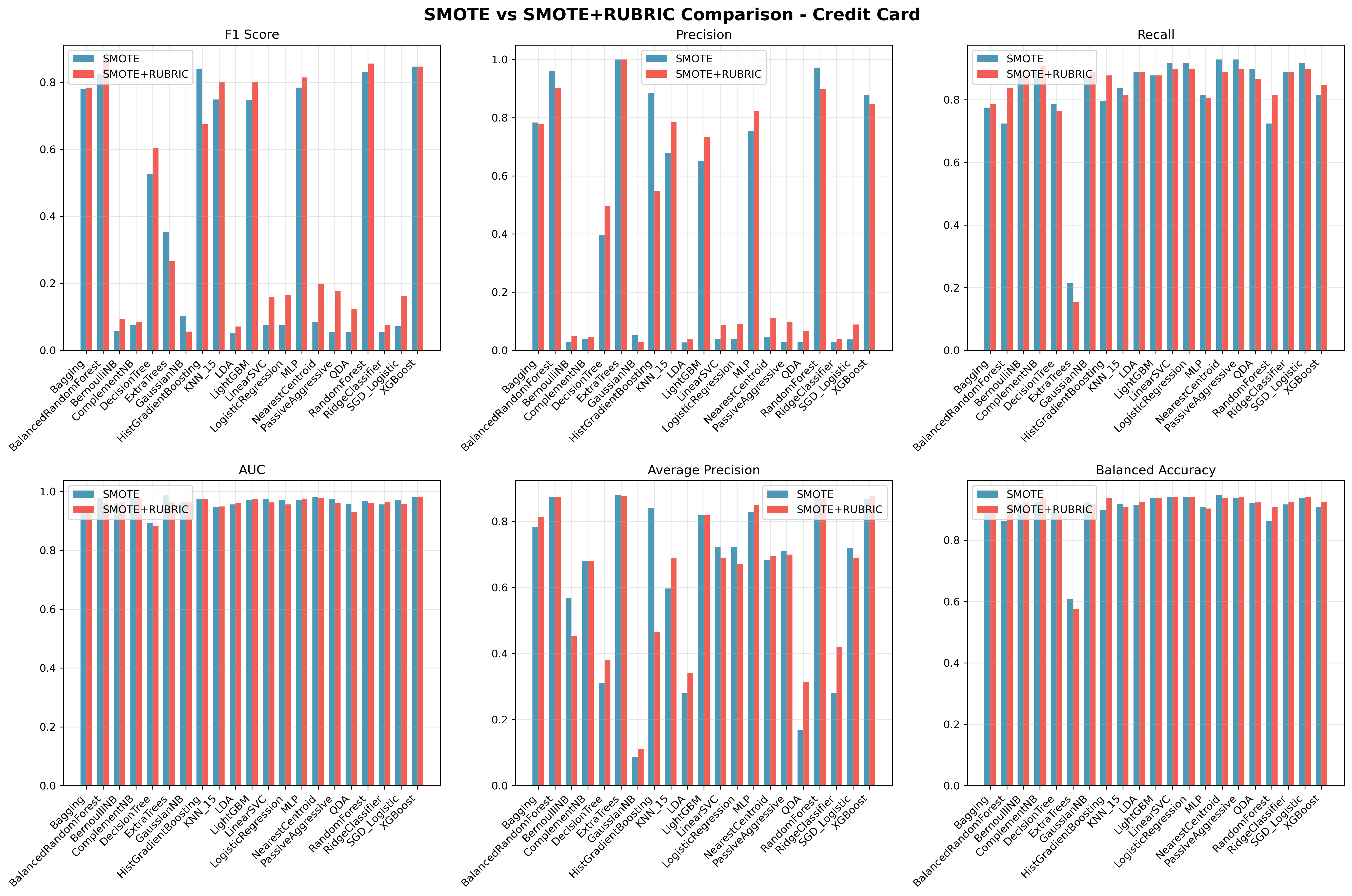}
\caption{Credit Card: SMOTE vs SMOTE+RUBRIC across metrics (mean over seeds; see main text for protocol).}
\label{fig:comparison_plots_credit}
\end{figure*}

\begin{figure*}[t]
\centering
\includegraphics[width=\linewidth]{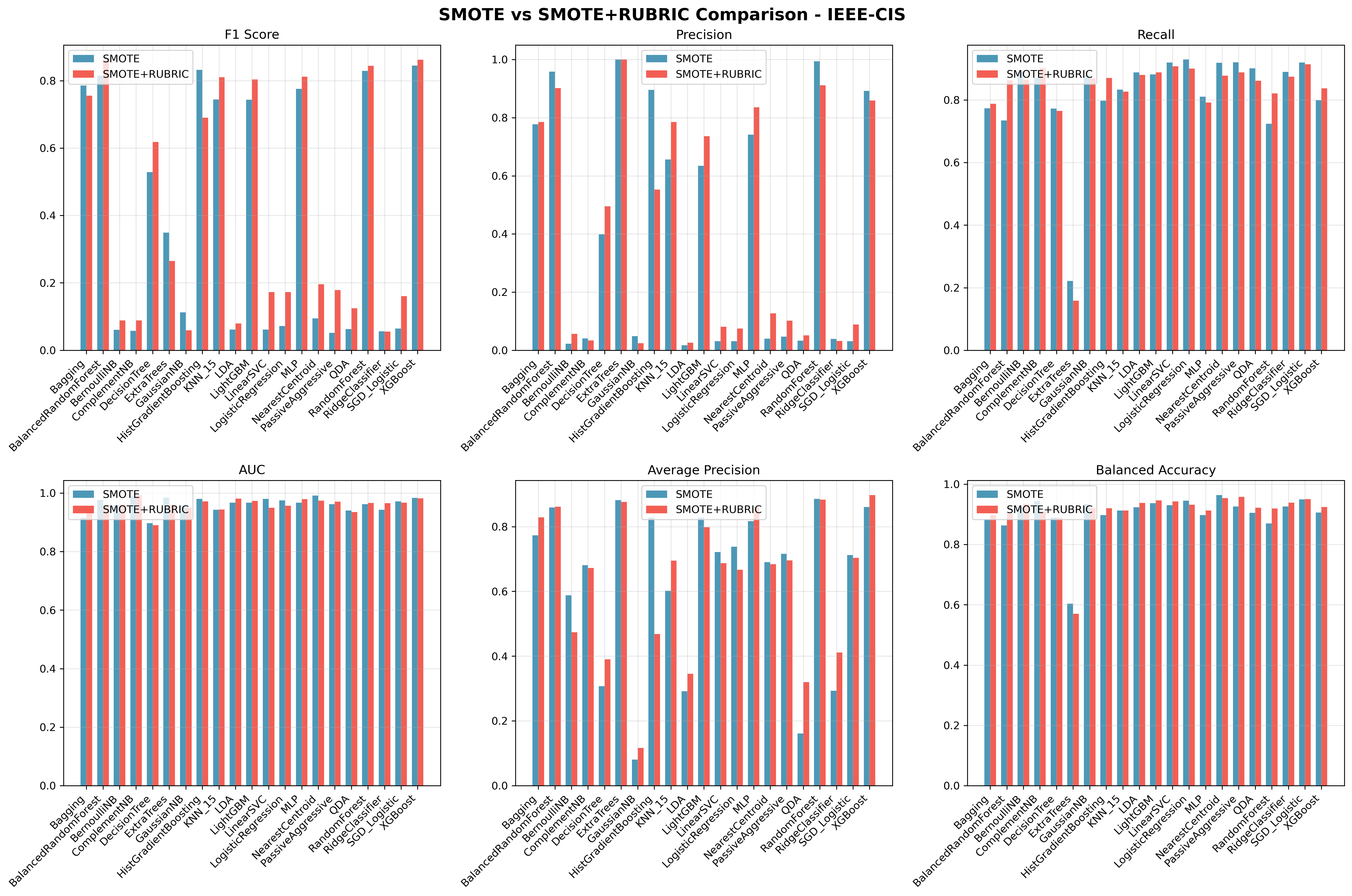}
\caption{IEEE-CIS: SMOTE vs SMOTE+RUBRIC across metrics (mean over seeds; see main text for protocol).}
\label{fig:comparison_plots_ieee}
\end{figure*}

\begin{figure*}[t]
\centering
\includegraphics[width=\linewidth]{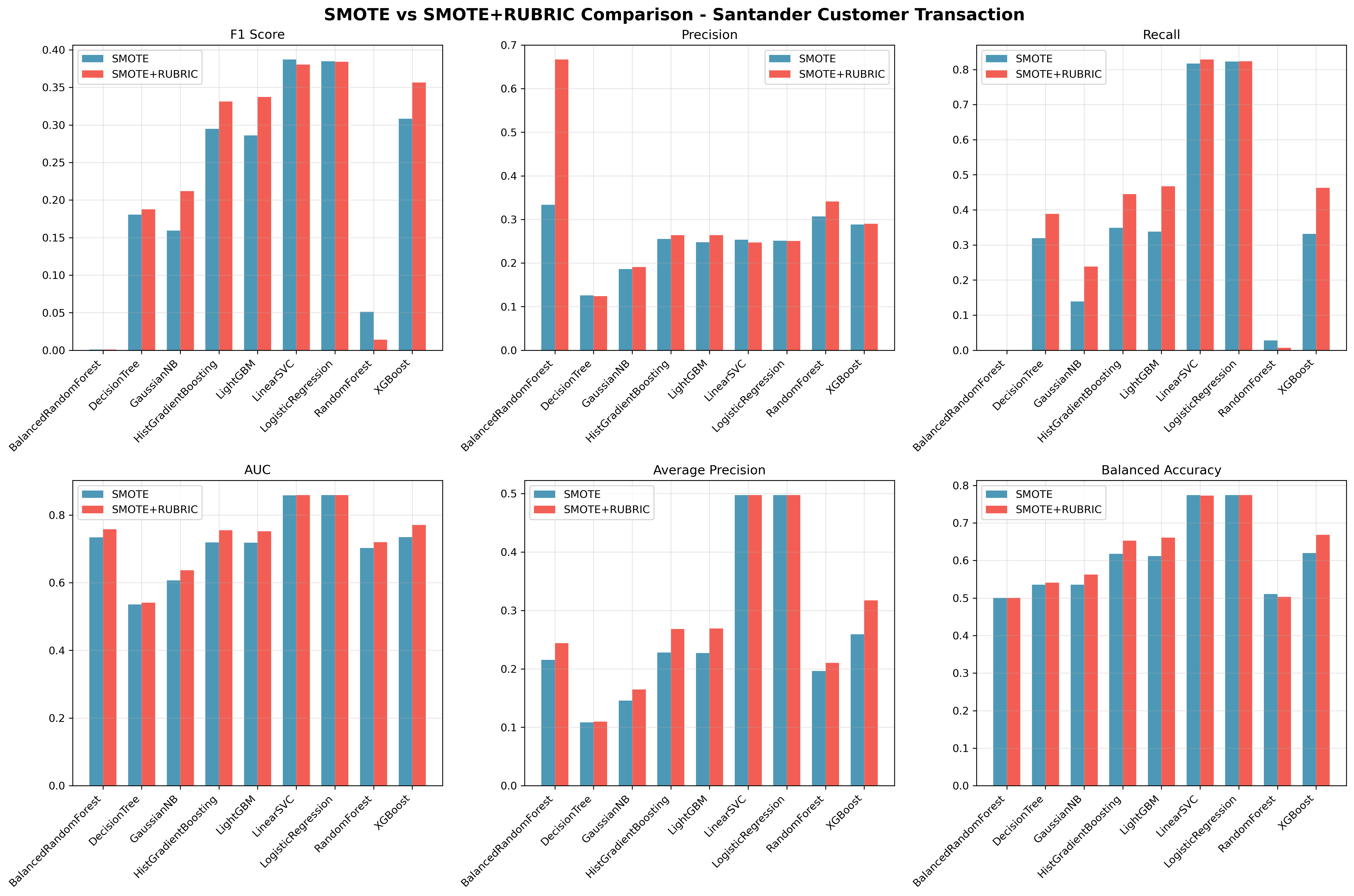}
\caption{Santander: SMOTE vs SMOTE+RUBRIC across metrics (mean over seeds; see main text for protocol).}
\label{fig:comparison_plots_santander}
\end{figure*}

\subsection{Generator-wise Percentage Gain and the Certified Sign}
\label{app:fig-heat}

\begin{proposition}[Percentage panels against the $\Delta$ tables]
\label{prop:heat}
Let $\beta_c=\min_s m^{\mathrm{base}}_{c,s}>0$ and $\bar m_c=\mathrm{A}(m^{\mathrm{base}}_{c,\cdot})$, $\bar\Delta_c=\mathrm{A}(\Delta_{c,\cdot})$.
The cell in Figures~\ref{fig:heatmap_plots_ieee},~\ref{fig:heatmap_credit}, and~\ref{fig:heatmap_santander} is $100\,\mathrm{A}(\rho_{c,\cdot})$.
It has the same sign as $\bar\Delta_c$ whenever the right-hand side of~\eqref{eq:gap-bound}, evaluated at $(\Delta_{c,\cdot},m^{\mathrm{base}}_{c,\cdot})$, is strictly smaller than $|\bar\Delta_c|/\bar m_c$.
A cell of absolute value larger than $100$ forces some seed into~\eqref{eq:rho-range}.
The interval $[L,U]$ is the classifier-averaged contrast of Definition~\ref{def:cluster}.
It does \emph{not} bound an individual classifier mean $\widehat{\Delta}_c$: a component of an average can lie closer to zero than the interval for the average.
On $\{\sigma(L,U)\neq 0\}$ one has $|\bar\Delta_{\mathcal{C}}|\ge\mathrm{dist}(0,[L,U])$ for the classifier average only, and a heatmap cell (which is a per-classifier seed average) carries that certified sign only when it is read against its own paired interval, not against $[L,U]$.
The summands of~\eqref{eq:quotient-gap} depend on $\beta_c$ and on $\widehat{\mathrm{Cov}}(\Delta_{c,\cdot},m^{\mathrm{base}}_{c,\cdot})$, which $(L,U)$ does not determine.
\end{proposition}

\begin{proof}
Lemma~\ref{lem:quotient-gap} writes $\mathrm{A}(\rho_{c,\cdot})=\bar\Delta_c/\bar m_c+\Gamma(\Delta_{c,\cdot},m^{\mathrm{base}}_{c,\cdot})$ and bounds $|\Gamma|$ by~\eqref{eq:gap-bound}.
If that bound is smaller than $|\bar\Delta_c|/\bar m_c$, the sum $\bar\Delta_c/\bar m_c+\Gamma$ lies in the open half-line of $\bar\Delta_c$.
Lemma~\ref{lem:rho-range} turns an absolute average above $1$ into a seed with $m^{\mathrm{rub}}_{c,s}>2m^{\mathrm{base}}_{c,s}$.
Lemma~\ref{lem:endpoint-gap} gives $|\bar\Delta_c|\ge\mathrm{dist}(0,[L,U])$ on a nonzero certified sign, so the same half-line comparison with threshold $\mathrm{dist}(0,[L,U])/\bar m_c$ identifies the cell with $\sigma(L,U)$.
The first and third summands of~\eqref{eq:quotient-gap} scale with $\widehat{\sigma}(m^{\mathrm{base}}_{c,\cdot})^{2}/\beta_c^{3}$, and the second summand is a covariance.
None of these three quantities is a coordinate of the percentile pair $(L,U)$.
\end{proof}

The arrow in the same generator cell of Tables~\ref{tab:delta_creditcard}--\ref{tab:delta_santander} is $\sigma(L,U)$.

\begin{figure*}[t]
\centering
\includegraphics[width=\linewidth]{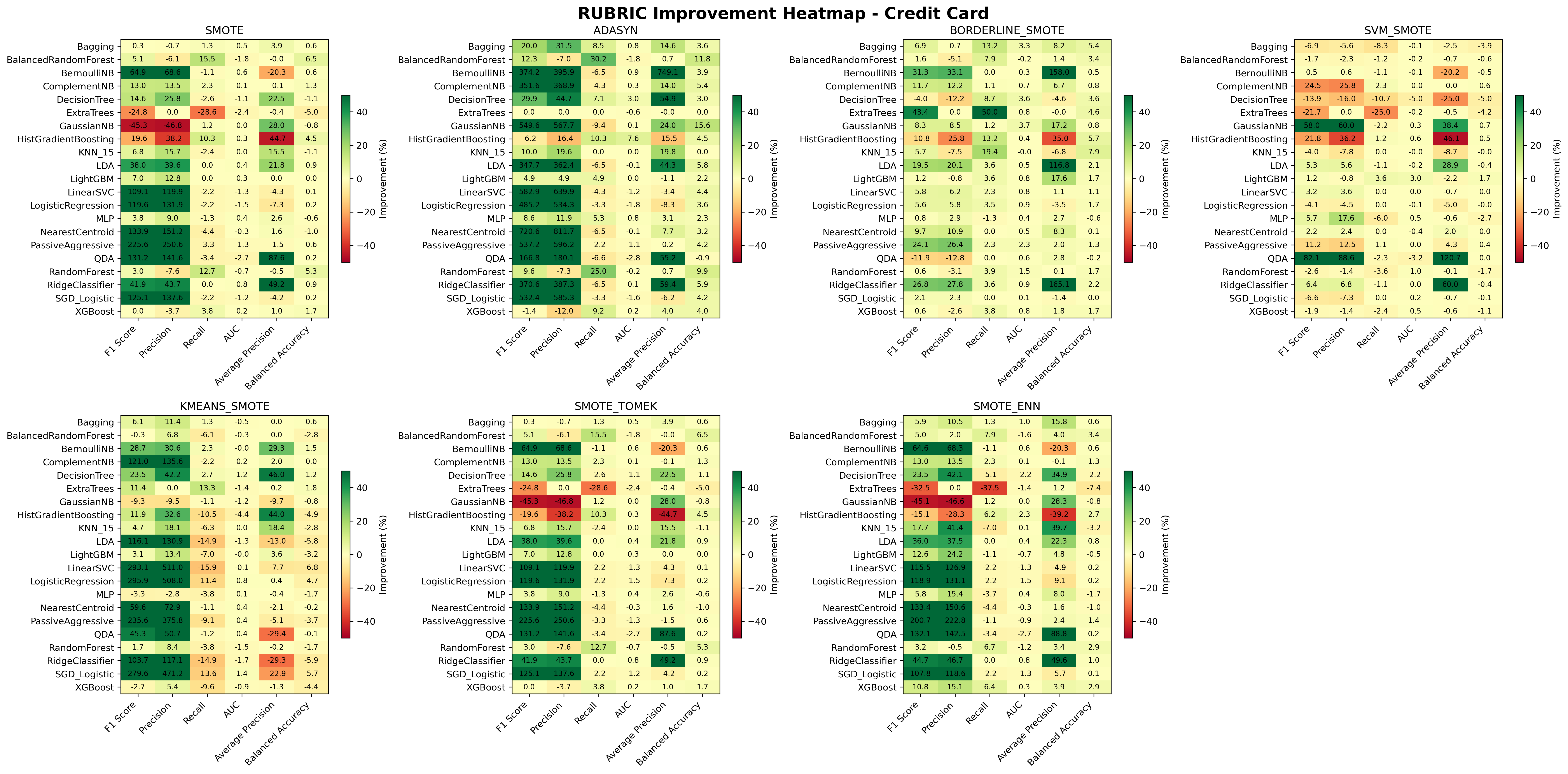}
\caption{Credit Card improvement heatmap (mean across seeds).}
\label{fig:heatmap_credit}
\end{figure*}

\begin{figure*}[t]
\centering
\includegraphics[width=\linewidth]{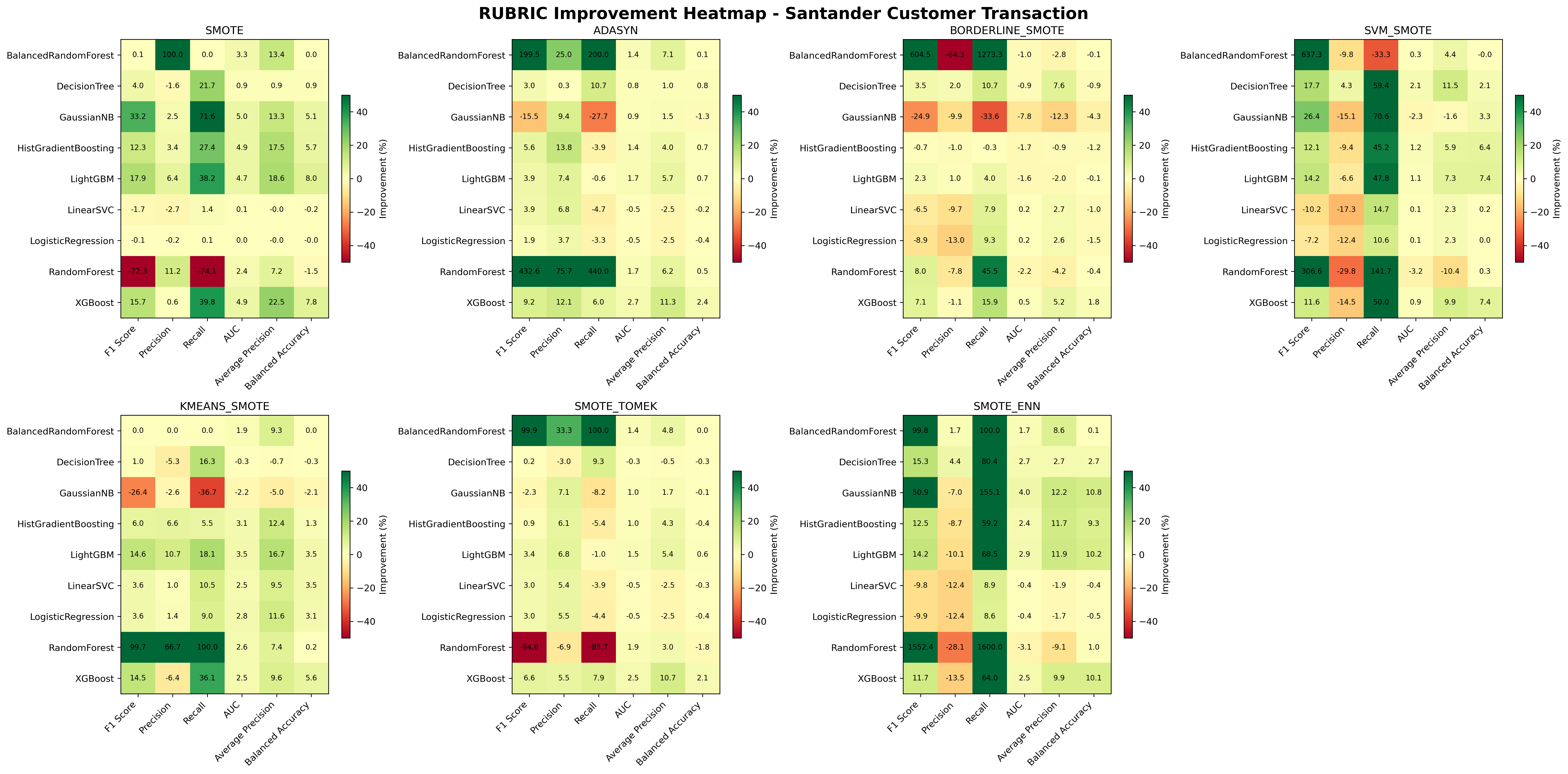}
\caption{Santander improvement heatmap (mean across seeds).}
\label{fig:heatmap_santander}
\end{figure*}

\subsection{Classifier-Averaged Filter Gain and the Seed-Cluster Interval}
\label{app:fig-summary}

\begin{lemma}[Shared-seed term in the summary dispersion]
\label{lem:summary-quad}
Let $M_s=|\mathcal{C}|^{-1}\sum_{c\in\mathcal{C}}m_{c,s}$ and $\widehat{\Sigma}_{cc'}=\widehat{\mathrm{Cov}}(m_{c,\cdot},m_{c',\cdot})$.
Then $\mathrm{A}(M^{\mathrm{rub}})-\mathrm{A}(M^{\mathrm{base}})=\overline{\Delta}_{\mathcal{C}}$ as in~\eqref{eq:delta-config}, and
\begin{equation}
\label{eq:summary-quad}
\begin{aligned}
\widehat{\sigma}(M)^2
&=
\mathrm{A}\bigl((M-\mathrm{A}(M))^2\bigr)
\\
&=
\mathrm{A}\Biggl(\Bigl(|\mathcal{C}|^{-1}\sum_{c\in\mathcal{C}}\tilde m_{c}\Bigr)^{2}\Biggr)
\\
&=
|\mathcal{C}|^{-2}\sum_{c\in\mathcal{C}}\sum_{c'\in\mathcal{C}}\mathrm{A}(\tilde m_{c}\,\tilde m_{c'})
\\
&=
|\mathcal{C}|^{-2}\sum_{c\in\mathcal{C}}\sum_{c'\in\mathcal{C}}\widehat{\Sigma}_{cc'}
\\
&=
\frac{1}{|\mathcal{C}|^2}\sum_{c}\widehat{\Sigma}_{cc}
+
\frac{1}{|\mathcal{C}|^2}\sum_{c\neq c'}\widehat{\Sigma}_{cc'}.
\end{aligned}
\end{equation}
If $\widehat{\Sigma}_{cc'}\ge 0$ off the diagonal and the off-diagonal sum is positive, the second sum is strictly positive.
A bootstrap that resamples classifiers independently of the seed deletes that sum and targets a smaller quadratic form.
\end{lemma}

\begin{proof}
Linearity of $\mathrm{A}$ gives
\begin{align*}
\mathrm{A}(M^{\mathrm{rub}})-\mathrm{A}(M^{\mathrm{base}})
&=
\mathrm{A}\Bigl(|\mathcal{C}|^{-1}\sum_{c}(m^{\mathrm{rub}}_{c,\cdot}-m^{\mathrm{base}}_{c,\cdot})\Bigr)
\\
&=
|\mathcal{C}|^{-1}\sum_{c}\mathrm{A}(\Delta_{c,\cdot})
=
\overline{\Delta}_{\mathcal{C}}.
\end{align*}
The same linearity yields $\mathrm{A}(M)=|\mathcal{C}|^{-1}\sum_{c}\mathrm{A}(m_{c,\cdot})$, hence
$M_s-\mathrm{A}(M)=|\mathcal{C}|^{-1}\sum_{c}\tilde m_{c,s}$.
The first two equalities of~\eqref{eq:summary-quad} are this substitution into the definition of $\widehat{\sigma}$.
Expanding the square and passing the finite sum through $\mathrm{A}$ produces the double sum of $\mathrm{A}(\tilde m_{c}\,\tilde m_{c'})$.
That average is $\widehat{\Sigma}_{cc'}$.
Splitting the index set into $\{c=c'\}$ and $\{c\neq c'\}$ is the last line.
Each off-diagonal summand is the seed covariance of a pair of classifiers.
If those summands are nonnegative and not all zero, the second sum is positive, and deleting them leaves $|\mathcal{C}|^{-2}\sum_{c}\widehat{\Sigma}_{cc}$.
That diagonal form is the target of a bootstrap which resamples classifiers independently of the seed.
Its population counterpart is the diagonal part of Lemma~\ref{lem:cluster}.
\end{proof}

Figures~\ref{fig:summary_plots_credit}--\ref{fig:summary_plots_santander} and Figure~\ref{fig:summary_plots_ieee} display $\mathrm{A}(M^{\mathrm{arm}})$ on Credit Card, IEEE-CIS, and Santander, with dispersion $\widehat{\sigma}(M^{\mathrm{arm}})$.
The interval for $\overline{\Delta}_{\mathcal{C}}$ is the one in Tables~\ref{tab:delta_creditcard}--\ref{tab:delta_santander}.
Table~\ref{tab:thr_stats} records the seed mean and the seed dispersion of~\eqref{eq:thr-f1} for logistic regression on the same three datasets.
Those two moments are not the percentile endpoints of the $\Delta$ tables.
On a fixed score order, Proposition~\ref{prop:rank-thr-factor} leaves the AUC and AUPRC columns of Tables~\ref{tab:main_creditcard}--\ref{tab:main_santander} and of Tables~\ref{tab:delta_creditcard}--\ref{tab:delta_santander} unchanged under that threshold.

\begin{figure*}[t]
\centering
\includegraphics[width=\linewidth]{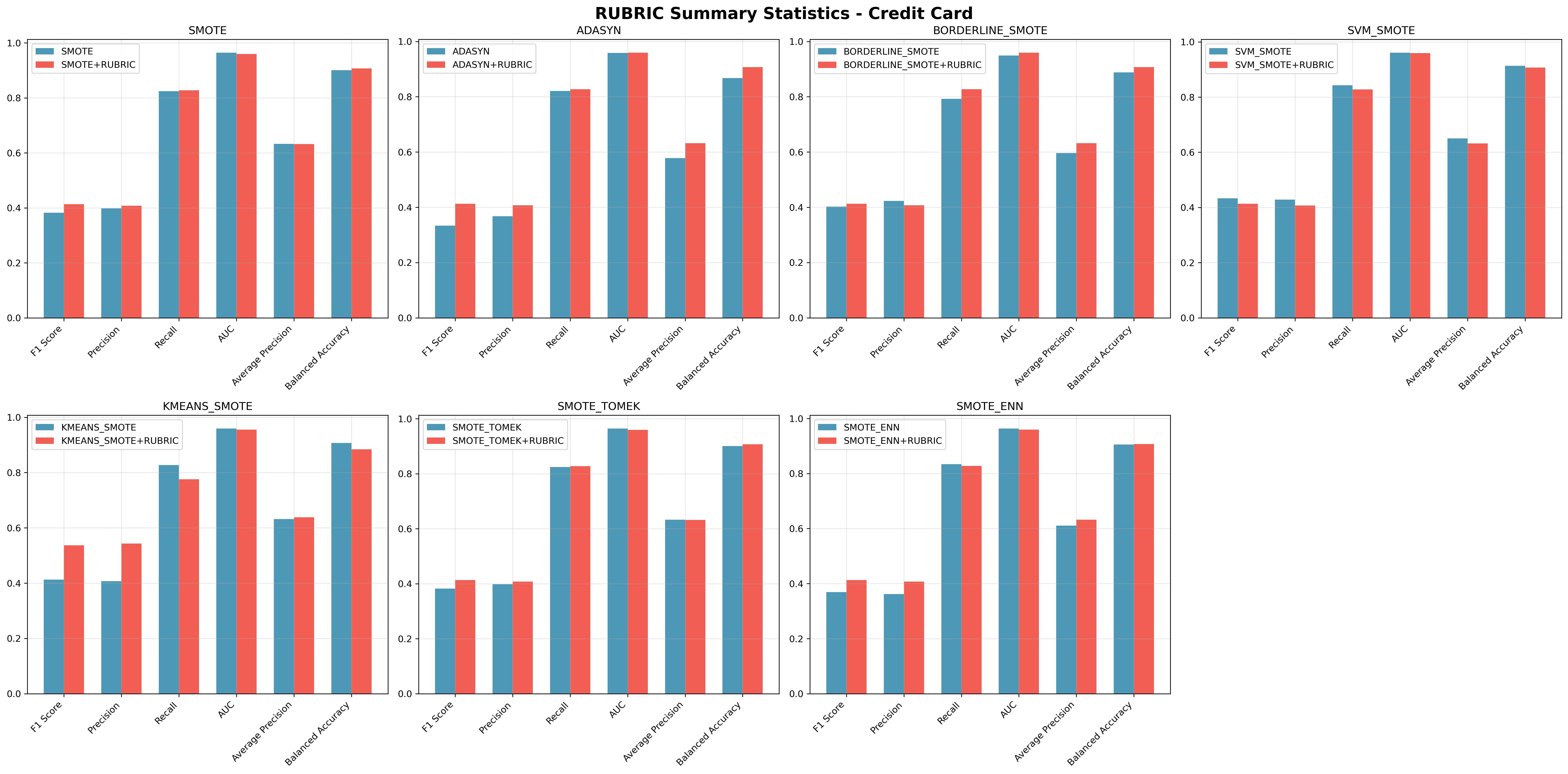}
\caption{Credit Card: average improvement summary (mean±std over seeds).}
\label{fig:summary_plots_credit}
\end{figure*}

\begin{figure*}[t]
\centering
\includegraphics[width=\linewidth]{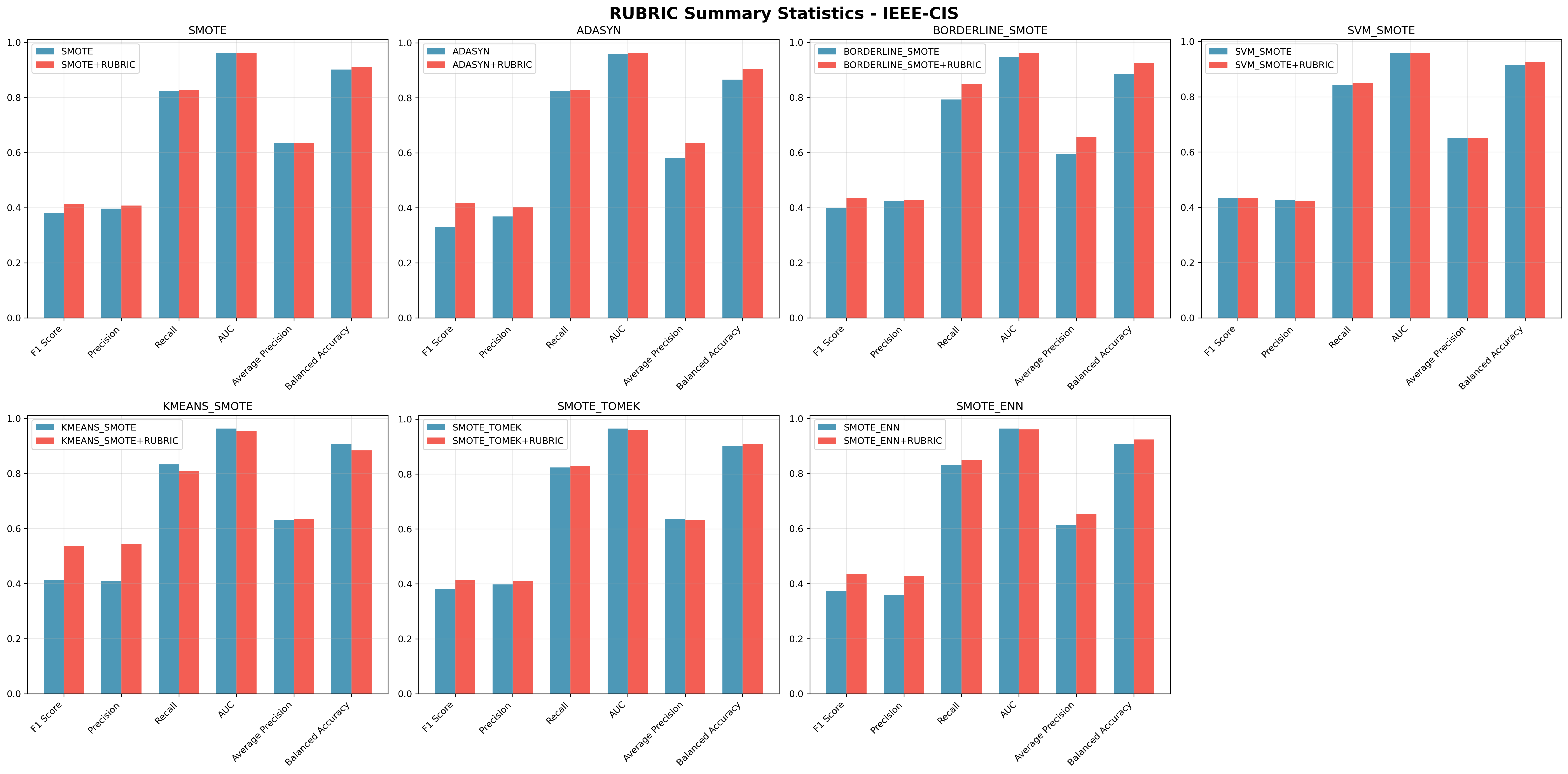}
\caption{IEEE-CIS: average improvement summary (mean±std over seeds).}
\label{fig:summary_plots_ieee}
\end{figure*}

\begin{figure*}[t]
\centering
\includegraphics[width=\linewidth]{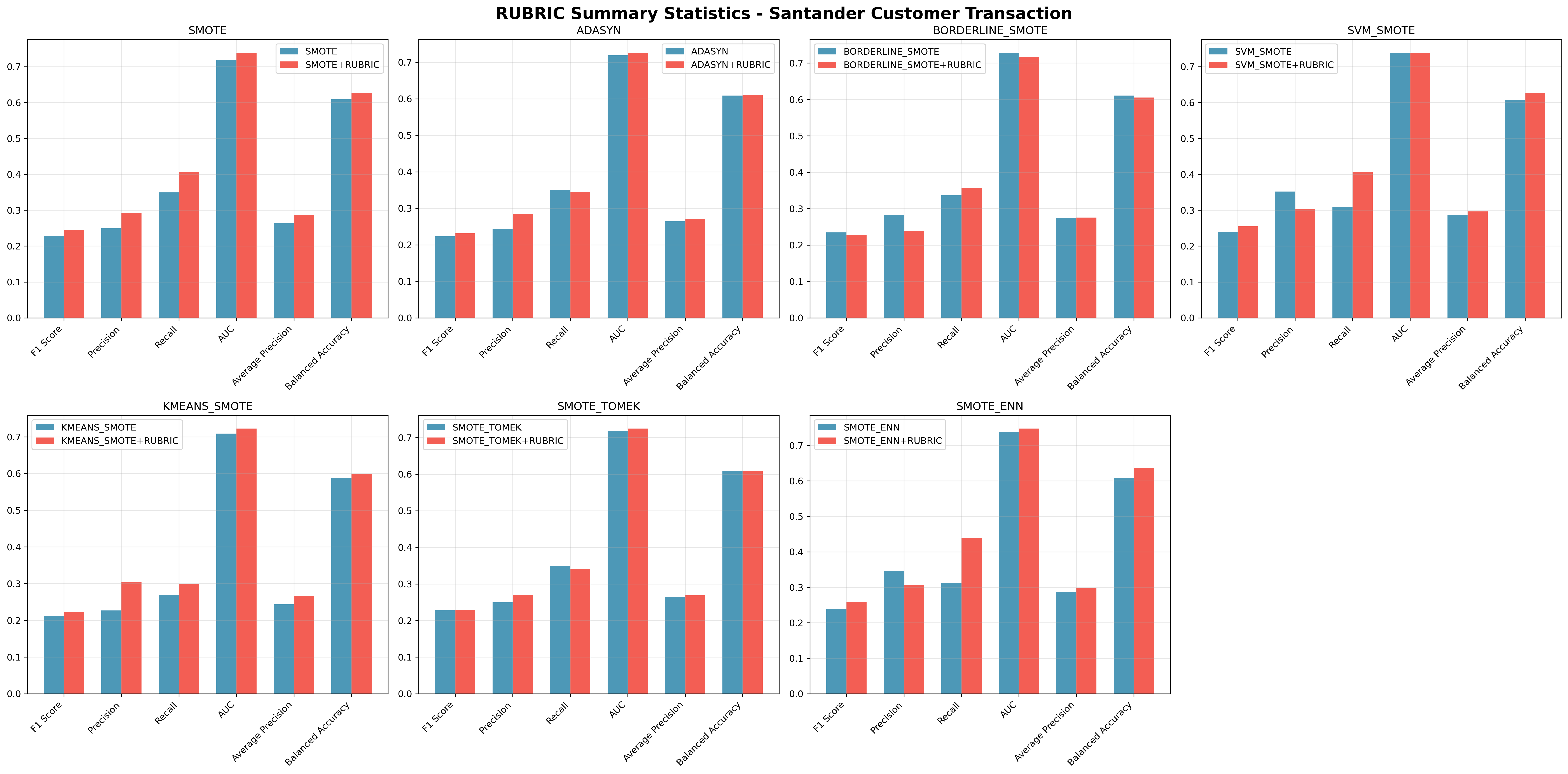}
\caption{Santander: average improvement summary (mean±std over seeds).}
\label{fig:summary_plots_santander}
\end{figure*}


\end{document}